\documentclass[journal]{IEEEtran}
\usepackage{amsfonts}
\usepackage{amssymb}
\usepackage{amsmath}
\usepackage{algorithmic}
\usepackage{array}
\usepackage{soul}
\usepackage{array}
\usepackage{makecell}
\usepackage{enumitem}
\usepackage{multirow}
\usepackage[table]{xcolor}
\PassOptionsToPackage{hyphens}{url}
\usepackage{hyperref}

\usepackage{booktabs} 
\usepackage{textcomp}
\usepackage[ruled]{algorithm2e}
\usepackage{stfloats}
\usepackage{dsfont}
\usepackage{url}
\usepackage{xcolor}
\usepackage{subcaption}
\usepackage{verbatim}
\usepackage{graphicx}
\usepackage{amssymb}
\allowdisplaybreaks[4]
\usepackage{cite}
\usepackage{upgreek}
\usepackage{pifont}

\usepackage{amsthm}
\usepackage{amssymb}
\usepackage{makecell}
\usepackage{enumitem}
\newtheorem{theorem}{Theorem}

\theoremstyle{remark}
\theoremstyle{proposition}
\newtheorem{remark}{Remark}

\theoremstyle{Assumption}

\theoremstyle{Lemma}
\newtheorem{lemma}{Lemma}

\ifodd 0
\newcommand{\rev}[1]{{\color{blue}#1}} %revise of the text
\else
\newcommand{\rev}[1]{#1}
\fi

\begin{document}

\title{\rev{Sense4FL: Vehicular Crowdsensing Enhanced Federated Learning for Object Detection in Autonomous Driving}}
% \vspace{-0.0 cm} }
% \IEEEauthorblockA{$^{\dag}$ City University of Hong Kong, Hong Kong, China}

% \IEEEauthorblockA{$^{\ast}$ Tsinghua University, Beijing, China}}

\author{Yanan Ma,~\IEEEmembership{Student Member, IEEE}, Senkang Hu, Zhengru Fang,~\IEEEmembership{Student Member, IEEE},  Yun Ji,\\
Yiqin Deng,~\IEEEmembership{Member,~IEEE}, and Yuguang Fang,~\IEEEmembership{Fellow,~IEEE}
        % <-this % stops a space
\thanks{The research work described in this paper was conducted in the JC STEM Lab of Smart City funded by The Hong Kong Jockey Club Charities Trust under Contract 2023-0108. The work described in this paper was also partially supported by a grant from the Research Grants Council of the Hong Kong Special Administrative Region, China (Project No. CityU 11216324). The work of Y. Fang was also supported in part by the Hong Kong SAR Government under the Global STEM Professorship. 
The work of Y. Deng was also supported in part by the National Natural Science Foundation of China under Grant No. 62301300, in part by the Shandong Province Science Foundation under Grant No. ZR2023QF053.}
\thanks{Y. Ma, S. Hu, Z. Fang, Y. Deng, and Y. Fang are with Hong Kong JC Lab of Smart City and the Department of Computer Science, City University of Hong Kong, Hong Kong, China. (e-mail: yananma8-c@my.cityu.edu.hk, senkang.forest@my.cityu.edu.hk, zhefang4-c@my.cityu.edu.hk, yiqideng@cityu.edu.hk, my.fang@cityu.edu.hk.)}

\thanks{Y. Ji is with the Key Laboratory
of Advanced Sensor and Integrated System, Tsinghua Shenzhen International Graduate School, Tsinghua University, Shenzhen 518055, China. (email: jiyunthu@gmail.com).}
}

\maketitle

\begin{abstract}
To accommodate constantly changing road conditions, real-time vision model training is essential for autonomous driving (AD). Federated learning (FL) serves as a promising paradigm to enable autonomous vehicles to train models collaboratively with their onboard computing resources. However, existing vehicle selection schemes for FL all assume predetermined and location-independent vehicles' datasets, neglecting the fact that vehicles collect training data along their routes, thereby resulting in suboptimal vehicle selection. 
In this paper, we focus on the fundamental perception problem and propose Sense4FL, a vehicular crowdsensing-enhanced FL framework featuring \textit{trajectory-dependent} vehicular \textit{training data collection} to \rev{improve the object detection quality} in AD for a region. To this end, we first derive the convergence bound of FL by considering the impact of both vehicles' uncertain trajectories and uploading probabilities, from which we discover that minimizing the training loss is equivalent to minimizing a weighted sum of local and global earth mover's distance (EMD) between vehicles' collected data distribution and global data distribution. Based on this observation, we formulate the trajectory-dependent vehicle selection and data collection problem for FL in AD. Given that the problem is NP-hard, we develop an efficient algorithm to find the solution with an approximation guarantee. Extensive simulation results have demonstrated the effectiveness of our approach in improving object detection performance compared with existing benchmarks.
\end{abstract}


\section{Introduction}
\IEEEPARstart{A}{utonomous} driving (AD) enhances road safety, reduces traffic congestion, and provides environmental benefits, which has gained substantial attention lately~\cite{xu2017internet,chen2024vehicle, ma2025raise}. With joint efforts from the academia and auto industry, significant strides have been made in AD. For instance, multiple manufacturers, including Honda and Mercedes-Benz, start selling level 3 cars; Waymo offers rides in self-driving taxis to the public in Arizona (Phoenix) and California (San Francisco and Los Angeles) as of 2024~\cite{Honda2021,Mercedes2021}. 
%Without any doubt, the aforementioned thrilling progress would not have been possible without the advancement of artificial intelligence (AI) capabilities on vehicles.
% \rev{One of the core elements in AD is the perception capability to detect objects (e.g, vehicles, pedestrians, bicycles, and traffic cones) on the road, which enables vehicles to perceive their environment and navigate in diverse scenarios~\cite{10976336,11020620}.}
% , which enable vehicles to detect objects, perceive their environment, and navigate in diverse scenarios \cite{10976336}.
%\rev{AI, particularly visual models, play a crucial role in enabling vehicles to accurately detect objects, perceive their environment, make informed decisions, and navigate in complex scenarios, serving as the cornerstone of AD.}
Despite significant progress in AD, its commercial use is still hampered by real-world deployment challenges and accidents. For instance, several Cruise vehicles were entangled in Muni wires and caution tapes since they failed to detect these objects in extreme weather~\cite{storm}. A Cruise robotaxi struck a pedestrian and dragged her 20 feet as it failed to classify and track the pedestrian~\cite{Cruise}. One primary reason for these accidents is that vision models, such as object classification and detection models, are known to lack generalization capabilities under changing environmental conditions and domain shifts, such as diverse street scenes and extreme weather conditions~\cite{fursa2021worsening}. To enhance the safety of AD, it is essential to adapt a vision model for a specific region and improve it over time. 
To improve models on the fly, federated learning (FL)~\cite{pmlr-v54-mcmahan17a} serves as a promising paradigm,
%The trained (fine-tuned) model can then be distributed to all vehicles passing through this region to improve their perception accuracy. 
where vehicles update and upload their local models to a server for aggregation. Compared with centralized learning, FL has the following salient advantages in the context of AD. First of all, FL may significantly reduce data upload volume. {The raw data rate of a 1080p video stream from a typical vehicle camera can amount to 1493 Mbps\footnote{We calculate the data rate by considering a color depth of 24 bits and a frame rate of 30 fps.}~\cite{vehiclerate}, and each vehicle may be equipped with six or more cameras and other sensors, such as LiDAR}. In comparison, the state-of-the-art object detection model Yolov8m has 25.9 million parameters~\cite{Yolo2024}, i.e., around 52 MB in 16-bit, which is considerably smaller than the sensory data size generated over a period of interest. Second, by leveraging the onboard computing capabilities of a large number of vehicles in parallel, FL is more scalable than centralized learning as it eliminates the need for a powerful central server. Finally, FL safeguards the location privacy and driving behaviors of drivers by preventing application servers from directly accessing their precise locations\footnote{While the FL server can infer that the vehicle is in this region, the precise location can be protected.} and driving states.

Given the advantages of FL for AD, many research efforts have been made to design FL schemes for vehicles~\cite{9928621, 10205502}. 
\rev{However, all existing works consider location-independent vehicles' datasets, akin to conventional FL settings where clients have predetermined local datasets~\cite{nishio2019client,chen2022federated}. Nevertheless, considering a FL scenario where vehicles collectively \textit{collect street views and train (adapt) a vision model}, e.g., object detection/classification model, for a region of interest, like a city}\footnote{\rev{Our proposed framework can be extended to other critical tasks, such as semantic segmentation and control command (acceleration/deceleration, go/stop), which can be investigated in future work.}}, \rev{vehicles collect sensory data, such as street view information, along their routes. Vehicle selection schemes without explicitly considering their trajectory-dependent data distributions result in inferior performance in FL.} \rev{To illustrate the effect of the trajectory-aware vehicle selection, we consider a simple case with four street blocks forming three trajectories, represented by $h_1 = \{1\}$, $h_2 = \{2,4,3\}$, $h_3 = \{2,4\}$. Some street blocks may have more cars whereas some may have more pedestrians, as shown in Fig.~\ref{Fig:intro1}. We evaluate FL performance under several vehicle selection strategies based on trajectories, assuming that each vehicle follows one of these three trajectories. We also assume sufficient communication-computing resources so that trajectories only affect the data distributions without impacting other aspects, e.g., the model uploading probability. As shown in Fig.~\ref{Fig:intro2}, appropriate trajectory-aware vehicle selection can outperform trajectory-agnostic random selection by $4.4\%$, because it can select vehicles with more representative data for this region.}
%Moreover, the data collection time influences the likelihood of successful model uploading, as allocating more time for data collection may not leave enough time for local model training and uploading, reducing the likelihood of successful model reception.

\begin{figure}[t]
\centering
\begin{subfigure}[b]{0.26\textwidth} 
  \includegraphics[width=\textwidth]{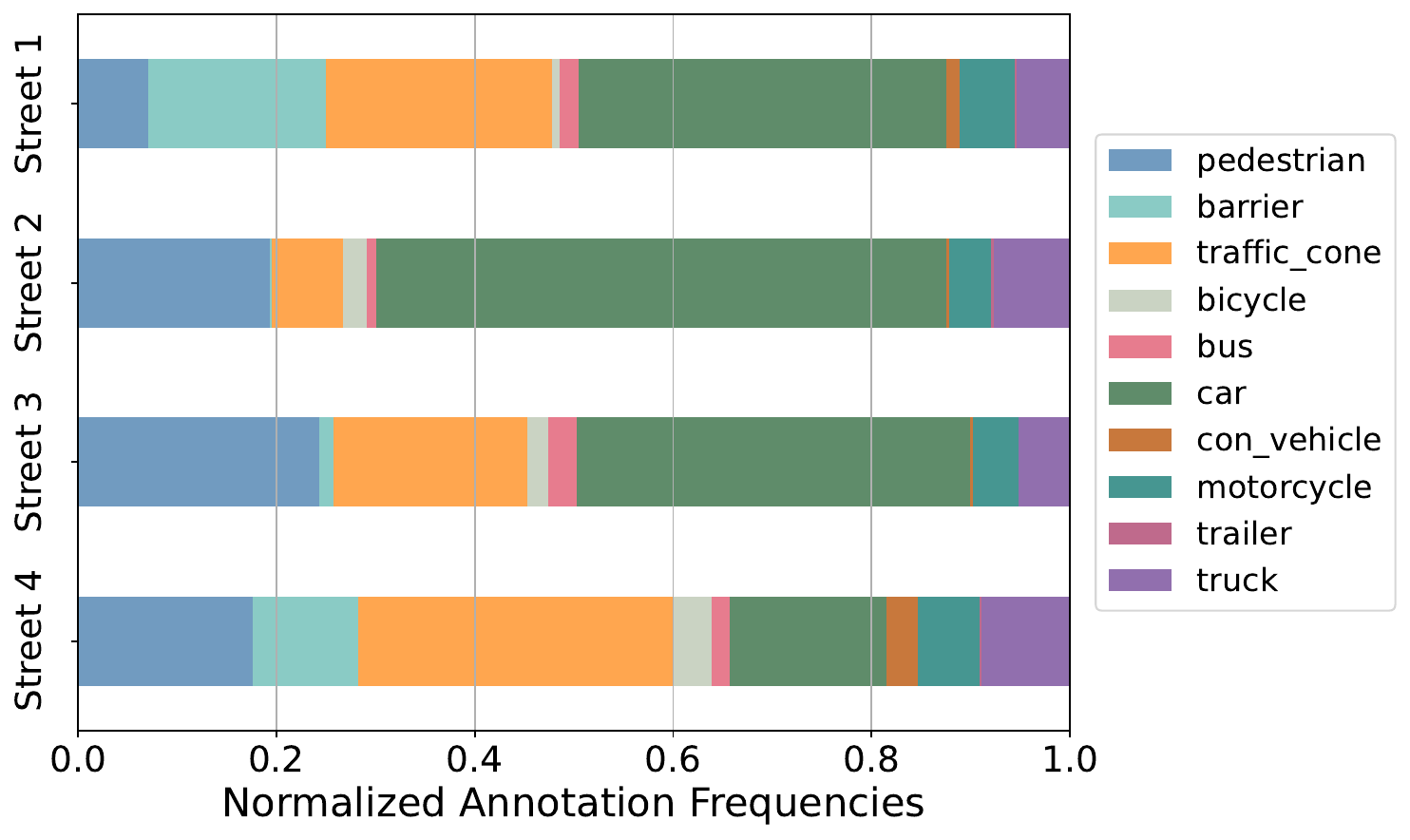}
  \caption{The normalized frequency of \\ objects for each street block.}\label{Fig:intro1}
\end{subfigure}
% \hfill
\begin{subfigure}[b]{0.217\textwidth} 
  \includegraphics[width=\textwidth]{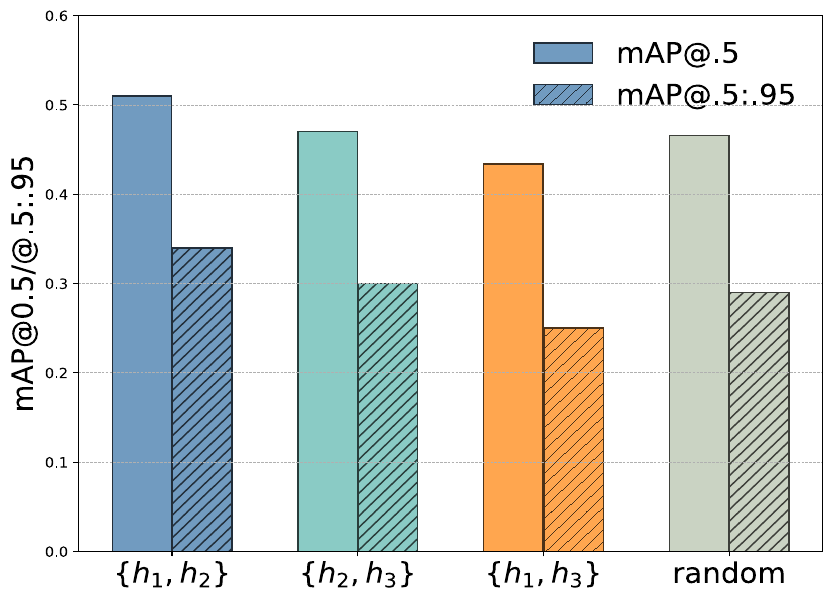}
  \caption{Final test accuracy for different selection strategies.} \label{Fig:intro2}
\end{subfigure}
\caption{\rev{The data distribution of four neighboring street blocks and the final test accuracy (after convergence) via different vehicle selection strategies in the nuImages dataset. $\{h_{x}, h_{y}\}$ means that we select one vehicle from trajectory $x$ and one vehicle from $y$ in each round; ``Random'' means we randomly choose two vehicles in each round.
}}\label{fig:intro}
% \vspace{-0.1cm}
\end{figure}

% We show the FL performance with different trajectory-aware selection strategies. The specific data distribution of each street block is visualized in Fig. \ref{Fig:intro1}.
% % and the FL performance results are shown in Fig. \ref{fig:intro}.
% % The specific data distribution for each street block is visualized in Fig. \ref{Fig:intro1}. In each FL round, we select 2 vehicles to participate in the FL process. 
% % Accordingly, we evaluate three trajectory-aware selection strategies—$\{h_1,h_2\}$, $\{h_2,h_3\}$, and $\{h_1,h_3\}$—and include a random baseline that selects two vehicles per round without constraints (both may come from the same trajectory). 
% Fig. \ref{Fig:intro2} shows that the choice of vehicle trajectories has an influence on model performance. 
% % Motivated by this observation, we introduce Sense4FL, a trajectory-aware vehicular data collection framework, to facilitate federated learning for autonomous driving.
%In this paper, we specifically focus on the object detection task in AD, where vehicles serve as data collectors to capture images and videos of road scenes to constitute their training datasets and act as FL clients to train a global model for a region of interest collaboratively

\rev{Based on the aforementioned observations, in this paper, we introduce a trajectory-dependent vehicular training data collection framework, i.e., \textit{vehicular crowdsensing enhanced FL (Sense4FL)}, to facilitate FL for AD. 
% Furthermore, to preserve privacy, the exact trajectory of a vehicle cannot be known by the edge node; instead, only an estimated set of possible traveling trajectories can be predicted based on its current location and direction. Consequently, when optimizing vehicle selection and data collection, it is essential to take into account the uncertainty of the trajectory.
%To characterize the impact of data heterogeneity on model accuracy and instruct the data collection scheme and vehicle selection design, in this paper, we employ the earth mover’s distance (EMD) metric \cite{zhao2018federated} to quantify the data heterogeneity, which measures the minimum average distance required to move data from one distribution to the other distribution and evaluates the similarity between two probability distributions. 
%  To obtain a model with high accuracy, we should design data collection and vehicle selection with more balanced data, so that we can enhance the performance of AD systems in the considered region.
% }
% \rev{Considering vehicles' data collection and mobility behaviour, the server selects a subset of vehicles to maximize the model accuracy (or equivalently, to minimize the accuracy loss) in FL.}
We first rigorously derive the convergence upper bound of Sense4FL by considering the trajectories of selected vehicles, characterizing the upper bound using earth mover’s distance (EMD) metrics in terms of vehicles' and global data distributions. To minimize the upper bound, our framework optimizes vehicle selection and data collection decisions by taking two factors into account: i) the distribution of collected datasets according to selected vehicles' routes, and ii) uploading probabilities resulting from vehicles' routes, computing capabilities, and communication capabilities. Note that prior works on vehicular FL neglect the first aspect, which may cause model performance degradation since the training datasets from participating vehicles may \textit{not} reflect the global data distribution in the region of interest.} 
The key contributions of this paper are summarized below.

\begin{itemize}

\item \rev{First, we present the Sense4FL framework for training an object detection model in AD, which determines vehicle selection and data collection in a region of interest by considering street data distribution and vehicular mobility.}
%we provide a comprehensive analysis of the Sense4FL framework for adapting an object detection model in AD, incorporating vehicular mobility, training data collection, the overview of Sense4FL, and the FL procedure.

\item \rev{Second, we derive the convergence bound of FL by considering the impact of both vehicles’ uncertain trajectories and uploading probabilities, establishing the theoretical relationship between model accuracy loss and data collection. We discover that minimizing the training loss is equivalent to minimizing a weighted sum of client and global EMD between vehicles’ collected data distribution and regional data distribution.}

\item Then, we formulate the joint vehicle selection and training data collection problem to minimize FL training loss. It turns out that this problem is a non-linear multiple-choice knapsack problem (MCKP) with a non-convex non-separable objective function with non-continuous variables. Given that the optimization problem is NP-hard and highly challenging, we develop an efficient algorithm to obtain the solution with an approximation guarantee.

\item Finally, we demonstrate the effectiveness of our approach for the state-of-the-art object detection model, YOLO, with the nuImages dataset. The simulation results show that our algorithm significantly improves the performance of object detection compared to existing benchmarks and enables fast adaptation of models under changing environments.
\end{itemize}

The remainder of this paper is organized as follows. Section \ref{sec:work} introduces the related work. Section \ref{sec:framework} elaborates on the proposed Sense4FL framework. Section \ref{sec:convergence} provides the
convergence analysis. We formulate the optimization problem
in Section \ref{sec:problem} and offer the corresponding solution approach
in Section \ref{sec:algorithm}. Section \ref{sec:exp} provides the simulation results.
Finally, concluding remarks are presented in Section \ref{sec:conclusion}.

% \item Based on the derived convergence bound, we formulate a joint optimization problem for MA and MS, aiming to minimize training latency for achieving the target learning performance.
% \item We decompose the joint optimization problem into two tractable sub-problems for MS and MA and develop efficient algorithms to solve them optimally. Then, we devise a block coordinate descent (BCD) based method to obtain efficient sub-optimal solutions to the joint optimization problem.
% \item We conduct extensive simulations across various datasets to validate the theoretical analysis and effectiveness of the proposed solutions.

%Autonomous vehicles, equipped with advanced sensing, communication capabilities, powerful computing, large storage, and high intelligence, are crucial for the future development of smart cities. 

%These vehicles can collaboratively update high-definition maps by collecting and uploading real-time road information, maintaining current map data and provide accurate navigation services to other autonomous vehicles and pedestrians, improving citizens' quality of life~\cite{chen2024vehicle}.

%Recent research has focused on collaborative perception for connected and autonomous vehicles~\cite{fang2024pacp},~\cite{hu2024collaborative}. Large language models (LLMs) have been instrumental in enhancing collaboration in driving tasks~\cite{hu2024agentscodriver}. Additionally, significant progress has been made in object detection~\cite{rjoub2022active} and in-vehicle monitoring systems~\cite{zheng2020v2ifi}.

\section{Related Work}
\label{sec:work}
A substantial body of research has focused on optimizing federated learning at the network edge~\cite{wu2021fast, 10850596, asaad2024joint, zarandi2021federated}. Due to data and resource heterogeneity in FL~\cite{pmlr-v54-mcmahan17a, 9928621, nishio2019client}, client selection is crucial for learning performance. Numerous research efforts have been conducted in this area. Nishio \textit{et al.} in~\cite{nishio2019client} proposed a client selection scheme in wireless networks aimed at maximizing the number of uploaded models to enhance learning performance. In \cite{cho2022towards}, Cho \textit{et al.} performed a convergence analysis of FL over biased client selection, demonstrating that selecting clients with higher local loss achieves faster convergence. Chen \textit{et al.} optimized client selection and radio resource allocation by taking packet errors into account~\cite{chen2020joint}. Considering bandwidth limitations, Huang \textit{et al.} in~\cite{huang2020efficiency} optimized client selection with a fairness guarantee based on Lyapunov optimization. Xu \textit{et al.} in~\cite{9237168} formulated a stochastic optimization problem for joint client selection and bandwidth allocation under long-term client energy constraints. By considering heterogeneous client hardware constraints and data quality, Deng~\textit{et al.} developed an automated, quality-aware client selection framework for FL~\cite{9647925}. 
Zhu \textit{et al.} introduced an asynchronous FL framework with adaptive client selection to minimize training latency while considering client availability and long-term fairness~\cite{zhu2022online}. 
\rev{By identifying and excluding adverse local updates, Wu \textit{et al.} proposed an optimal aggregation algorithm and a probabilistic client selection framework to accelerate model convergence \cite{wu2022node}.}
\rev{However, these client selection schemes do not account for user mobility that may affect FL performance, which are not suited for vehicular environments.}

\begin{table}[!t]
\centering
\caption{\rev{Summary of related works in vehicular FL systems.}}
\label{table_comp}
\resizebox{\linewidth}{!}
{
\renewcommand{\arraystretch}{1.4}
\setlength{\tabcolsep}{2mm}
\rev{
\begin{tabular}{|c|c|c|c|c|}
\hline
{\textbf{Ref.}}& 
\makecell[c]{\textbf{Client} \\ \textbf{Selection}}& 
\makecell[c]{\textbf{Data} \\ \textbf{Collection}}& 
\makecell[c]{\textbf{Theoretical} \\ \textbf{Analysis}}& 
\makecell[c]{\textbf{AD} \\ \textbf{Dataset}} 
  \\ \hline
\cite{8964354} %& 2023    
    & {\ding{52}}  & {\ding{55}}  & {\ding{55}} & {\ding{55}}    \\ \hline
\cite{pervej2023resource} %& 2024    
    & {\ding{52}}  & {\ding{55}}  & {\ding{52}} & {\ding{55}}    \\ \hline 
\cite{zhao2022participant} %& 2024     
    & {\ding{52}}  & {\ding{55}}  & {\ding{55}} & {\ding{55}}  \\ \hline
\cite{zhao2021system} %& 2024     
    & {\ding{52}}  & {\ding{55}}  & {\ding{55}} & {\ding{55}}  \\ \hline
\cite{xiao2021vehicle} %& 2024    
    & {\ding{52}}  & {\ding{55}}  & {\ding{55}} & {\ding{55}}   \\ \hline
% \cite{9806308} %& 2024
%     & {\ding{52}}    & {\ding{55}}  & {\ding{52}}  & {\ding{52}}  \\ \hline
\cite{zhang2023vehicle}
    & {\ding{52}}    & {\ding{55}}  & {\ding{52}} & {\ding{55}}  \\ \hline
\cite{10643168} 
    & {\ding{52}}  & {\ding{55}}   &  {\ding{52}} & {\ding{55}} \\ \hline
\cite{zheng2023autofed} %& 2023    
    & {\ding{52}}   & {\ding{55}}  & {\ding{55}} & {\ding{52}} \\ \hline
Ours %& 2024
    & {\ding{52}}  & {\ding{52}}  & {\ding{52}} & {\ding{52}}  \\ \hline
\end{tabular}
}
}
\end{table}

{As a special case of mobile users, vehicles can act as FL clients to collectively train a model for vehicular applications. Along this line, Ye \textit{et al.} in~\cite{8964354} proposed a contract-theory-based vehicle selection approach that accounts for image quality and heterogeneous vehicle capabilities. 
In~\cite{pervej2023resource}, Pervej \textit{et al.} presented a resource allocation and client selection framework and analyzed its learning performance under full and partial vehicle participation. 
Zhao~\textit{et al.} proposed Newt, an improved vehicle selection mechanism with feedback control by considering data and resource heterogeneity in dynamic environments~\cite{zhao2022participant}. To ensure timely completion of FL iterations within latency constraints, Zhao \textit{et al.} in~\cite{zhao2021system} maximized vehicle participation by accounting for dynamic wireless channels and heterogeneous computing capacities. 
In~\cite{xiao2021vehicle}, Xiao \textit{et al.} proposed a min-max optimization framework that selects vehicles based on image quality while minimizing the overall system cost in FL. 
\rev{Zhang \textit{et al.} in~\cite{zhang2023vehicle} introduced a mobility- and channel dynamic-aware FL scheme, which enables road side unit (RSU) to select appropriate vehicles and weightedly average the local models to improve the FL performance in vehicular networks.}
\rev{Zhang \textit{et al.} in \cite{10643168} investigate the joint optimization of vehicle selection, training time, and model quantization of FL with gradient quantization in vehicle edge computing by considering the mobility and the uncertainty of channel conditions.}
Furthermore, Zheng~\textit{et al.}~\cite{zheng2023autofed} introduced AutoFed, \rev{a heterogeneity-aware FL framework that leverages multimodal sensory data to improve object detection performance in autonomous vehicles and incorporates a vehicle selection mechanism based on model similarities to enhance training stability.}}

However, the aforementioned works assume vehicle training data is location-independent, similar to traditional FL frameworks. In reality, since vehicles proactively collect data from their surroundings, their data distributions are dependent on their routes. In FL, the data distribution of clients plays an essential role in learning performance, as deviation from the desired distribution can introduce biases during model training, resulting in severe accuracy degradation~\cite{zhao2018federated,li2019convergence}. \rev{To fill this research gap, this work provides a rigorous convergence analysis and devises a unified framework for trajectory-aware vehicle selection and training data collection to enhance learning performance in vehicular FL.
To compare our work and related works, we provide a summary table in Table \ref{table_comp}.}

\section{The Sense4FL Framework}
\label{sec:framework}

In this section, we elaborate on the Sense4FL framework, including the system model, specifically the vehicular mobility and the training data collection model, the federated learning procedure, and the overview of Sense4FL.
% in AD. Under this innovative framework, the data collected by vehicles for local training and the probability of successful model uploading are both influenced by their trajectories. The modeling of the vehicular crowdsensing network, FL procedures, and latency analysis are presented below.

\begin{figure}
    \centering
    \includegraphics[width=0.82\linewidth]{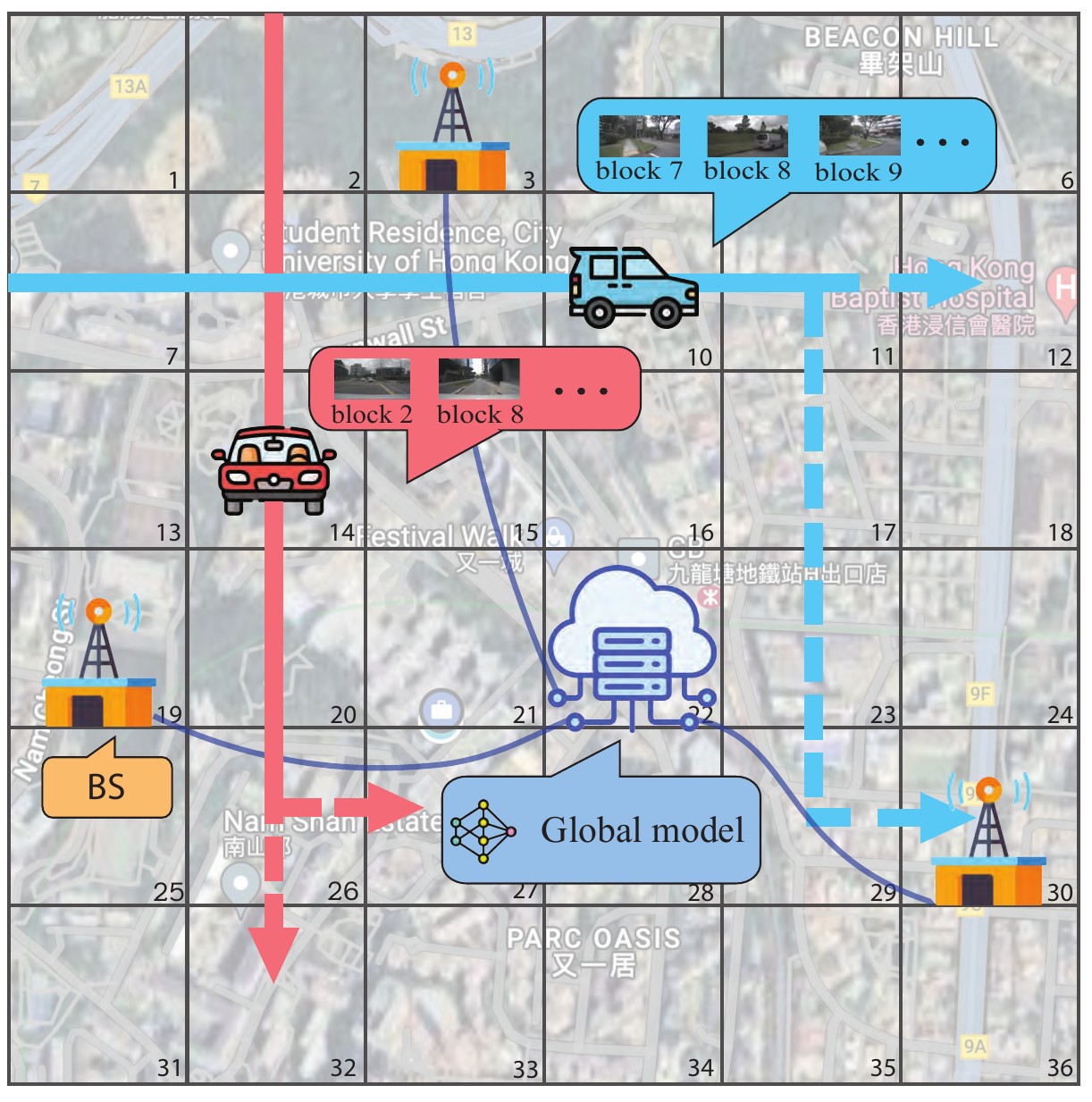}
    % \captionsetup{font={small}}
    \caption{Illustration of Sense4FL framework. Each autonomous vehicle acts as a mobile data collector and an FL client. As a vehicle traverses through a region, it collects data about street blocks and then leverages the collected data to train local models for FL. The selection of vehicles participating in the training process and the timing of starting their training are carefully designed by considering the impact of vehicles' uncertain trajectories.}
    \label{fig:system}
    \vspace{-0.0cm}
\end{figure}

\subsection{System Model}

\rev{As illustrated in Fig. \ref{fig:system}, we consider a region of interest consisting of multiple street blocks, where an FL server is responsible for model aggregation. Each autonomous vehicle serves as both a mobile data collector and an FL client, which collects training data from road environments, i.e., images of road conditions, traffic signs, pedestrians, and vehicles, upon traversing through a street block. The FL procedure selects a subset of vehicles from this region, leveraging their onboard computing capabilities and locally collected datasets to train an object detection model to enhance AD performance. We assume vehicles can upload their models to the FL server via cellular networks, i.e., any associated base station (BS), which then forwards the models to the FL server for aggregation through wired links.}

\subsubsection{Vehicular mobility}
\rev{Let $\mathcal{B} = \{1, 2, \dots, B\}$ and $\mathcal{V} = \{1, 2, \dots, V\}$ denote the sets of street blocks and vehicles, respectively. 
The set of sojourn time can be represented by $\mathcal{T} = \{t_{1,1}, t_{1,2}, ..., t_{V,B}\}$, where $t_{v,b}$ denotes the sojourn time of vehicle $v$ in street block $b$\footnote{\rev{$t_{v,b}$ can be estimated by the FL server based on historical traffic information, which follows a probability distribution, e.g., a truncated Gaussian distribution as often adopted in~\cite{yousefi2008analytical},~\cite{abuelenin2014empirical}.}}.}
As shown in Fig. \ref{fig:system}, a vehicle's trajectory can be modeled as a sequence of street blocks. The trajectory of vehicle $v$ can be denoted by $h_v = \{s_{v}^{1}, s_{v}^{2}, ..., s_{v}^{N_v}\}$, where $s_{v}^{n}$ represents the $n$-th element in this route with $N_v$ street blocks in total. Each $s_{v}^{n}$ corresponds to a street block in the set $\mathcal{B}$. In practice, an FL server cannot directly know the exact itinerary of a vehicle but can only predict the vehicle's trajectories based on its location and driving direction. Consequently, a vehicle may have multiple potential routes. Specifically, the set of possible trajectories for vehicle $v$ can be expressed as $\mathcal{H}_v = \{h_{v,1}, h_{v,2}, \ldots, h_{v,M_v}\}$, where $h_{v,m} = \{s_{v,m}^1, s_{v,m}^2, \ldots, s_{v,m}^{N_{v,m}}\}$ represents the $ m $-th potential trajectory out of $M_v$ trajectories. The probability of vehicle $v$ to choose trajectory $h_{v,m}$ is $q_{v,m}$.

%\cite{alnagar2019towards}. 
%Consequently, the residence time of each vehicle in a certain street block, as well as within the coverage area of the current edge node, can be determined.

%Without loss of generality, we assume that vehicles velocities are independent and identically distributed since drivers can choose their appropriate speeds according to the road conditions and 

\subsubsection{Training data collection}
%\rev{add another subsection of data collection? and emphasize why we optimize the data collection.}
We consider a multi-class object detection task~\cite{quemeneur2024fedpylot}. Each training data sample is represented by $(\mathbf{x}, y)$. Here, $\mathbf{x} \in \mathbb{R}^d$ is the input feature vector belonging to a compact space $\mathcal{X}$ whereas $y \in \mathbb{R}$ is the corresponding label from the label set $\mathcal{Y} = \{1, 2, \ldots, C\}$. With different street scenes, we assume different street blocks follow heterogeneous data distribution, e.g., with different numbers of pedestrians, vehicles, or other objects. By defining the probability of a data sample in street block $b$ belonging to class $i$ as $p_{b}^{i}$, the global data distribution for the entire region can be expressed as follows
\begin{equation}   p^{i}=\frac{\sum_{b=1}^B Q_b p_{b}^{i}}{\sum_{b=1}^B Q_b},
\end{equation}
where ${Q}_b$ is the average number of objects a vehicle encounters in street block $b$. It is noted that $p_{b}^{i}$ can be estimated in our systems based on public-domain information, such as satellite imagery and traffic information, or by requesting some vehicles to upload the statistical data (instead of the raw data) of the street block\footnote{\rev{In this paper, we consider data heterogeneity at the ``street block" level. However, the proposed Sense4FL framework can also be applied to broader spatial scales, e.g., distinctions between downtown and residential areas.}}.

%When a vehicle traverses a street block, its collected data is drawn from the street block, i.e., following the distribution of $p_{b}^{i}$.
Vehicles construct their training datasets along their routes, where the data collected in block $b$ follows the distribution $p_b^i$ of that block.
If vehicle $v$ collects the training data from the first ${g}_{v,m}$ street blocks in its trajectory $h_{v,m}$, the set of street blocks can be denoted by $h_{v,m}({g}_{v,m})= \{s_{v,m}^1,s_{v,m}^2,\ldots,s_{v,m}^{{g}_{v,m}}\}$, and the distribution of its collected dataset can be expressed as
\begin{equation}
    p_{v,m}^i = \frac{\sum_{b \in h_{v,m}({g}_{v,m})} Q_b p^{i}_b}{\sum_{b \in h_{v,m}({g}_{v,m})}Q_b},
\end{equation}
which will be used for training as detailed in the subsequent description.

%\subsection{The Overview of Sense4FL}
\subsection{The Federated Learning Procedure}

% \subsection{The Federated Learning Procedure}

%For model updating, 

The goal of Sense4FL is to derive the global model $\mathbf{w}$ to minimize the global loss function at the FL server
\begin{equation}
    \mathcal{F}(\mathbf{w}) \triangleq \sum_{b=1}^{B}l_b\mathcal{F}_b(\mathbf{w}),
\end{equation}
where 
\begin{equation}
\begin{aligned}
    \mathcal{F}_b(\mathbf{w}) 
    &=\sum_{i=1}^C p_{b}^{i}\mathbb{E}_{\mathbf{x}_b^i}\left[f(\mathbf{w},\mathbf{x}_b^i)\right]
\end{aligned} 
\end{equation}
denotes the local loss function for street block $b$, $l_b$ is the weighting factor with $\sum_{b=1}^{B} l_b=1$, and $f(\mathbf{w}, \mathbf{x}_b^i)$ denotes the loss function for samples of class $i$ in street block $b$. For object classification/detection tasks in AD, the local loss function for street block $b$ can be cross-entropy loss, logistic regression, or a combination of them~\cite{wang2023yolov7}. Moreover, the weighting factor can be determined according to the vehicle traffic density (i.e., how many vehicles will encounter the situation) or accident probability. Besides, when there are stringent road safety requirements, like in a school zone, a higher weighting factor can be assigned\footnote{Without loss of generality, we assume that samples have the same weighting factor if they are in the same street block. Our framework, however, can be easily extended to various weighting settings.}.
%\footnote{As mentioned earlier, we consider multi-class object detection in this paper, which is a fundamental block for AD systems.}

\begin{table}[!t]
\centering
\caption{Summary of important notations.}
\label{table:notations}
\renewcommand{\arraystretch}{1.4}
\setlength{\tabcolsep}{2mm}
\begin{tabular}{@{}p{0.9cm}p{7.0cm}@{}}
\hline
\textbf{Notation} & \textbf{Description} \\ 
\hline
~~$\mathcal{B}$ & The set of street blocks\\
~~$\mathcal{V}$ & The set of vehicles\\
~~$Q_b$ & The dataset size of street block $b$ \\
~~$h_{v,m}$ & The $ m $-th potential trajectory of vehicle $v$ \\
~~$q_{v,m}$ & The probability of vehicle $v$ to choose trajectory $ h_{v,m} $ \\
~~$t_{v,b}$ & The sojourn time of vehicle $v$ in street block $b$ \\
~~$p_b^{i}$ & The probability of a data sample in street block $b$ belonging to class $i$ \\
% ~~$p^{i}$ & The probability of a data sample belonging to class $i$ in the entire region  \\
~~$p_{v,m}^i$ & The probability of a data sample belonging to class $i$ in the dataset collected by vehicle $v$
from trajectory $h_{v,m}({g}_{v,m})= \{s_{v,m}^1,s_{v,m}^2,\ldots,s_{v,m}^{{g}_{v,m}}\}$\\
% ~~$\mathcal{F}(\mathbf{w})$ & The global loss function  \\
~~~$\mathbf{w}_{v}^{(k),T}$ & The local model uploaded by vehicle \(v\) in round $k$ \\
~~~$\mathbf{w}_{\mathrm{f}}^{(k),T}$ & The aggregated FL model in round $k$ \\
~~$\mathbf{w}^\star$ & The optimal global model \\
~~~$q_{v,m}^{(k),\mathrm{rcv}}$ & The probability of successfully receiving the local trained model from vehicle \( v \) on its \( m \)-th trajectory in round \( k \)
\\
~~$z_{v,m}^{(k)}$ & The indicator representing vehicle \( v \) selecting trajectory \( h_{v,m} \) in round \( k \) \\
~~$e_{v,m}^{(k)}$ & The indicator of successfully receiving the local trained model from vehicle $v$ in round $k$ \\
~~$l_b$ & The weighting factor of street block $b$ \\
~~$\rho_v^{(k)}$ & The weighting factor for the model uploaded by vehicle $v$ in round $k$\\
~~$a_{v}$ & The vehicle selection decision variable \\
~~$\mathbf{g}_v$  & The data collection decision variable for vehicle $v$\\
 \hline
\end{tabular}
\end{table}

\subsubsection{\rev{Local model updating}}
\rev{To obtain the desired global model $\mathbf{w}$, the FL server selects vehicles to participate in each round. Let $\mathcal{V}^{(k)}$ denote the set of available vehicles in the $k$-th round, with the cardinality being $V^{(k)}$. Also, we define the vehicle selection decision variable as $a_{v}^{(k)}\in \{0,1\}$, where $a_{v}^{(k)} = 1$ indicates vehicle $v$ is selected in round $k$ and $a_{v}^{(k)}=0$ otherwise. After selection, the FL server broadcasts the current global model $\mathbf{w}^{(k)}$ to the selected vehicles. 
If vehicle $v$ collects the training data in the trajectory $h_{v,m}$, the local loss function is given by}
\begin{equation}
\begin{aligned}
    \mathcal{F}_{v,m}(\mathbf{w}_{v,m}^{(k)}) =\sum_{i=1}^C p_{v,m}^{i,(k)}\mathbb{E}_{\mathbf{x}_{v,m}^i}\left[ f(\mathbf{w}_{v,m}^{(k)}, \mathbf{x}_{v,m}^i)\right],
\end{aligned}
\end{equation}
\rev{where $\mathbf{w}_{v,m}^{(k)}$ is the local  model in round $k$ and $\mathbf{x}_{v,m}^i$ is the $i$-class dataset collected by vehicle $v$ in the trajectory $h_{v,m}$.}

Each selected vehicle updates its local model by performing $T$ steps of the local stochastic gradient descent (SGD) update~\cite{stich2018local}. 
% \rev{Assuming vehicle $v$ chooses the trajectory $h_{v,m}$, 
The local update at step $t$ can be computed via
\begin{equation}\label{eq:wv_update}
\begin{aligned}
       \mathbf{w}_{v,m}^{(k),t+1}=& \mathbf{w}_{v,m}^{(k),t}-\eta\sum_{i=1}^C p_{v,m}^{i,(k)}  \nabla_{\mathbf{w}} \mathbb{E}_{\mathbf{x}_{v,m}^i}\left[ f(\mathbf{w}_{v,m}^{(k),t}, \mathbf{x}_{v,m}^i)\right], 
\end{aligned}
\end{equation}
where $\eta$ is the learning rate.

\rev{
At the end of the $k$-th round, 
% each selected vehicle obtains the updated local model $\mathbf{w}_{v,m}^{(k), T}$ and sends it to the nearby BS. 
% In the $k$-th round, 
the resulting model uploaded by vehicle \(v\) is hence given by
\begin{equation}\label{eq:w_vm}
     \mathbf{w}_{v}^{(k),T} = \sum_{m=1}^{M_v^{(k)}}\frac{
     z_{v,m}^{(k)}e_{v,m}^{(k)}}{\sum_{m=1}^{M_v^{(k)}}z_{v,m}^{(k)}e_{v,m}^{(k)}}\mathbf{w}_{v,m}^{(k),T},
\end{equation}
where 
\begin{equation}
   z_{v,m}^{(k)}= \begin{cases}1, &\text {with probability } q_{v,m}^{(k)}, \\ 0, &\text{otherwise},\end{cases}
\end{equation}
and $q_{v,m}^{(k)}$ is the probability that vehicle \( v \) follows the trajectory $h_{v,m}$ in round \( k \) and}
\begin{equation}
   e_{v,m}^{(k)}= \begin{cases}1, &\text {with probability } 
 q_{v,m}^{(k),\mathrm{rcv}}, \\ 0, &\text{otherwise},\label{upload_prob}
 \end{cases}
\end{equation}
where $q_{v,m}^{(k),\mathrm{rcv}}$ denotes the probability of successful reception of the local trained model from vehicle $v$, and we will show how to derive it in Section \ref{sec:problem}.

\subsubsection{\rev{Model aggregation}} 
The FL server aggregates local models successfully uploaded by vehicles within the time constraint. Consequently, the global FL model aggregation can be written as
\begin{equation} \label{eq:global_ini}    
\mathbf{w}_{\mathrm{f}}^{(k),T}=\sum_{v=1}^{ V^{(k)}}\frac{a_{v}^{(k)} \rho_v^{(k)}}{\sum_{v=1}^{ V^{(k)}}a_{v}^{(k)} \rho_v^{(k)}}\mathbf{w}_{v}^{(k),T},
\end{equation}
% where $\mathbf{a}^{(k)}\triangleq [a_{1}^{(k)}, a_{2}^{(k)},...,a_{V^{(k)}}^{(k)}]^T$ is the vector of the vehicle selection index, $\{\mathbf{g}_v^{(k)}\}$ represents the set of data collection decision variables for vehicles, and  
where $\rho_v^{(k)}$ is the weighting factor for the model uploaded by vehicle $v$, which can be obtained by combining the weighting factor $l_b$ of traversed street blocks as follows
\begin{equation}
    \rho_v^{(k)} = \sum_{m=1}^{M_v^{(k)}}q_{v,m}^{(k)}\sum_{b\in h_{v,m}({g}_{v,m}^{(k)})}l_{b}.
\end{equation}

For readers' convenience, the important notations in this paper are summarized in Table \ref{table:notations}.

%At the beginning of each FL training round, an edge node runs our vehicle selection and data collection algorithm. 
\begin{remark}
    In FL for AD, the data collection framework plays a pivotal role in training performance.  Vehicle selection and data collection influence not only the probability of model uploading but also the directions of local updates. As a result, without judicious design, FL may not reflect the global data distribution in the region of interest, resulting in model bias and poor training accuracy.
    % The data collected by vehicles is inherently linked to their trajectories.
    %Traditional FL client selection schemes often overlook these critical aspects. 
    %However, in real-world vehicular FL scenarios, optimizing a trajectory-dependent data collection framework and vehicle selection process is essential to address these challenges. By carefully designing such frameworks, we can mitigate data biases and enhance the robustness and accuracy of the global model, representing a significant contribution to advancing FL in vehicular networks.
\end{remark}

\begin{figure}
    \centering
    \includegraphics[width=0.98\linewidth]{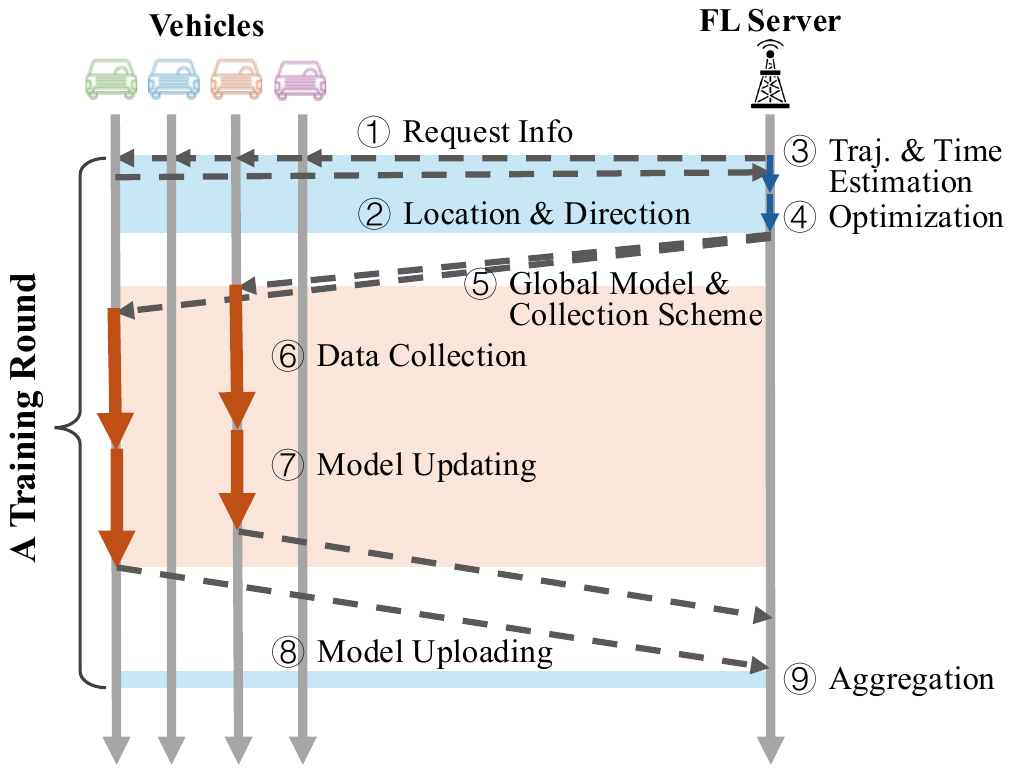}
    \caption{The workflow of Sense4FL framework.}
    \label{fig:workflow}
    % \vspace{-0.1cm}
\end{figure}

\subsection{The Overview of Sense4FL Framework}
In summary, the proposed Sense4FL workflow, as illustrated in Fig. \ref{fig:workflow}, consists of the following steps for each training round.

\begin{enumerate}[noitemsep]
\item The FL server initiates the process by broadcasting probing information to vehicles within the region. Vehicles that are interested in participating respond by uploading their location and direction data to the FL server through the cellular network (via the associated BS).

\item Based on the received location and direction information, the FL server estimates the trajectories and travel times of the responding vehicles.

\rev{\item The FL server selects a subset of vehicles to participate in the FL training and determines their data collection blocks according to our algorithm. It then transmits the decisions and the global model to the selected vehicles.}

\item The chosen vehicles proceed to collect data according to the specified policy, train their models, and then upload their trained models back to the FL server for aggregation.

\item The FL server aggregates the received models and prepares for the next round of training.
\end{enumerate}

\section{Convergence Analysis of Sense4FL}
\label{sec:convergence}
%and the local model at vehicle $v$ are denoted by and $\mathbf{w}_v \in \mathbb{R}^d$, respectively. 
In Sense4FL, one fundamental problem is how vehicle selection, data collection, and the successful model uploading probability could impact model convergence. Understanding these factors lays the foundation for subsequent optimization. In this section, we derive an upper bound on convergence as a function of the aforementioned factors.

We first introduce four widely used assumptions on loss function  $\mathcal{F}(\cdot)$~\cite{wang2019adaptive,zhang2024coalitional}:

\textbf{Assumption 1.} $\mathcal{F}(\cdot)$ is convex\footnote{Assumption 1 holds for AI models, including squared-SVM and linear regression models. The experimental results in Section \ref{sec:exp} also demonstrate that our algorithm works effectively for deep neural networks not satisfying Assumption 1.}.

\textbf{Assumption 2.} $\mathcal{F}(\cdot)$ is $\beta$-Smooth, i.e., for any $\mathbf{w}$ and $ \mathbf{w}^{\prime}$,
$\mathcal{F}(\mathbf{w}) \leq \mathcal{F}(\mathbf{w}^{\prime})+\nabla \mathcal{F}(\mathbf{w}^{\prime})^T(\mathbf{w}-\mathbf{w}^{\prime})+\frac{\beta}{2}\left\|\mathbf{w}-\mathbf{w}^{\prime}\right\|^2.$

\textbf{Assumption 3.} $\mathcal{F}(\cdot)$ is $L$-Lipschitz, i.e., for any $\mathbf{w}$ and $\mathbf{w}^{\prime}$, $\left\|\mathcal{F}(\mathbf{w})-\mathcal{F}\left(\mathbf{w}^{\prime}\right)\right\| \leq L \| \mathbf{w}-\mathbf{w}^{\prime} \|$.
    
\textbf{Assumption 4.} $\nabla_{\mathbf{w}} \mathbb{E}_{\mathbf{x}^i}\left[f(\mathbf{w}, \mathbf{x}^i)\right]$ is $\lambda_i$-Lipschitz for each class $i \in \mathcal{Y}$, i.e., for any $\mathbf{w}$ and $\mathbf{w}^{\prime}$, $\| \nabla_{\mathbf{w}}\mathbb{E}_{\mathbf{x}^i}\left[ f(\mathbf{w}, \mathbf{x}^i)\right]-$ $\nabla_{\mathbf{w}} \mathbb{E}_{\mathbf{x}^i}\left[ f\left(\mathbf{w}^{\prime}, \mathbf{x}^i\right)\right]\left\|\leq \lambda_i\right\| \mathbf{w}-\mathbf{w}^{\prime} \|$.

Based on the assumptions above, the convergence upper bound is provided below.

\begin{theorem}\label{theo:FL}
    Under Assumptions 1-4 and when the following conditions hold:\\
1) $\eta \leq \frac{1}{\beta}$\\
2) $\eta<\frac{2}{\beta}\left(1-\frac{L U \sum_{k=1}^K\Omega^{(k)}}{KT \phi \epsilon^2}\right)$\\
3) $\mathcal{F}(\mathbf{w}_{\mathrm{c}}^{(k),T})-\mathcal{F}(\mathbf{w}^\star)\geq \epsilon$, $\forall k$\\
4) $\mathcal{F}(\mathbf{w}_\mathrm{f}^{(K),T}) -\mathcal{F}(\mathbf{w}^\star)\geq \epsilon$\\
for $\epsilon >0$, where $U \triangleq \max_{k}\max_{j}\mu_{\max}(\mathbf{w}_{\mathrm{c}}^{(k),j})$, 
{$\mu_{\max}(\mathbf{w}_{\mathrm{c}}^{(k), j})\triangleq \max_{i=1}^C\|\nabla_\mathbf{w} \mathbb{E}_{\mathbf{x}_b^i}[f(\mathbf{w}_{\mathrm{c}}^{(k), j},\mathbf{x}_b^i)]\|$ represents the maximum norm of the expected gradient for each class of model $\mathbf{w}_{\mathrm{c}}^{(k), j}$, $\mathbf{w}_{\mathrm{c}}^{(k), j}$ is the model obtained at step $j$ in the $k$-th round under centralized training,
$\phi\triangleq\min_{k}\frac{1}{\|\mathbf{w}_{\mathrm{c}}^{(k),1}-\mathbf{w}^{\star}\|^2}$,} and $\mathbf{w}^\star$ is the optimal global model, the training loss of Sense4FL satisfies
\begin{equation}\label{eq:loss}
    \begin{aligned}
        &\mathcal{F}(\mathbf{w}_\mathrm{f}^{(K),T}) -\mathcal{F}(\mathbf{w}^\star)\\
        \leq& \frac{1}{\eta(\phi KT(1-\frac{\beta\eta}{2})-\frac{L}{\epsilon^2}
         U\sum_{k=1}^{K}\Omega^{(k)})},
    \end{aligned}
\end{equation}
where
\begin{equation}
    \begin{aligned}
    \Omega^{(k)}\triangleq 
    &\delta\sum_{v=1}^{ V^{(k)}}\frac{a_{v}^{(k)} \rho_v^{(k)}}{\sum_{v=1}^{ V^{(k)}}a_{v}^{(k)} \rho_v^{(k)}}\underbrace{\sum_{m=1}^{M_v^{(k)}}\overline{\xi_{v,m}^{(k)}}
    \sum_{i=1}^C \|p_{v,m}^{i,(k)}- \sum_{b=1}^B l_b p_{b}^{i}\|}_{\text{client divergence}}\\
    &+\underbrace{\sum_{i=1}^C \|\sum_{v=1}^{ V^{(k)}}\frac{a_{v}^{(k)} \rho_v^{(k)}}{\sum_{v=1}^{ V^{(k)}}a_{v}^{(k)} \rho_v^{(k)}}\sum_{m=1}^{M_v^{(k)}}\overline{\xi_{v,m}^{(k)}}p_{v,m}^{i,(k)}- \sum_{b=1}^B l_b p_{b}^{i}\|}_{\text{global divergence}}\\
    =&\delta\sum_{v=1}^{ V^{(k)}}\frac{a_{v}^{(k)} \rho_v^{(k)}}{\sum_{v=1}^{ V^{(k)}}a_{v}^{(k)} \rho_v^{(k)}}{D_{v,\text{client}}^{(k)}}+{D_{\text{global}}^{(k)}},
    \end{aligned}
\end{equation}
with
${D_{v,\text{client}}^{(k)}}\triangleq\sum_{m=1}^{M_v^{(k)}}\overline{\xi_{v,m}^{(k)}}\sum_{i=1}^C \|p_{v,m}^{i,(k)}- \sum_{b=1}^B l_b p_{b}^{i}\|
$, ${D_{\text{global}}^{(k)}}\triangleq \sum_{i=1}^C \|\sum_{v=1}^{ V^{(k)}}\frac{a_{v}^{(k)} \rho_v^{(k)}}{\sum_{v=1}^{ V^{(k)}}a_{v}^{(k)} \rho_v^{(k)}}\sum_{m=1}^{M_v^{(k)}}\overline{\xi_{v,m}^{(k)}}p_{v,m}^{i,(k)}- \sum_{b=1}^B l_b p_{b}^{i}\|$,   $\delta \triangleq \sum_{j=1}^{T-1}(1+\eta \lambda_{\max})^{j}$ with $\lambda_{\max}\triangleq
 \max_i\lambda_i$, and $\overline{\xi_{v,m}^{(k)}} \triangleq\frac{q_{v,m}^{(k)}q_{v,m}^{(k),\text{rcv}}}{\sum_{m=1}^{M_v^{(k)}}{q_{v,m}^{(k)}q_{v,m}^{(k),\text{rcv}}}}$.
\end{theorem}
 % being the maximum among each class.
%Intuitively, if the vehicle's data distribution differs significantly from the region distribution, the EMD will be large.

\begin{proof}
    Please refer to Appendix \ref{app:converge}.
\end{proof}

%In \textbf{Theorem 1.}, $\mathbf{w}_\mathrm{f}^{(K), T}$ represents the global Sense4FL model, which is an aggregation of the local models from the selected vehicles, while $\mathbf{w}^\star$ denotes the optimal FL model under ideal settings, $\phi, U$, and $\epsilon$ are constants, and $\delta$ is a parameter which is related to Lipschitz constant and the number of local steps $T$. 

We remark that $\epsilon>0$ in Conditions (3) and (4) in Theorem \ref{theo:FL} follows from the convergence lower bound of gradient
descent given in Theorem 3.14 in~\cite{bubeck2015convex}.
Some key observations can be made based on \textbf{Theorem \ref{theo:FL}}:

\textbf{Observation 1.} {${D_{v,\text{client}}^{(k)}}$ denotes the client divergence (i.e., the weighted EMD), which measures the divergence between the data distribution of vehicle $v$ and the data distribution of the entire region, and ${D_{\text{global}}^{(k)}}$ denotes the global weighted EMD, which measures the divergence between the data distribution of all selected vehicles and the data distribution of the region. Since $\phi, U$, and $\epsilon$ are independent of our decision variables, the upper bound of training loss is proportional to $\Omega^{(k)}$, which depends on the client divergence and global divergence, i.e., \(D_{v,\text{client}}^{(k)}\) and \(D_{\text{global}}^{(k)}\), of the selected vehicles. This indicates that minimizing the upper bound of training loss is equivalent to minimizing $\Omega^{(k)}$.}

\rev{Intuitively, a large client divergence makes FL hard to converge, while a large global divergence makes the distribution of the training dataset different from that of the test dataset. As a result, minimizing the combined objective with these two metrics leads to smaller training loss, which matches our intuition.}

%that a larger weighted EMD for the selected vehicles results in worse performance for Sense4FL.

\textbf{Observation 2.} 
% $\Omega^{(k)}$ depends on client divergence \(D_{v,\text{client}}^{(k)}\) and global divergence \(D_{\text{global}}^{(k)}\). 
To reduce \(D_{v,\text{client}}^{(k)}\), the data distribution of each selected vehicle should be close to the overall data distribution in the region.  To reduce \(D_{\text{global}}^{(k)}\), the combined data distribution of all selected vehicles should be close to the global regional data distribution. Intuitively, a small client divergence ensures that each local model will not diverge significantly, whereas a small global divergence ensures that all the selected vehicles can collectively train a global model that better represents the region of interest.

\textbf{Observation 3.} 
$\delta$ is a parameter related to the number of local steps $T$. If a vehicle conducts one SGD step, i.e., $T=1$, then $\delta$ equals $0$, yielding $\Omega^{(k)} = D_{\text{global}}^{(k)}$. In this scenario, the process is equivalent to the centralized training. For this reason, only the global divergence $D_{\text{global}}^{(k)}$ matters, i.e., we only need to pay attention to the combined data distribution of all selected vehicles.

\section{Problem Formulation}
\label{sec:problem}
%\textcolor{blue}{However, mobility is an intrinsic characteristic of the vehicles, which cannot be directly adjusted to improve the FL performance. Alternatively, the model performance can be enhanced by optimizing the vehicle selection and training data collection. To this end, we will elaborate on}
The previous analysis has demonstrated that the trajectories of vehicles greatly affect training data distribution and uploading opportunities. In this section, we formulate the optimization problem of minimizing the training loss for this region by jointly optimizing vehicle selection and training data collection.

\subsection{Model Reception Probability Analysis}
%As we discussed above, the time vehicles spend passing through different street blocks follows a specific distribution. In addition, vehicles may have multiple possible trajectories through the area, leading to varying total sojourn periods. A longer staying time within the coverage area ensures the completion of the training process and the successful uploading of the trained model. In particular, t

To minimize the upper bound of training loss in (\ref{eq:loss}), we first need to calculate the successful model reception probability $q_{v,m}^{(k),\mathrm{rcv}}$, which is related to latency, i.e., whether a vehicle has sufficient time to upload the model or not. \rev{When vehicle $v$ follows trajectory $h_{v,m}$ and stops data collection after traversing the first ${g}_{v,m}^{(k)}$ street blocks, it needs to train the local model and upload it to the FL server before the deadline for one training round\footnote{{We assume the delay for model aggregation is negligible, as in~\cite{shiUAV}.}}.}
% Simply speaking, the sum of local training and uploading time should be no more than the to ensure the FL server successfully receives the model

\rev{\textit{1) Local data collection:}
When vehicle $v$ stops data collection after traversing the first ${g}_{v,m}^{(k)}$ blocks along trajectory $h_{v,m}^{(k)}$, the local data collection time (DCT) $t_{v,m}^{(k),\mathrm{DCT}}$ in round $k$ can be given by
\begin{equation}
    t_{v,m}^{(k),\mathrm{DCT}}\triangleq \sum_{b \in h_{v,m}({g}_{v,m}^{(k)})}t_{v,b}^{(k)} - \sum_{b \in h_{v,m}({c}_{v,m}^{(k)})}t_{v,b}^{(k)},
\end{equation}
where $c_{v,m}^{(k)}$ represents the number of street blocks from which vehicle \( v \) has already collected data at the time of selection. DCT indicates how long the vehicle spends collecting data after the time of selection.
% $T_{v,m}^{(k),\mathrm{DCT}}({g}_{v,m}^{(k)}) \triangleq \sum_{b \in h_{v,m}}t_{v,b}^{(k)}-\sum_{b \in h_{v,m}({g}_{v,m}^{(k)})}t_{v,b}^{(k)}$ represents the data collecting time (DCT) in the area after stopping collecting training data.
% , where $t_{v,b}^{(k)}$ denotes the staying time of vehicle $v$ in street block $b$ at round $k$. 
}

\textit{2) Local model computing:}
Let $c_v$ denote the required number of processing cycles for computing one sample and $f_v$ denote the computing capability (in FLOPS) of vehicle $v$. The computing time for vehicle $v$ can be expressed as
% \begin{equation}
%     t_{v,m}^{(k),\mathrm{comp}}({g}_{v,m}^{(k)}) = \frac{c_v\lvert\mathcal{D}_{v,m}({g}_{v,m}^{(k)})\rvert}{f_v}.
% \end{equation}
\begin{equation}        
    t_{v}^{(k),\mathrm{comp}} = \frac{Tc_v{D}_\mathrm{Batch}}{f_v},
\end{equation}
where ${D}_\mathrm{Batch}$ is the batch size of local training.

\textit{3) Local model uploading:}
We adopt the Orthogonal Frequency Division Multiple Access (OFDMA) scheme for model uploading. Due to the movement of vehicles and their uncertain routes, predicting their channel state information (CSI) accurately is very challenging, if not impossible. As such, we conservatively estimate the uploading time based on the average uploading data rate at the boundary of the cell coverage. 
\rev{The local model uploading time of vehicle $v$ for trajectory $h_{v,m}^{(k)}$ is
\begin{equation}
    t_{v,m}^{(k),\mathrm{up}} = \frac{\omega}{R_{v,m}^{(k),\mathrm{min}}} + t^{\mathrm{trans}},
\end{equation}
where $\omega$ is the local model size, $R_{v,m}^{(k),\mathrm{min}}$ is the minimum expected uplink data rate within cellular coverage, and \( t^{\mathrm{trans}} \) represents the time required for the BS to transmit a model to the FL server via a wired link, which is assumed to be a constant.} Since it is hard to predict the uplink data rate of a vehicle and its associated BS in the future, we assume full cellular coverage and consider the minimum expected data rate within the coverage as a conservative measurement for model uploading time.

% During the sojourn period within the region, the vehicle should complete the computing and communication tasks. 
\rev{The local data collection, local model computing, and uploading should be finished before the deadline, which means
% The local model should be uploaded before the deadline.
% Hence, the time allocated for computing and uploading the local model should satisfy
\begin{equation}
t_{v,m}^{(k),\mathrm{DCT}}+ t_{v}^{(k),\mathrm{comp}}+t_{v,m}^{(k),\mathrm{up}}\leq T^{\mathrm{task}},
\end{equation}
where \( T^{\mathrm{task}} \) denotes the time budget for one training round. With this time constraint, we can calculate the successful reception probability $q_{v,m}^{(k),\mathrm{rcv}}$ in (\ref{upload_prob}) by
\begin{equation}
q_{v,m}^{(k),\mathrm{rcv}} = P(t_{v,m}^{(k),\mathrm{DCT}} + t_{v}^{(k),\mathrm{comp}}+t_{v,m}^{(k),\mathrm{up}} \leq T^{\mathrm{task}}).
\end{equation}
}

\subsection{Problem Formulation}\label{AA}
In our Sense4FL system, the FL server selects a subset of vehicles and determines when each vehicle should start training after collecting data from a number of street blocks. We define the vector of the vehicle selection index as $\mathbf{a}^{(k)}\triangleq [a_{1}^{(k)}, a_{2}^{(k)},...,a_{V^{(k)}}^{(k)}]^T$ and the vector of data collection decision variables  as $\mathbf{g}_v^{(k)} \triangleq [g_{v,1}^{(k)},g_{v,2}^{(k)},\ldots,g_{v,M_v^{(k)}}^{(k)}]^T, \forall v$, where $a_{v}^{(k)} = 1$ indicates that vehicle $v$ is selected at round $k$ and $g_{v,m}^{(k)}$ implies that vehicle $v$ stop data collection at the $g_{v,m}^{(k)}$-th street block (or equivalently, collect data from the first $g_{v,m}^{(k)}$ blocks) in its trajectory $h_{v,m}$. To minimize the training loss in (\ref{eq:loss}), the optimization problem is formulated as
\begin{subequations}
    \begin{align} \label{p:original}
        \min_{\mathbf{a}^{(k)}, \mathbf{g}_v^{(k)}}&\mathcal{F}(\mathbf{w}_{\mathrm{f}}^{(K),T})\\  \text{s.t.}~~&\sum_{v=1}^{V^{(k)}}a_v^{(k)}= S, ~\forall k \in \{1,...,K\},\label{p1:S}\\
        &c_{v,m}^{(k)}\leq g_{v,m}^{(k)}\leq N_{v,m}^{(k)},~\forall v \in \{1, ..., {V}^{(k)}\}, \nonumber \\
        &~~~~~~~~~ m \in \{1, ..., M_v^{(k)}\}, ~ k \in \{1, ..., K\}, \label{p1:g1}\\
        &g_{v,m}^{(k)} \in \mathbb{N}^+, ~\forall v \in \{1, ..., {V}^{(k)}\}, \nonumber \\
        &~~~~~~~~~m \in \{1, ..., M_v^{(k)}\}, ~k \in \{1, ..., K\},\label{p1:g2}\\
        &a_v^{(k)}\in\{0,1\}, ~\forall v \in \{1, ..., {V}^{(k)}\}, ~k \in \{1, ..., K\}, \label{p1:a}
    \end{align}
\end{subequations}
where $S$ denotes the number of selected vehicles.
 % and $c_{v,m}^{(k)}$ represents the number of street blocks from which vehicle \( v \) has already collected data at the time of selection. 
{Constraint \eqref{p1:S} restricts the number of selected vehicles due to limited resources, e.g., bandwidth and budget\footnote{{Since it has been empirically observed that FL performance increases with the number of clients, we use equality to enforce the exact number of vehicle selection.}}. {Constraint \eqref{p1:g1} ensures that vehicles stop data collection after the time of vehicle selection and no later than traversing all the street blocks in their trajectories.}
Considering the mobility of vehicles, the set of available vehicles in the region varies significantly across training rounds. Thus, it is generally impossible to optimize the learning performance by considering the varying vehicle selection in all training rounds~\cite{pervej2023resource}. 
For this reason, we concentrate on minimizing the upper bound of the loss in \eqref{eq:loss} or $\Omega^{(k)}$ in one round and execute our proposed algorithm for each training round, as done in~\cite{pervej2023resource}. For simplicity, we get rid of $k$ in the subsequent development. The optimization problem can be transformed into
\begin{subequations}\label{formulation}
    \begin{align}
        \min_{\mathbf{a}, \mathbf{g}_v}~&\delta\sum_{v=1}^{ V}\frac{a_{v} \rho_v}{\sum_{v=1}^{ V}a_{v} \rho_v}{D_{v,\text{client}}}+{D_{\text{global}}}\\  \text{s.t.}~&\sum_{v=1}^{V}a_v = S,\\
        &c_{v,m}\leq g_{v,m}\leq N_{v,m},\forall v \in \{1, ..., {V}\}, m \in \{1, ..., M_v\}, \\
        &g_{v,m} \in \mathbb{N}^+,~\forall v \in \{1, ..., {V}\}, ~m \in \{1, ..., M_v\},\\
        &a_v\in\{0,1\},~\forall v \in \{1, ..., {V}\}.
    \end{align}
\end{subequations}

We note that the objective function is the weighted client divergence between the data distribution of each selected vehicle and the data distribution of the region and the global divergence between the combined data distribution of all selected vehicles and the region data distribution, i.e., ${D_{v,\text{client}}}$ and ${D_{\text{global}}}$, which is particularly challenging to solve. 
% We have the following lemma for the optimization problem.

% \rev{ We have the following lemma for the optimization problem.}
% % \textit{\textbf{Lemma 1.} 
% {The above problem is a non-linear multiple-choice knapsack problem with a non-convex non-separable objective function and non-continuous variables~\cite{bretthauer2002nonlinear},~\cite{kellerer2004introduction}, which is known as NP-hard and extremely challenging to solve. In the next section, we will propose a local-search-based algorithm to obtain the solution effectively.}

%In our nonlinear multiple-choice knapsack problem, the objective function is the weighted local EMD between the data distribution of each selected vehicle and the data distribution of the region and the global EMD between the combined data distribution of all the selected vehicles and the region data distribution, making it more challenging to solve compared to the classic linear separable knapsack problem. Consequently, the formulated problem is also NP-hard, implying that as the scale of the problem increases, finding an exact optimal solution becomes increasingly intractable. Therefore, we propose an efficient algorithm based on local search to obtain a sub-optimal solution in the following.

\section{Algorithm Design}
\label{sec:algorithm}
In this section, we first show the NP-hardness of the formulated problem in Section \ref{sec:problem}. Then, we develop an efficient algorithm to solve the problem with an approximate guarantee.
 
\subsection{NP-Hardness of the Problem}
\begin{theorem} \label{theorem:np_hard}
    Problem (\ref{formulation}) is NP-hard, which can be reduced to a classical non-linear multiple-choice knapsack problem, with a non-convex non-separable objective function and non-continuous variables.
\end{theorem}

\begin{proof}
    The proof is shown in Appendix \ref{sec: app2}.
\end{proof}

The optimization problem we address follows the structure of a non-linear multiple-choice knapsack problem (MCKP)~\cite{bretthauer2002nonlinear}, which is widely recognized to be NP-hard and extremely challenging to solve. Compared to its linear counterpart, the complexity of our problem is further exacerbated by the objective function that involves the global divergence between data distribution of selected vehicles and the global data distribution. Since this global divergence results from the combined effect of selected vehicles' datasets, it is infeasible to decompose the problem into independent subproblems, rendering classic algorithms, such as dynamic programming~\cite{bellman1966dynamic}, inapplicable.
Besides, the NP-hardness implies that no algorithm can be found to solve the problem in a polynomial time.

%The selection or exclusion of any dataset alters the overall distribution in ways that cannot be captured by simple recursive formulations, a hallmark of many traditional optimization metrics. Furthermore, the non-differentiable nature of the objective function poses additional challenges, making conventional gradient descent methods unsuitable for optimizing variables $\mathbf{a}$ and $\{\mathbf{g}_v\}$. 

% Complexity: The problem is NP-hard because it generalizes the classical knapsack problem, which is already NP-hard.

% By systematically prioritizing selections that optimize the objective function and determining the ideal data collection intervals, the algorithm ensures an efficient and effective refinement of the vehicle set.

\subsection{An Approximate Algorithm} 
Given the NP-hardness of the problem, we propose a low-complexity algorithm to solve it with an approximation ratio. We begin with analyzing the objective function in (\ref{formulation})
%{\mathbf{g}_v\}}~&\delta\sum_{v=1}^{ V}a_{v} \rho_v{D_{v,\text{client}}}+{D_{\text{global}}}
\begin{equation}
    \begin{aligned}
    Obj = D_{\text{client}} +D_{\text{global}},
    \end{aligned}
\end{equation}
where
\begin{equation}
D_{\text{client}} = \delta\sum_{v=1}^{ V}\frac{a_{v} \rho_v}{\sum_{v=1}^{ V}a_{v} \rho_v}{D_{v,\text{client}}}.
\end{equation}

We observe that $D_{\text{client}}$ is a separable term of decision variables because it is the summation of the local EMD divergence of selected vehicles. In contrast, $D_{\text{global}}$ is non-separable, which is more challenging to optimize. Moreover, as both $D_{\text{client}}$ and $D_{\text{global}}$ quantify EMD divergence, we observe that the first term is at least a constant ratio of the second term, and will dominate when $\delta$ or local update step $T$ is large. This inspires us to develop a two-step optimization algorithm: 1) minimize the separable term $D_{\text{client}}$ first, for which we can achieve the minimum value of $D_{\text{client}}$ due to the tractability of the expression, and 2) use a local search procedure to improve the algorithm by evaluating $D_{\text{client}} + D_{\text{global}}$ until no improvement can be made. In this way, a good upper bound can be achieved for the minimization problem.

%We can see that when $\delta$ is larger (e.g., the local update step $T$ is large), the linear combination of local EMD of the selected vehicles, i.e., $F_1$, dominates the objective function. In addition, we can see that the objective function $F_1$ increases with the weighted local EMD $D_{v,\mathrm{client}}$, which is the linear combination of the selected vehicles. Therefore, we will ignore the intractable second term $F_2$ and only focus on the first term and we will derive the optimal solution and we will provide an analysis of the approximation guarantee. 

\subsubsection{Step 1} By focusing on $D_{\text{client}}$, the optimization problem becomes
\begin{subequations}\label{p:special}
    \begin{align} 
        \min_{\mathbf{a}, \mathbf{g}_v}~& D_{\text{client}}\\  
        \text{s.t.}~&\sum_{v=1}^{V}a_v = S,\\
        &c_{v,m}\leq g_{v,m}\leq N_{v,m},\forall v \in \{1, ..., {V}\}, m \in \{1, ..., M_v\}, \\
        &g_{v,m} \in \mathbb{N}^+,~\forall v \in \{1, ..., {V}\}, ~m \in \{1, ..., M_v\},\\
        &a_v\in\{0,1\},~\forall v \in \{1, ..., {V}\}.
    \end{align}
\end{subequations}

% \rev{We can get rid of $\mathbf{g}_v$ because the optimal value $\mathbf{g}_v^\star$ can be obtained by traversing the feasible set for each vehicle independently with a low time complexity $\mathcal{O}(N_{\mathrm{tot}})$ with $N_{\mathrm{tot}} = \sum_{v=1}^V\sum_{m=1}^{M_v}N_{v,m}$.}
We can get rid of $\mathbf{g}_v$ because the optimal value $\mathbf{g}_v^\star$ can be obtained by traversing the feasible set for each vehicle independently with a low time complexity $\mathcal{O}(\sum_{m=1}^{M_v}N_{v,m})$. The resultant problem becomes
\begin{subequations}\label{reformulation_a}
    \begin{align}
        \min_{\mathbf{a}}~& \frac{\delta\sum_{v=1}^{V}a_v\rho_{v}\Tilde{d}_v}{\sum_{v=1}^{V}a_v \rho_{v}}\\  \text{s.t.}~&\sum_{v=1}^{V}a_v= S, \label{eq:re_a_s}\\
        &a_v\in\{0,1\},~\forall v \in \{1, ..., {V}\}, \label{eq:re_a_av}
    \end{align}
\end{subequations}
where $\Tilde{d}_v = \sum_{m=1}^{M_v}\overline{\xi_{v,m}}\sum_{i=1}^C \|p_{v,m}^{i}- \sum_{b=1}^B l_b p_{b}^{i}\|$ is the client divergence with the optimal data collection scheme $\mathbf{g}_v^\star$. 
We observe that the problem is an integer fractional programming. To address this rather hard problem, we introduce an auxiliary variable $d$ and decompose Problem \eqref{reformulation_a} into two subproblems. Specifically, by introducing $d$, the optimization problem can be reformulated as follows
%We observe that the objective function increases with the selected vehicle's local EMD, $\Tilde{d}_v$, while exhibiting a non-monotonic relationship with the weighting factor $\rho_{v}$. Consequently, the objective function depends on the vehicle selection strategy $\mathbf{a}$ in a complex way.
{\begin{subequations} \label{p:original_ad}
    \begin{align}
                \min_{\mathbf{a},d}~~& d\\  \text{s.t.}~~&\frac{\delta\sum_{v=1}^{V}a_v\rho_{v}\Tilde{d}_v}{\sum_{v=1}^{V}a_v \rho_{v}}\leq d, \label{eq:ad_d}\\
                &\sum_{v=1}^{V}a_v= S, \label{eq:ad_S}\\
        &a_v\in\{0,1\},~\forall v \in \{1, ..., {V}\}. \label{eq:ad_av}
    \end{align} 
\end{subequations}

%We first introduce the main idea of the algorithm, which breaks the vehicle selection optimization problem \eqref{p:special} into two subproblems.  We then improve the vehicle selection considering $D_{client}+D_{global}$. Finally, we discuss the complexity and performance guarantee of the algorithm. 

% The main idea of solving Problem \eqref{reformulation_a} is to break it into two subproblems. 
% The first subproblem introduces an additional constraint enforcing that the objective function is no more than  

Given a fixed value of $d$, Problem \eqref{p:original_ad} reduces to a feasibility-check problem formulated as}
\begin{subequations}\label{reformulation_d}
    \begin{align}
                \text{Find}~~&{\mathbf{a}}\\
                \text{s.t.}~~&\eqref{eq:ad_d}, \eqref{eq:ad_S},\eqref{eq:ad_av}.
    \end{align} 
\end{subequations}

Under a fixed $d$, Constraint \eqref{eq:ad_d} is equivalent to
\begin{equation}
    \sum_{v=1}^{V}a_v\rho_{v}(\delta\Tilde{d}_v-d)\leq 0. \label{eq: nnew_obj}
\end{equation}

To solve Problem \eqref{reformulation_d}, we sort vehicles in ascending order of $\rho_{v}(\delta\Tilde{d}_v-d)$ and select the first $S$ vehicles. If such a vehicle selection strategy satisfies \eqref{eq: nnew_obj}, then Problem \eqref{reformulation_d} has feasible solutions.
Note that in this process, we use the metric $\rho_{v}(\delta\Tilde{d}_v-d)$ to comprehensively capture the effect of both the local EMD $\Tilde{d}_v$ and the weighting factor $\rho_{v}$.

The second subproblem is to find the minimum value of $d$, denoted by $d^\dag$, under which there is a feasible solution to \eqref{reformulation_d}. Since $\Tilde{d}_v \in [0, 2],1\leq v\leq {V}$, we have $d\in [0, 2\delta]$. To compute $d^\dag$, we use the bisection method over $[0,2\delta]$, with an error tolerance of $\sigma$. Given $d^\dag$, vehicles are sorted in ascending order of $\rho_{v}(\delta\Tilde{d}_v-d^\dag)$, and the first \( S \) vehicles are selected. This selection strategy yields the optimal solution to Problem \eqref{reformulation_a}, with $d^\dag$ being the optimal objective value. We have the following Theorem.

\begin{theorem} \label{theorem:optimal}
    The vehicle selection strategy is the optimal solution to Problem \eqref{reformulation_a}.
\end{theorem}
% \begin{proof}
%     The proof is shown in Appendix \ref{sec:selection}.
% \end{proof}

\begin{proof}
    For a given $d$, if there exists a solution $\mathbf{a}$ satisfying \eqref{eq:ad_d}-\eqref{eq:ad_av}, then for any $d^{\prime}>d$, there also exists a solution $\mathbf{a}^\prime$ satisfying \eqref{eq:ad_d}-\eqref{eq:ad_av}. On the other hand, if for a given $d$ where problem \eqref{reformulation_d} does not have a feasible solution, then for any $d^{\prime}<d$, it also has no feasible solution. Therefore, we can obtain the minimum $d$ making Problem \eqref{reformulation_d} feasible based on a bisection method, which corresponds to the optimal objective value of Problem (\ref{p:original_ad}). Thus, the corresponding vehicle selection strategy \( \mathbf{a} \) is the optimal solution to Problem (\ref{p:original_ad}), and hence (\ref{reformulation_a}).
    The proof is completed.
\end{proof}

%We summarize our proposed vehicle selection algorithm in Algorithm \ref{alg:optimal}. 

\begin{algorithm}[t] %其中这里面不能有H不然会报错，不过不影响结果
	\caption{Vehicle Selection and Data Collection Algorithm for Sense4FL}%算法名字
 \label{alg:optimal}
	\LinesNumbered %要求显示行号
	\KwIn{$\mathcal{V}$, $\mathcal{T}$, $T^{\mathrm{task}}$, $S$, $p_{b}^i$, $l_b$, $\mathcal{H}_v$, $N_{\max}, \delta, \sigma$}%输入参数
	\KwOut{$\mathcal{S}^\star$ and $\mathbf{g}_{v}^\star$}%输出
$\text{/}**~\text{STEP 1. solution by sorting} **\text{/}$\\
Set $d_l = 0$, $d_r = 2\delta$\;
\While{$d_r - d_l \geq \sigma$}{
$d \gets (d_r+d_l)/2$\;
Optimize $\mathbf{g}_v$ to minimize $\rho_{v}(\delta\Tilde{d}_v-d)$\;
% Calculate $\rho_{v}(\delta\Tilde{d}_v-d), \forall v $\;
Sort vehicles in ascending order of $\rho_{v}(\delta\Tilde{d}_v-d)$ and select the first $S$ vehicles as $\mathcal{S}_0$\;
Calculate $D_{\text{client}}^\dag \gets \sum_{v=1}^{S}\rho_{v}(\delta\Tilde{d}_v-d)$\;
\eIf{$D_{\text{client}}^\dag \leq 0$}{
    $d_r \gets d$\;
}{
    $d_l \gets d$\;
}
}
% ~\\
$\text{/}**~\text{STEP 2. solution improvement by local search} **\text{/}$\\
Calculate $Obj^\ast$ based on $\mathcal{S}_0$\;
Set $c^\ast = \infty$, $i =1$, $\mathcal{S}^\star \gets \mathcal{S}_0$ \;
\While{$c^\ast \neq 0$ and $i \leq N_{\max}$}{
$i \gets i+1$\;
\For{$v \in \mathcal{S}_0$}{
     $\mathcal{J} \gets \mathcal{V} \setminus \mathcal{S}_0$\;
     $c^\ast \gets 0$\;
        \For{$c \in \mathcal{J}$}{
        $\mathcal{S}^{\prime} \gets \mathcal{S}_0 \setminus \{v\} \cup \{c\}$\;
        Optimize $\mathbf{g}_v$ by traversing the feasible set\;
        Calculate $Obj^\prime$ based on $\mathcal{S}^{\prime}$\;
        \If{$Obj^\prime < Obj^\ast$}{
            $c^\ast \gets c$\;
            $Obj^\ast \gets Obj^\prime$\;
            $v^\ast \gets v$\;
        }
        }
        \If{$c^\ast \neq 0$}{
        $\mathcal{S}^{\star} \gets \mathcal{S}_0 \setminus \{v^\ast \} \cup \{c^\ast\}$\;  
        } 
     }
}
Return vehicle selection set $\mathcal{S}^\star$ and data collection $\mathbf{g}_v^\star$
\end{algorithm}

\subsubsection{Step 2} Since Step 1 only considers the term $D_{\text{client}}$, in what follows, we refine the algorithm by considering the original objective function $D_{\text{client}} + D_{\text{global}}$ based on a local search procedure. Specifically, during each iteration, we replace the least effective vehicle, i.e., introducing the maximum incremental value to the objective function $D_{\text{client}} + D_{\text{global}}$, with the one that has the minimum value in the unselected vehicle set. This procedure continues until no further improvements can be made or the maximum number of iterations is reached. The proposed two-step algorithm is presented in Algorithm \ref{alg:optimal}.

%The above two-step approach ensures a robust and efficient vehicle selection and data collection process, ultimately enhancing the effectiveness of the Sense4FL for AD systems.

Next, we provide the provable approximation guarantee of our proposed algorithm. 
\begin{theorem} \label{theo:bound}
    The objective $Obj^\dag$ obtained from Algorithm \ref{alg:optimal} satisfies $Obj^\dag \leq \frac{1+\delta}{\delta} Obj^\star$, where $Obj^\star >0$ is the optimal objective value to Problem (\ref{formulation}) and $\delta \triangleq \sum_{j=1}^{T-1}(1+\eta \lambda_{\max})^{j}$.
\end{theorem}

% \begin{proof}
%     The proof can be found in Appendix \ref{sec: the2}.
% \end{proof}

\begin{proof}
We establish the following relationship between $D_{\text{client}}$ and $D_{\text{global}}$.
    \begin{equation}\label{eq:F_2}
    \begin{aligned}
        D_{\text{global}} = 
        &\sum_{i=1}^C \|\sum_{v=1}^{ V}\frac{a_{v}\rho_v}{\sum_{v=1}^{ V}a_{v}\rho_v}\sum_{m=1}^{M_v}\overline{\xi}_{v,m}p_{v,m}^{i}- \sum_{b=1}^B l_b p_{b}^{i}\|\\
        \leq &\sum_{v=1}^{ V}\frac{a_{v} \rho_v}{\sum_{v=1}^{ V}a_{v} \rho_v}\sum_{m=1}^{M_v}\overline{\xi}_{v,m}\sum_{i=1}^C \|p_{v,m}^{i}- \sum_{b=1}^B l_b p_{b}^{i}\|\\
        \leq &\frac{1}{\delta}D_{\text{client}}. 
    \end{aligned}
\end{equation}

From Step 1, the optimal value of $D_{\text{client}}$ is $D_{\text{client}}^\dag = d^\dag$, and in this case, we can calculate the value of $D_{\text{global}}^\dag$ based on the corresponding solution $\mathbf{a}^\dag$. Moreover, we assume the true optimal value for the original problem is $D_{\text{client}}^\star+D_{\text{global}}^\star$, which is obtained by jointly optimizing both terms. It holds that $D_{\text{client}}^\dag \leq D_{\text{client}}^\star$.
Since $D_{\text{global}} \leq \frac{1}{\delta}D_{\text{client}}$ as derived in \eqref{eq:F_2}, we can get $D_{\text{global}}^\dag \leq \frac{1}{\delta}D_{\text{client}}^\dag$. 
Therefore, we arrive at 
\begin{equation}
    \begin{aligned}
        &\frac{D_{\text{client}}^\dag+D_{\text{global}}^\dag}{D_{\text{client}}^\star+D_{\text{global}}^\star
        } \leq \frac{D_{\text{client}}^\dag(1+\frac{1}{\delta})}{D_{\text{client}}^\star+D_{\text{global}}^\star
        }  \\
        \leq &\frac{D_{\text{client}}^\dag(1+\frac{1}{\delta})}{D_{\text{client}}^\star
        }
        \leq \frac{D_{\text{client}}^\star(1+\frac{1}{\delta})}{D_{\text{client}}^\star
        } \\
        \leq &\frac{1+\delta}{\delta},
    \end{aligned}
\end{equation}
which completes the proof.
\end{proof}

\rev{Finally, we analyze the computational complexity of the proposed algorithm.}
% \subsection{\rev{Computational Complexity}}

\begin{theorem}
    \rev{The overall computational complexity for Algorithm \ref{alg:optimal} is $\mathcal{O}\left((V N_{\mathrm{tot}}\log \frac{2 V}{\sigma}+ N_{\mathrm{ite}}N_{\mathrm{tot}}S(V-S))\right)$, where $N_{\mathrm{tot}} = \sum_{v=1}^V\sum_{m=1}^{M_v}N_{v,m}$.}
\end{theorem}

\begin{proof}
    \rev{We propose a two-step algorithm to optimize both vehicle selection and data collection strategy. In Step 1, a bisection method is employed to find the optimal vehicle selection. This process has a computational complexity of $\mathcal{O}\left(VN_{\mathrm{tot}} \log \frac{2 V}{\sigma}\right)$, where $N_{\mathrm{tot}} = \sum_{v=1}^V\sum_{m=1}^{M_v}N_{v,m}$.
    To refine the selection by considering the original objective function, we introduce a local search procedure, with a complexity of $\mathcal{O}\left(N_{\mathrm{ite}}N_{\mathrm{tot}}S(V-S)\right)$, where $N_{\mathrm{ite}}$ is the number of iterations until convergence.
    Consequently, the overall computational complexity for Algorithm \ref{alg:optimal} is $\mathcal{O}\left((V N_{\mathrm{tot}}\log \frac{2 V}{\sigma}+ N_{\mathrm{ite}}N_{\mathrm{tot}}S(V-S))\right)$.}
\end{proof}
\rev{The low complexity of the algorithm, combined with the parallel processing of vehicles, ensures the scalability of our proposed framework even in dense urban scenarios.}

\begin{table}[!t]
\centering
\caption{\rev{Parameter settings for simulations.}}
\label{table_para}
\renewcommand{\arraystretch}{1.4}
\rev{
\setlength{\tabcolsep}{2mm}
\begin{tabular}{|c|c|}
\hline
{Number of selected vehicles} &  {$S$ = 10}\\  \hline
{Model size}& {$\omega = 5.904 \times 10^8$ bit}   \\  \hline
Number of processing cycles & {$c_v = 9.8304 \times 10 ^{7}$}  \\  \hline
{Vehicle speed }  & {40-50 km/h / 50-60 km/h}   \\ \hline
Time constraint for one round & {$T^\text{task} = 80$ seconds}  \\  \hline
{Number of possible trajectories}  & {$M_v = 2$ } \\ \hline    
{Steps of local SGD updates}  & {\( T = 2\)}  \\  \hline
{Time required via wired link} &\( t^{\mathrm{trans}} \) = 1 second \\ \hline
Computing capability & {$f_v$ = 40 GFLOPS }  \\  \hline
{Minimum uploading data rate}  & {50 Mbps } \\ \hline
Lipschitz parameter & {$\lambda_{\text{max}} = 0.01$}  \\  \hline
Batch size  & {$D_{\text{Batch}} = 32$} \\ \hline
\end{tabular}
}
\end{table}

\section{Experiments}
\label{sec:exp}
In this section, we provide numerical experiments to evaluate the performance of our proposed Sense4FL framework. We compare our Sense4FL framework with several benchmark methods and demonstrate the superiority of our scheme.

% \begin{table*}[!t]
% \centering
% \caption{\rev{Parameter settings for simulations}}
% \label{table_para}
% \renewcommand{\arraystretch}{1.4}
% \rev{
% \setlength{\tabcolsep}{2mm}
% \begin{tabular}{|c|c|c|c|}
% \hline
% {Number of selected vehicles} &  {$S$ = 10}&  {Model size}& {$\omega = 5.904 \times 10^8$ bit}   \\  \hline
% Number of processing cycles for one sample & {$c_v = 9.8304 \times 10 ^{7}$}  & {Vehicle speed }  & {40 km/h-72 km/h}   \\ \hline
% Time constraint for one round & {$T^\text{task} = 2$ min}  & {Number of possible trajectories}  & {$M_v = 2$ } \\ \hline    
% {Steps of local SGD updates}  & {\( T = 2\)}  & {Time required via wired link} &\( t^{\mathrm{trans}} \) = 1 second \\ \hline
% Computing capability & {$f_v$ = 10 GFLOPS }  & {Minimum expected uploading data rate}  & {40 Mbps } \\ \hline
% Lipschitz parameter & {$\lambda_{\text{max}} = 0.01$}  & {Task completion probability}  & {$\eta$ = 0.75 } \\ \hline
% \end{tabular}
% }
% \end{table*}

\subsection{Experiment Settings}
We conduct experiments with the nuImages dataset, an extension of nuScenes designed for 2D object detection~\cite{caesar2020nuscenes}. The images were captured by six cameras mounted around a vehicle, and data was collected in Singapore and Boston, respectively. Each image’s bounding box was converted to YOLO format, where a bounding box is represented by normalized center coordinates, box width, and height~\cite{quemeneur2024fedpylot}. We divide the images based on their collection locations and generate vehicle traffic using the SUMO simulator. 
Moreover, the images captured in Singapore's One-North area are designed for training from scratch. After that, we adapt the model in the Boston Seaport scenario to illustrate the necessity for adapting models based on Sense4FL when the environment has changed. 
\rev{Specifically, the data is split into training/adaptation and test sets with a ratio of 4:1.
Images are divided into 36 distinct blocks based on their capture location. From these blocks, we generate 10 realistic trajectories, which are illustrated in Fig. \ref{fig:map}.}
\rev{This location-dependent data splitting strategy naturally induces heterogeneity, which can be observed for each trajectory in both Singapore’s One-North and Boston Seaport in Fig. \ref{fig:data}. 
}
We repeat our simulation process 15 times to obtain the average performance.

\begin{figure}[t]
\centering
\begin{subfigure}[b]{0.23\textwidth}
  \includegraphics[width=\textwidth]{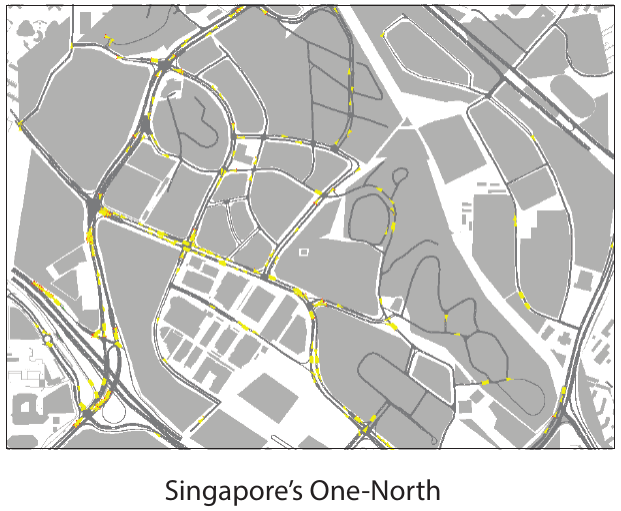}
  \caption{The trajectories in Singapore’s One-North.}\label{fig:singapore_map}
\end{subfigure}
\hfill
\begin{subfigure}[b]{0.23\textwidth}
  \includegraphics[width=\textwidth]{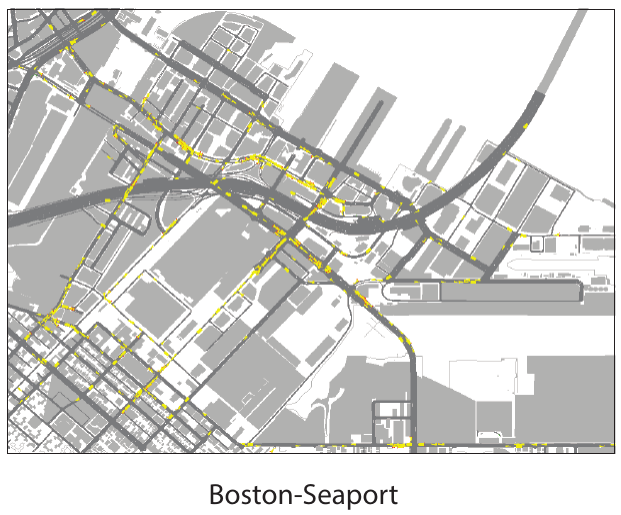}
  \caption{The trajectories in Boston Seaport.} \label{fig:boston_map}
\end{subfigure}
\caption{\rev{The trajectories in Singapore’s One-North and Boston Seaport.}}\label{fig:map}
\vspace{-0.0cm}
\end{figure}

\begin{figure}[t]
\centering
\begin{subfigure}[b]{0.24\textwidth}
  \includegraphics[width=\textwidth]{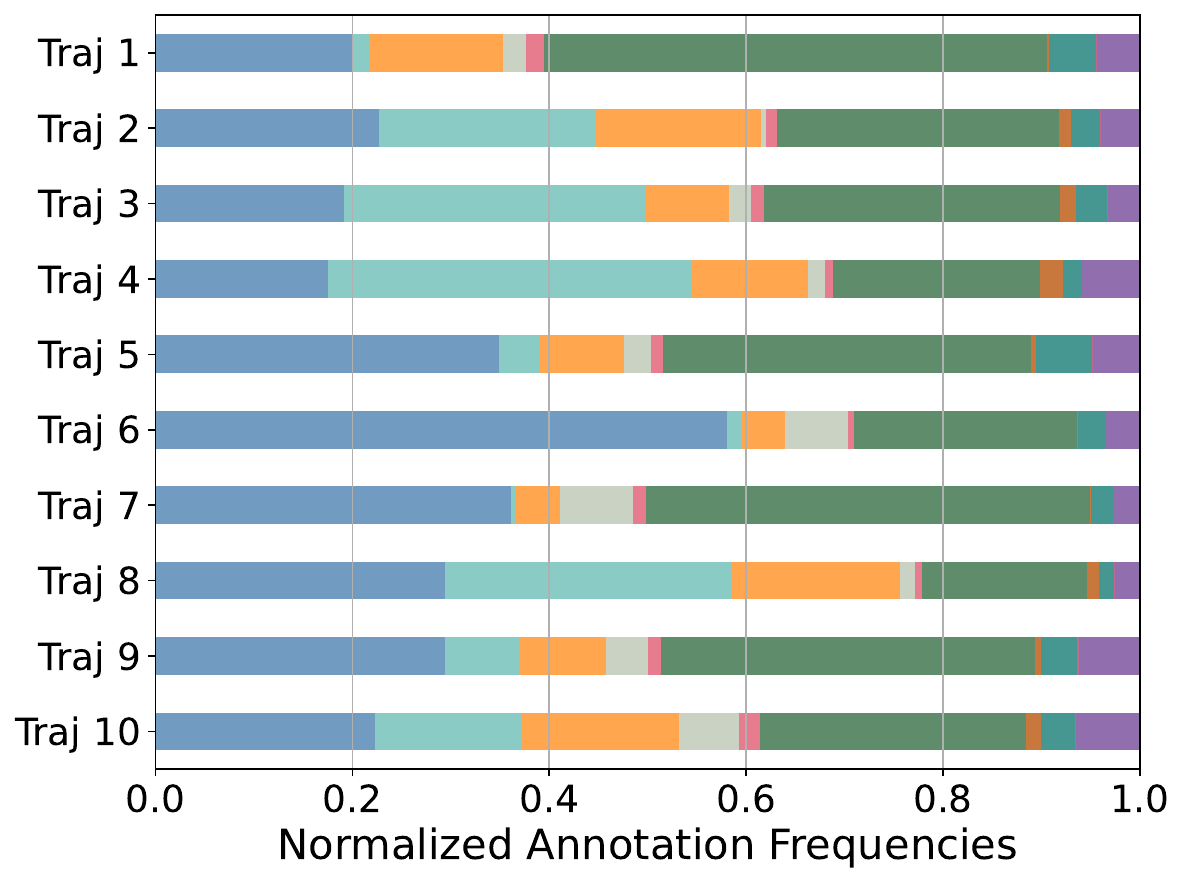}
  \caption{The normalized frequency of objects for each trajectory in Singapore’s One-North.} \label{fig:singapore_data}
\end{subfigure}
\hfill
\begin{subfigure}[b]{0.24\textwidth}
  \includegraphics[width=\textwidth]{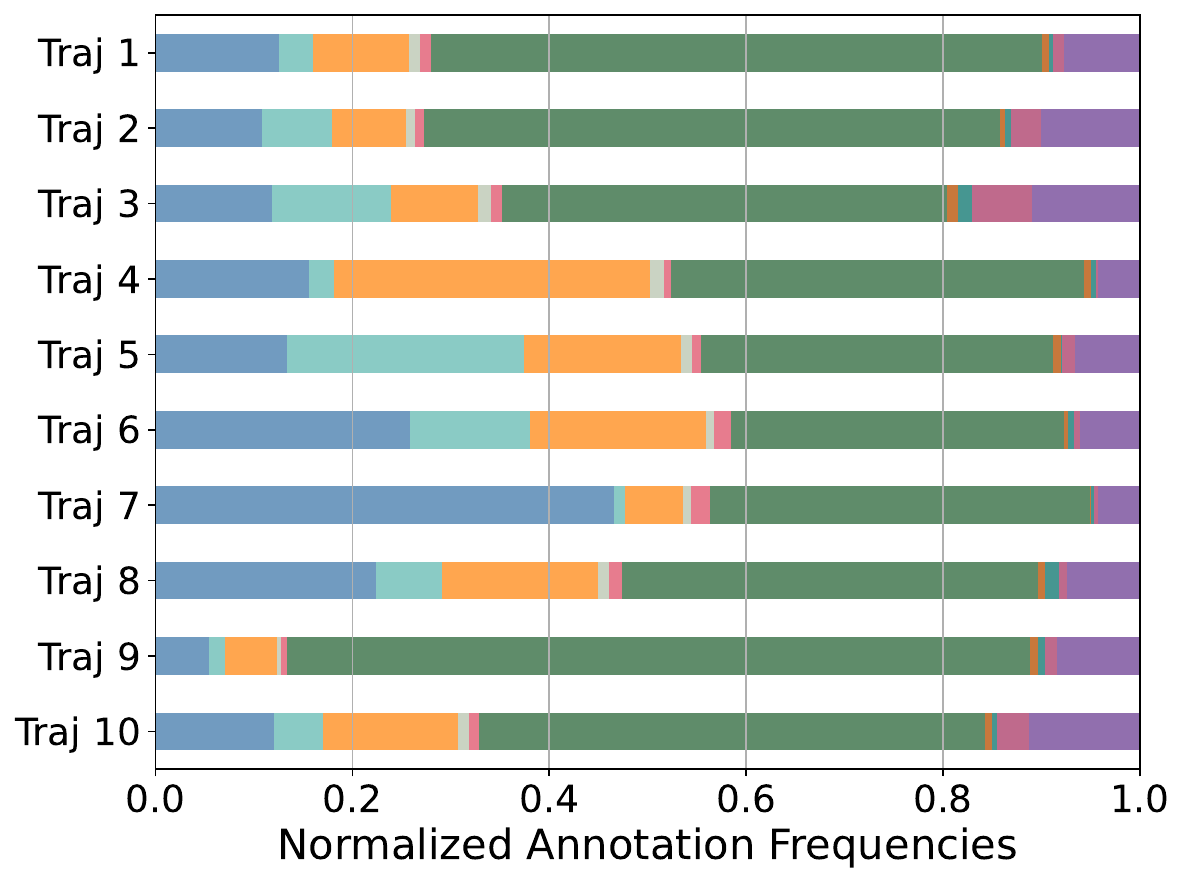}
  \caption{The normalized frequency of objects for each trajectory in Boston Seaport.} \label{fig:boston_data}
\end{subfigure}
\caption{\rev{The data distribution for each trajectory in Singapore’s One-North and Boston Seaport. The data is split non-IID among ten trajectories based on the data capture location. The different colors represent different classes: pedestrian, barrier, traffic cone, bicycle, bus, car, construction vehicle, motorcycle, trailer, and truck.}}\label{fig:data}
\vspace{-0.0cm}
\end{figure}

\begin{figure*}[t]
\centering
\begin{subfigure}[b]{0.323\textwidth}
  \includegraphics[width=\textwidth]{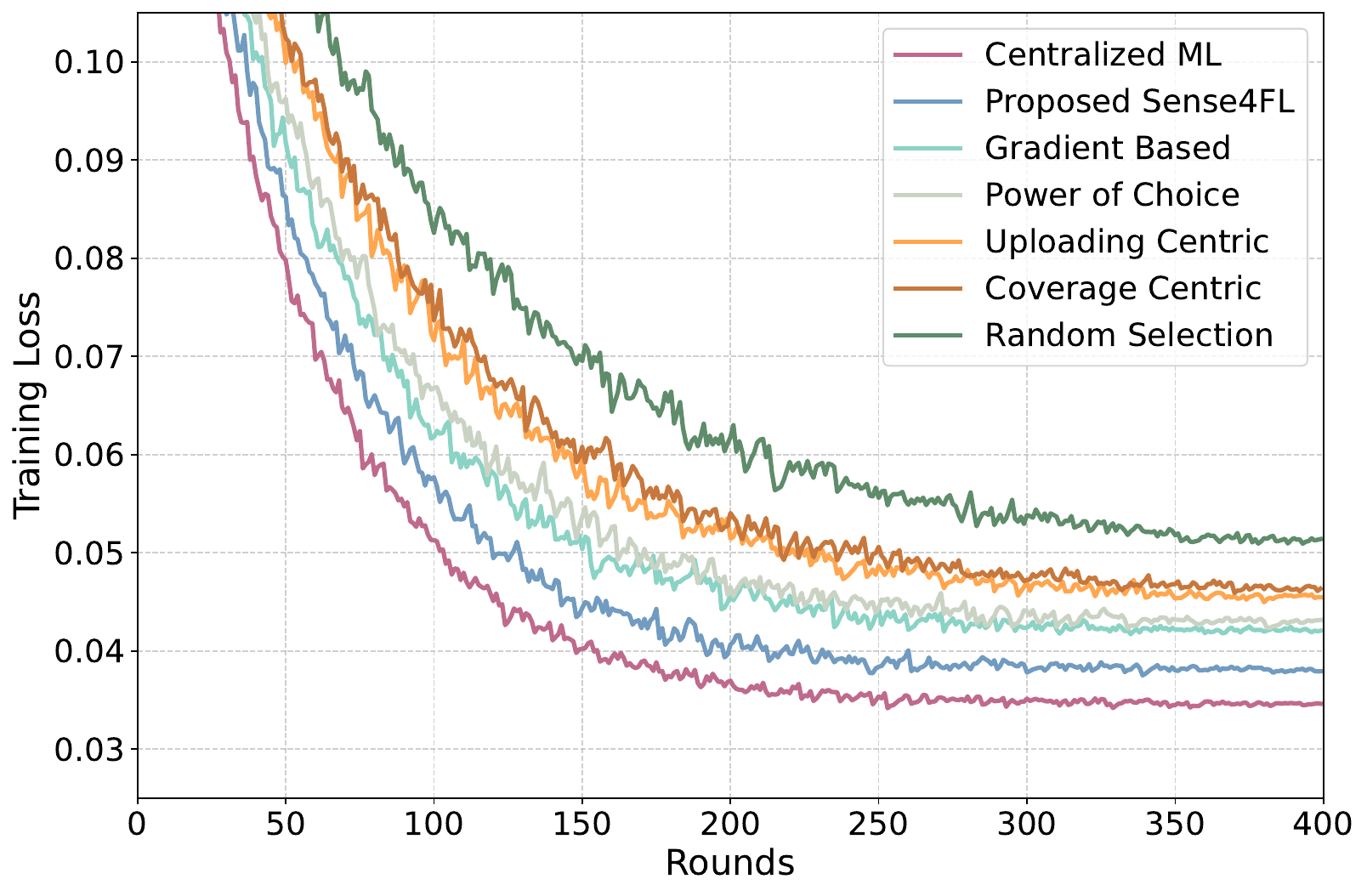}
  \caption{Training loss versus rounds.}
\end{subfigure}
\hfill
\begin{subfigure}[b]{0.32\textwidth}
  \includegraphics[width=\textwidth]{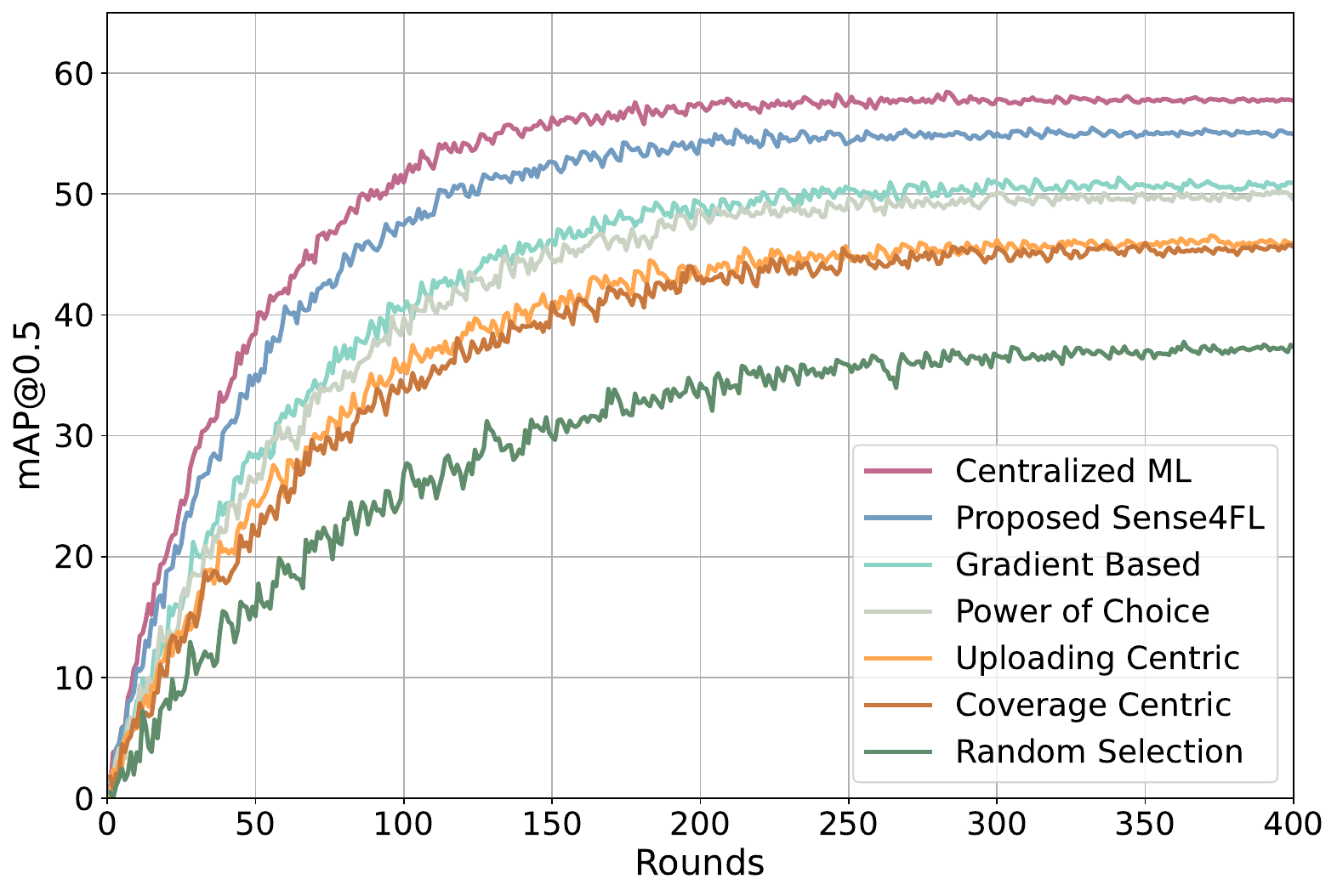}
  \caption{Testing accuracy versus rounds.}
\end{subfigure}
\hfill
\begin{subfigure}[b]{0.32\textwidth}
  \includegraphics[width=\textwidth]{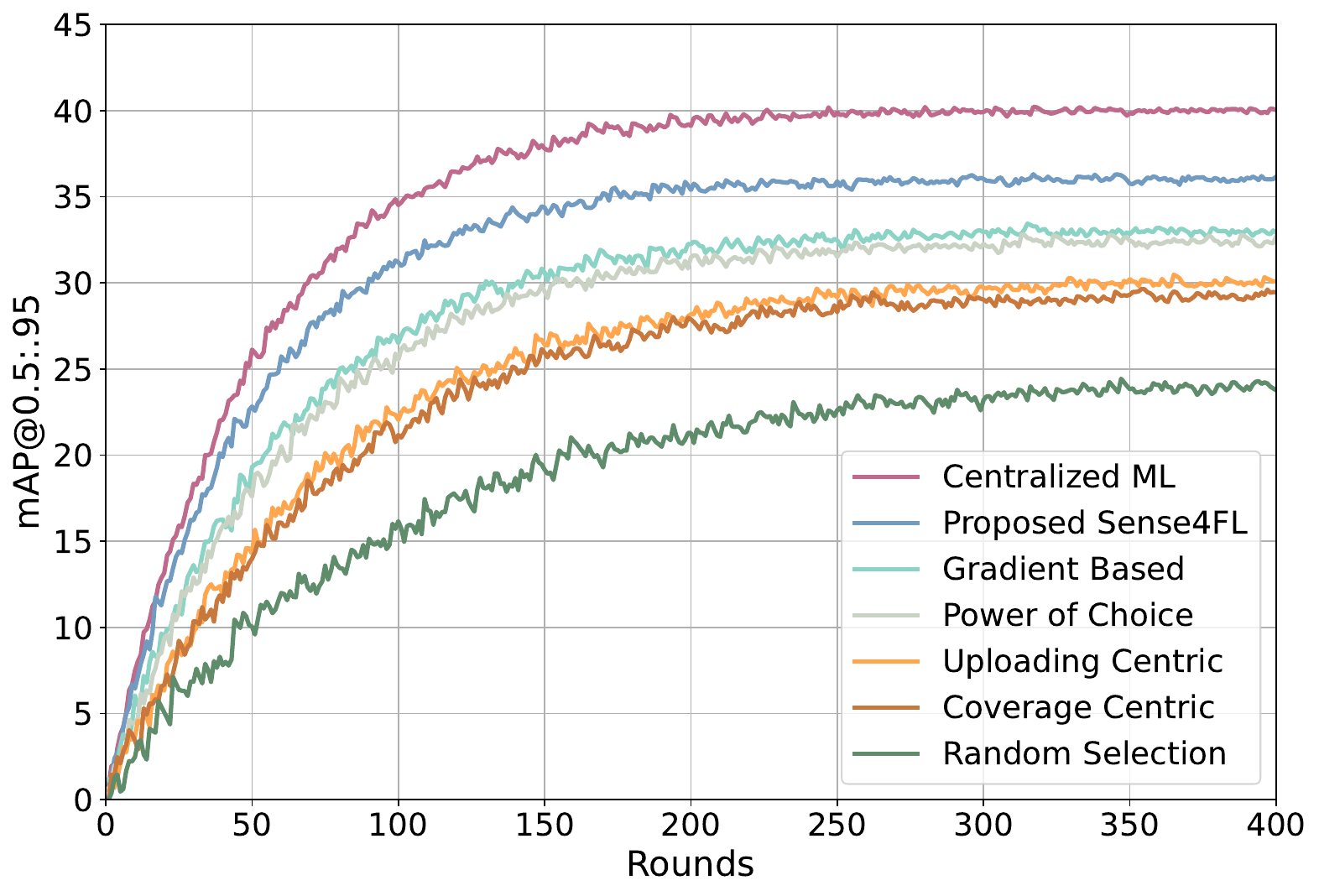}
  \caption{Testing accuracy versus rounds.}
\end{subfigure}
% \caption{Experiment results in training processing.}

\medskip % 添加一些垂直间距
% \hfill
\centering
\begin{subfigure}[b]{0.32\textwidth}
  \includegraphics[width=\textwidth]{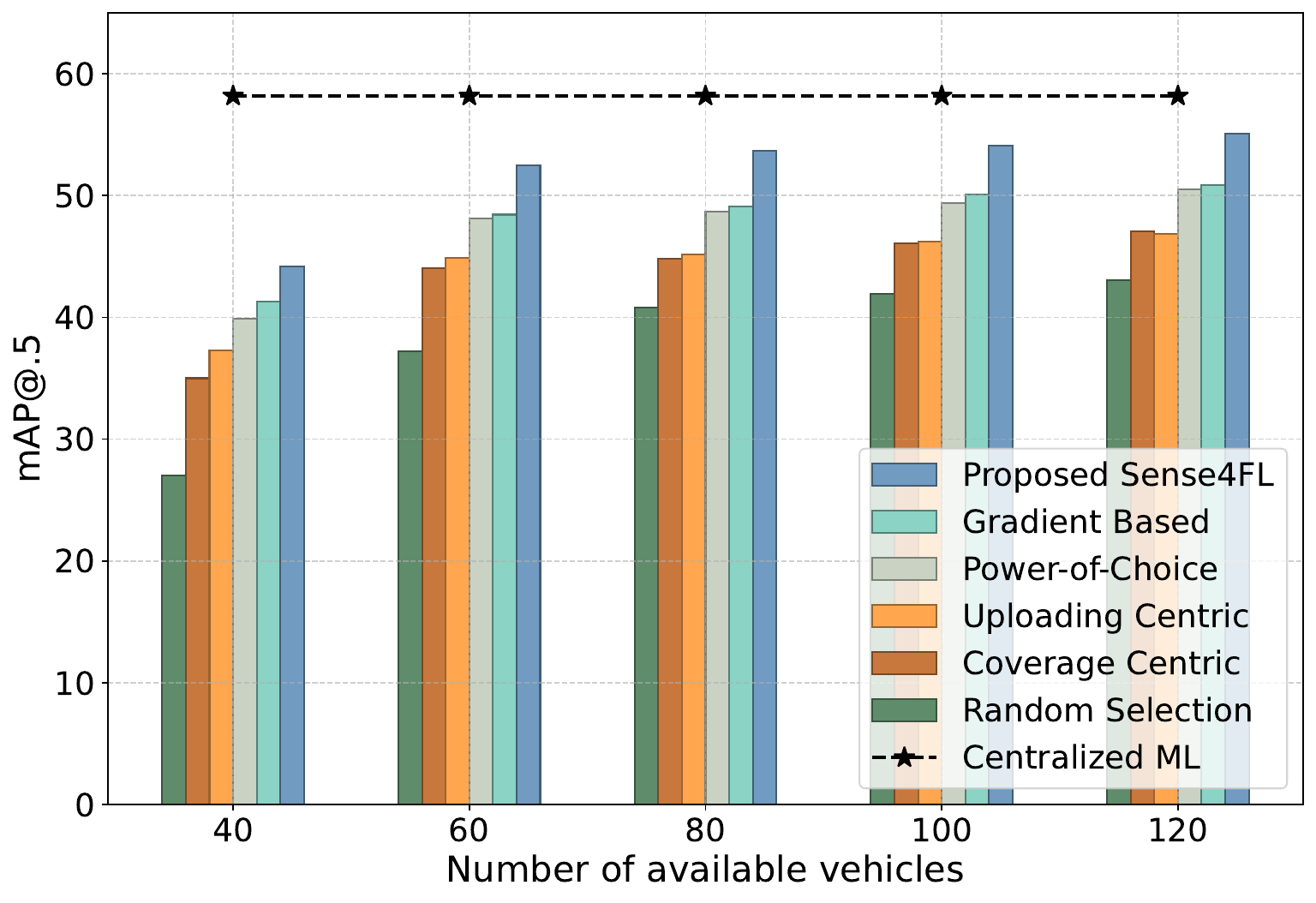}
  \caption{Testing accuracy versus the number of available vehicles.}
\end{subfigure}
\hfill
\begin{subfigure}[b]{0.32\textwidth}
  \includegraphics[width=\textwidth]{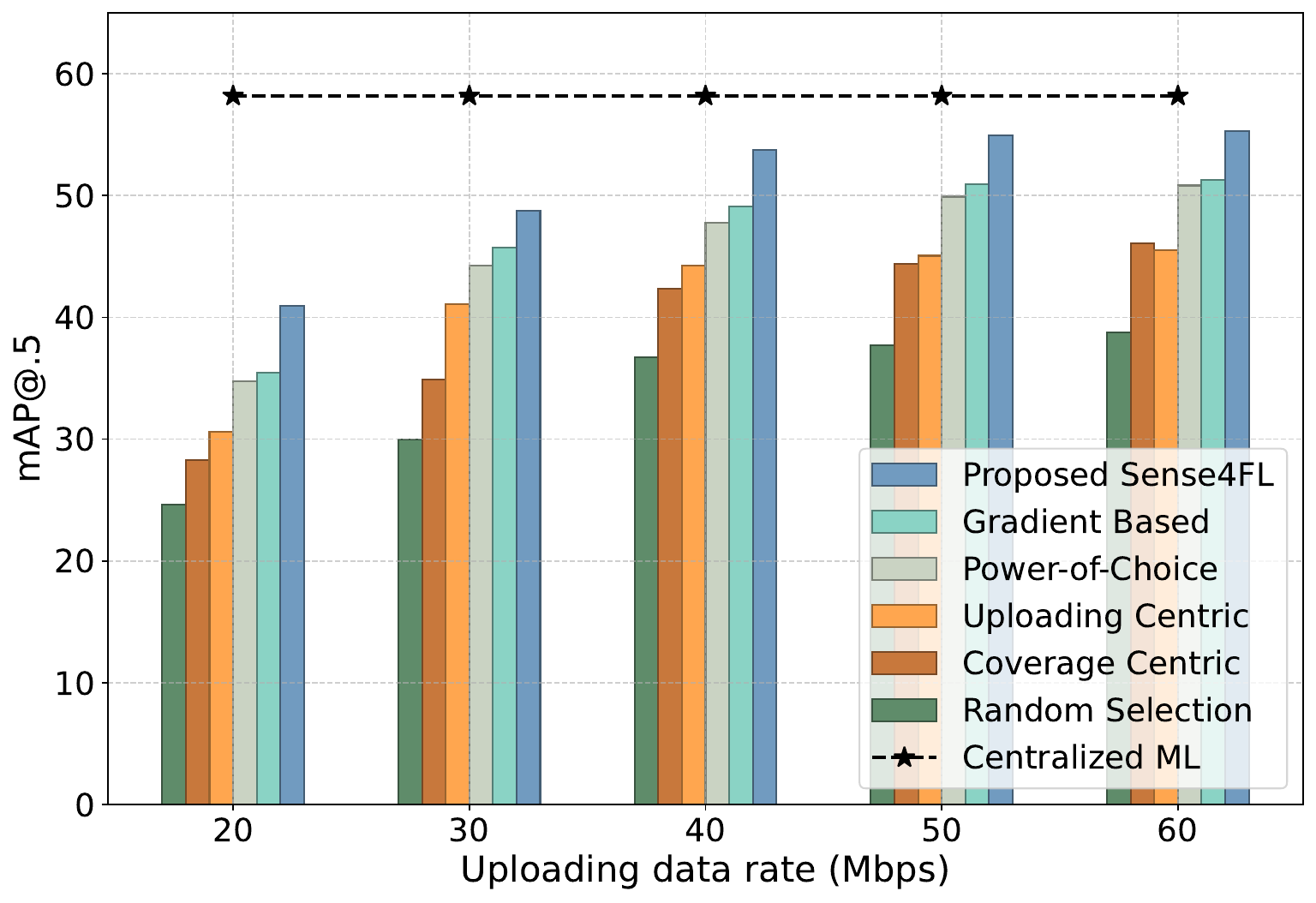}
  \caption{Testing accuracy versus the uploading data rate.}
\end{subfigure}
\hfill
\begin{subfigure}[b]{0.32\textwidth}
  \includegraphics[width=\textwidth]{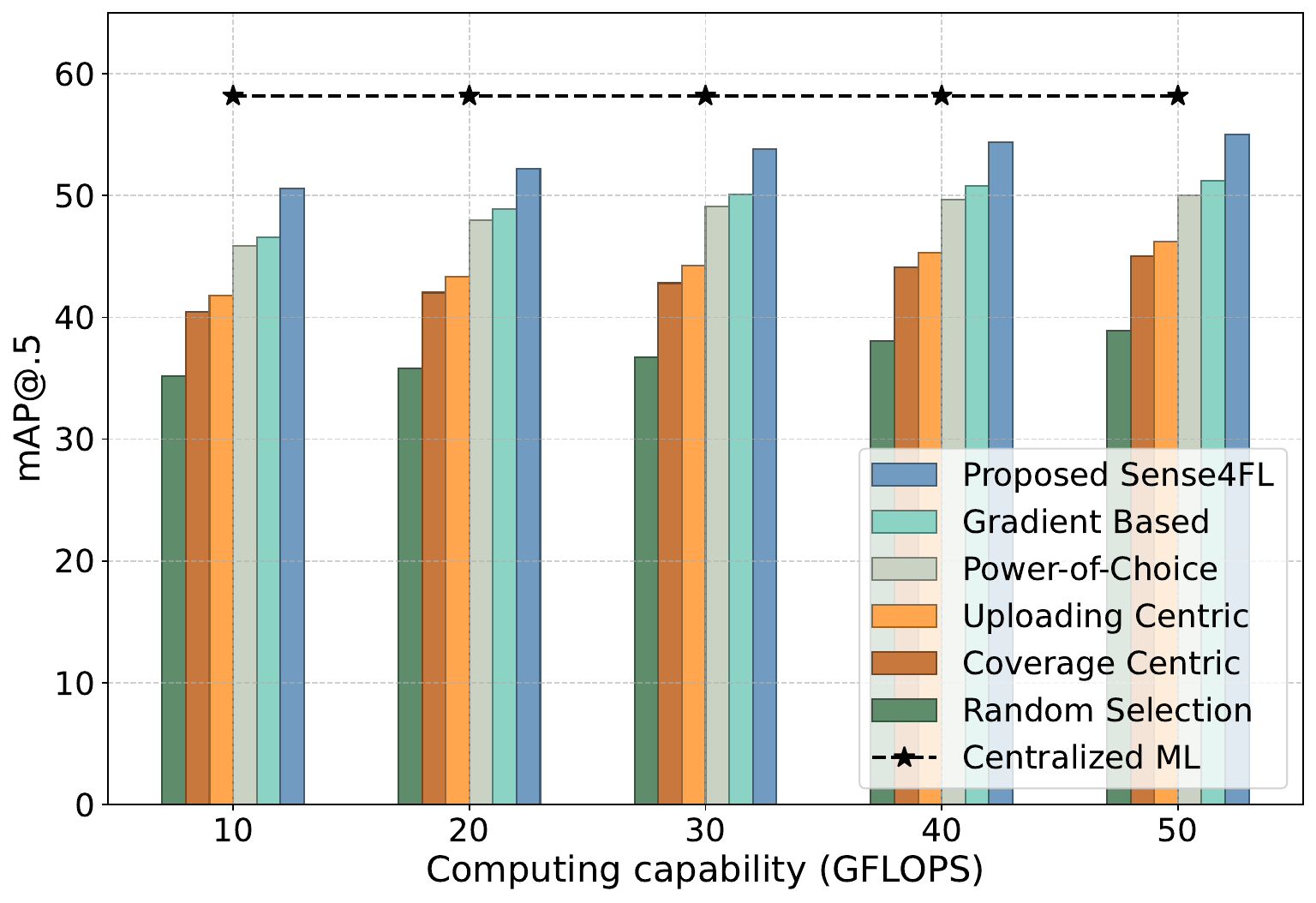}
  \caption{Testing accuracy versus the computing capability.}
\end{subfigure}
\caption{\rev{Experiment results for training from scratch in Singapore's One-North scenario in the nuImages dataset.}}\label{fig:training}
\vspace{-0.0cm}
\end{figure*}

Besides, unless specified otherwise, the default settings are provided as follows. \rev{The number of available vehicles is around 60-70 in Singapore’s One-North and 70-80 in Boston Seaport, and vehicles in Singapore’s One-North travel at speeds between 50-60 km/h, while those in Boston Seaport travel at  40-50 km/h. We set the maximum acceleration at 2.0 $m/s^2$, the maximum deceleration (braking) at 3.0 $m/s^2$ based on the Krauss model.} Each vehicle has $M_v = 2$ possible trajectories, which can be predicted based on its current location and orientation with historical traffic data. The number of vehicles to be selected is $S$ = 10. 
\rev{We consider 640*640 images with a color depth of 24 bits, and thus the required number of processing cycles for computing one sample is $c_v = 9.8304 \times 10 ^{7}$ \cite{shiUAV}.}
The computing capability of each vehicle is $f_v$ = 40 GFLOPS. The adopted YOLOv7 model, with 36.9 million parameters and using a 16-bit version, has a model size of $\omega = 5.904 \times 10^8$ bits.
The time constraint of the FL task for one round is $T^\text{task} = 80$ seconds.
The minimum expected data rate for uploading is 50 Mbps, and the time required for the BS to transmit the model to the FL server via a wired link is \( t^{\mathrm{trans}} \) = 1 second. Each selected vehicle performs $T=2$ steps local SGD updates before uploading, with the batch size $D_{\text{Batch}} = 32$. 
The learning rate for training from scratch is set to $\eta = 0.001$ and for adaptation is set to $\eta = 0.0001$.
For the hyperparameter, the Lipschitz parameter $\lambda_{\text{max}}$ can be estimated and is 0.01 in this model~\cite{wang2019adaptive}. \rev{The key parameters are summarized in Table \ref{table_para}.}

\begin{figure*}[t]
\centering
  \includegraphics[width=0.95\textwidth]{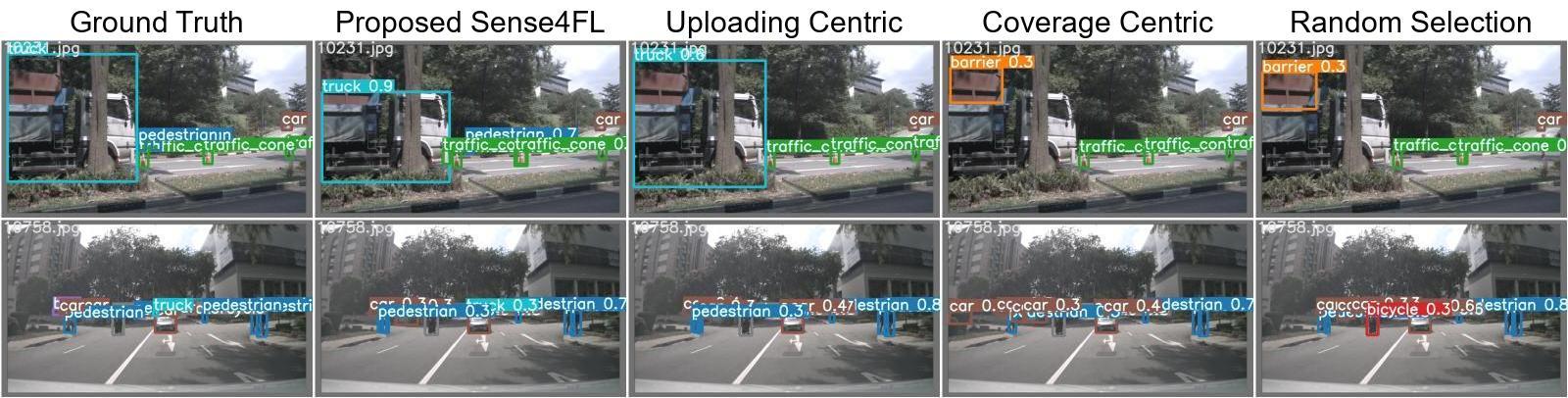}
  \caption{Illustration of object detection results.}\label{Fig:detection}
  % \caption{Experiment results in adapting processing.}
  \vspace{-0.0cm}
\end{figure*}

% \begin{figure}[t]
% \centering
% \begin{subfigure}[b]{0.32\textwidth}
%   \includegraphics[width=\textwidth]{figure/finetune_mAP_round.pdf}
%   \caption{Testing accuracy versus rounds.}
% \end{subfigure}
% % \hfill
% \begin{subfigure}[b]{0.32\textwidth}
%   \includegraphics[width=\textwidth]{figure/finetune_mAP95_round.pdf}
%   \caption{Testing accuracy versus rounds.}
% \end{subfigure}
% \caption{A adapting process based on the nuImages dataset. We adopt a pre-trained model from images captured in Singapore's One-North and then adapt it in Boston scenario.}\label{fig:adapting}
% \vspace{-0.3cm}
% \end{figure}

\begin{figure*}[t]
\centering
\begin{subfigure}[b]{0.323\textwidth}
  \includegraphics[width=\textwidth]{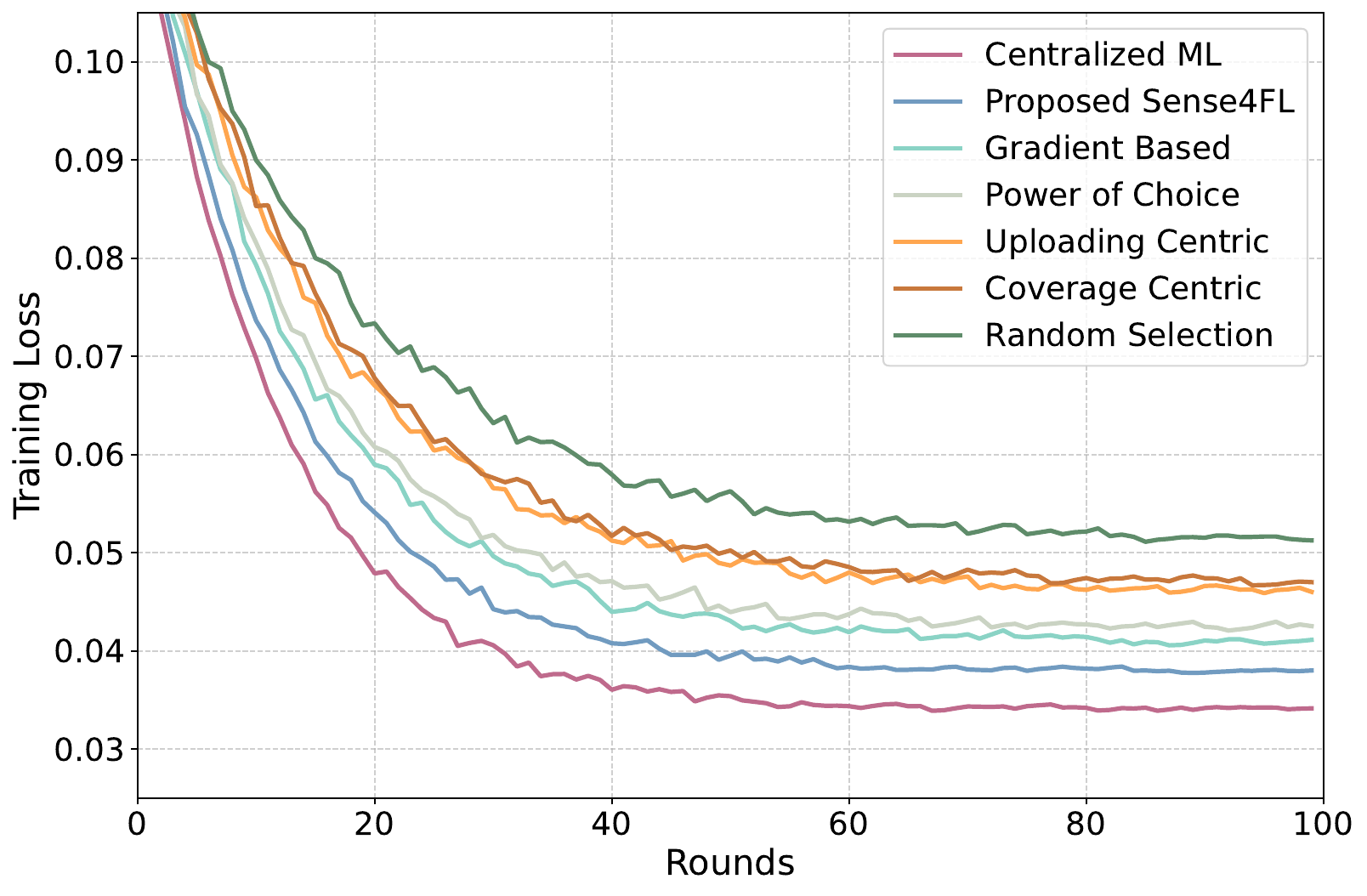}
  \caption{Training loss versus rounds.}
\end{subfigure}
% \hfill
\begin{subfigure}[b]{0.32\textwidth}
  \includegraphics[width=\textwidth]{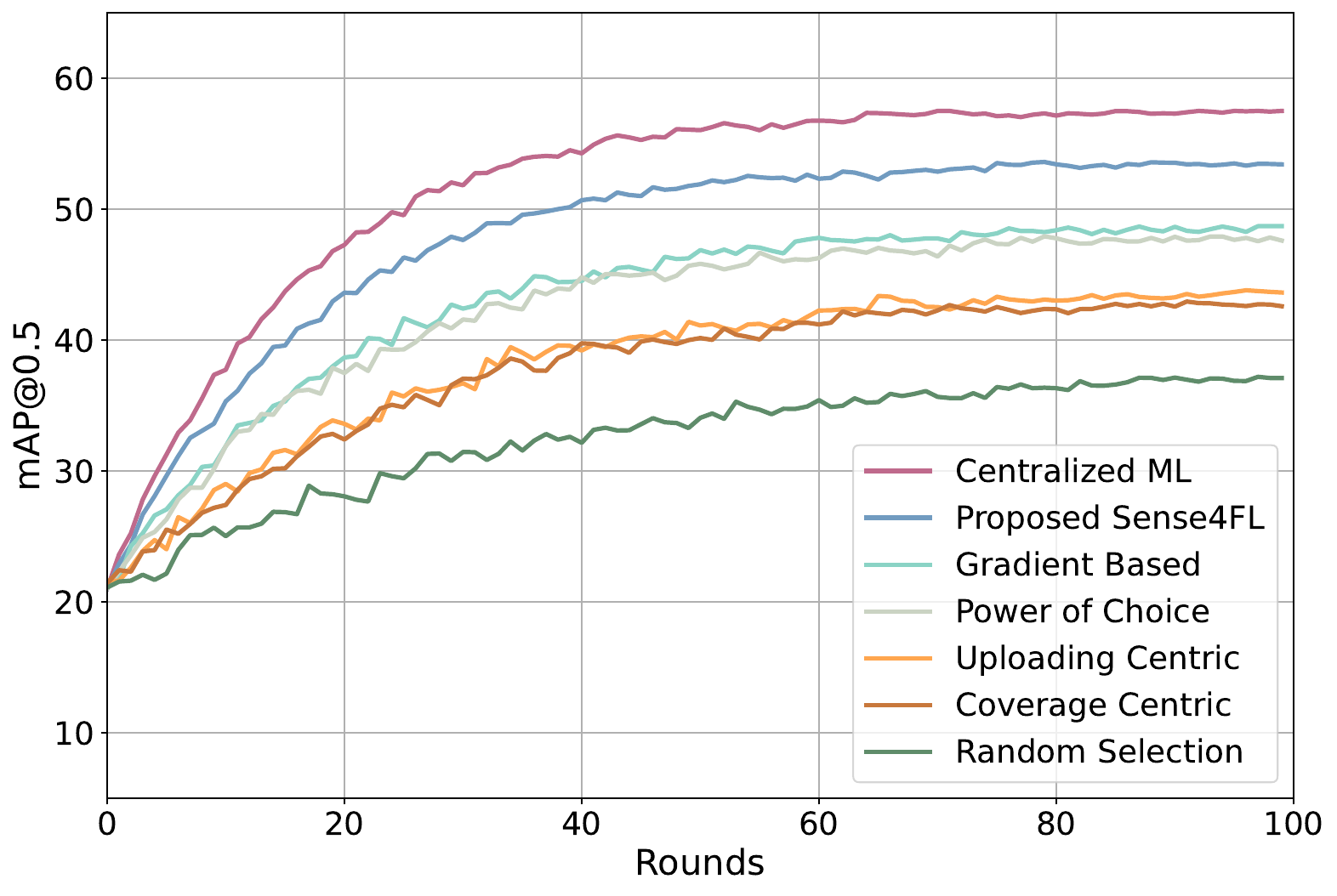}
  \caption{Testing accuracy versus rounds.}
\end{subfigure}
% \hfill
\begin{subfigure}[b]{0.32\textwidth}
  \includegraphics[width=\textwidth]{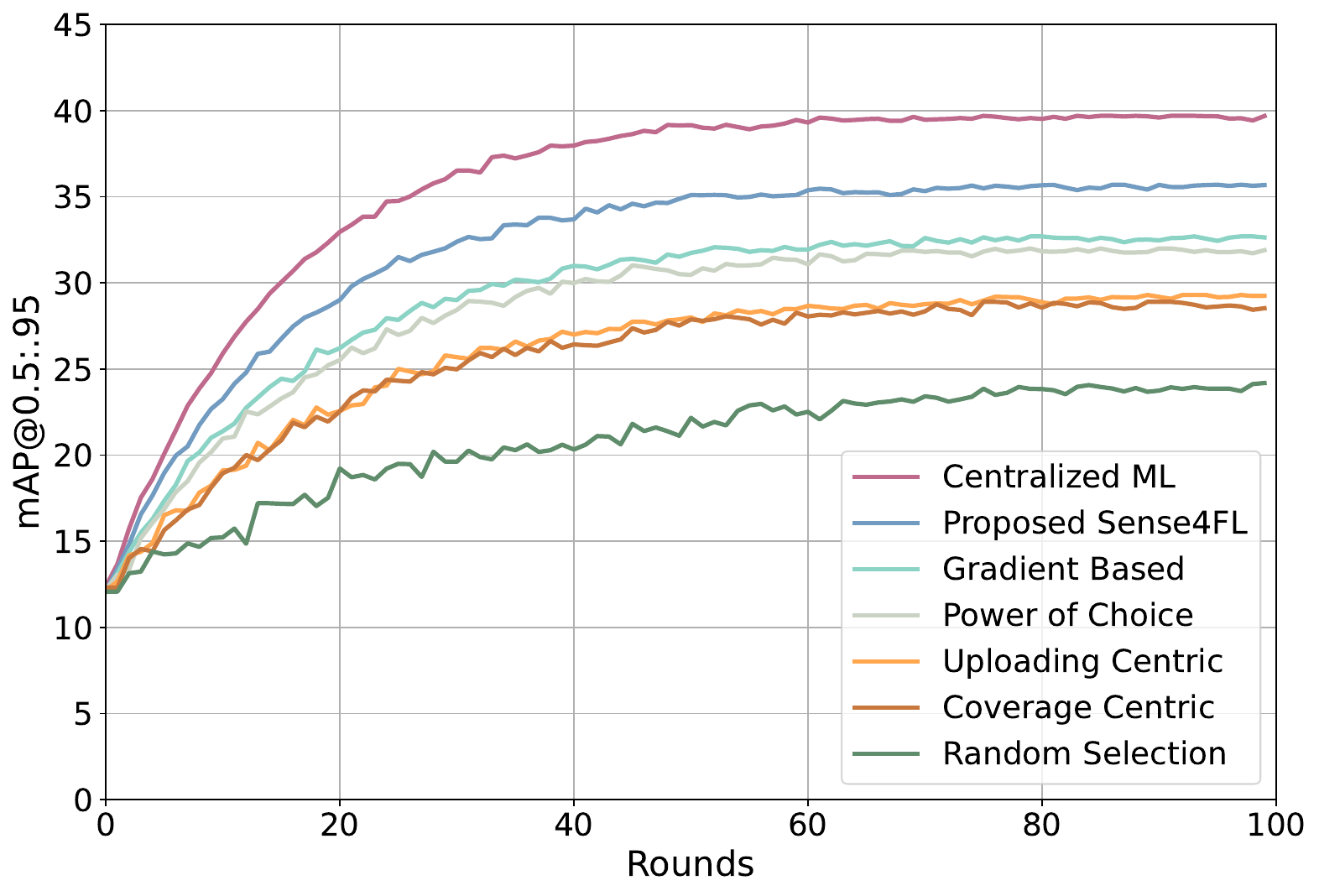}
  \caption{Testing accuracy versus rounds.}
\end{subfigure}
% \medskip % 添加一些垂直间距
% \hfill
\centering
\begin{subfigure}[b]{0.32\textwidth}
  \includegraphics[width=\textwidth]{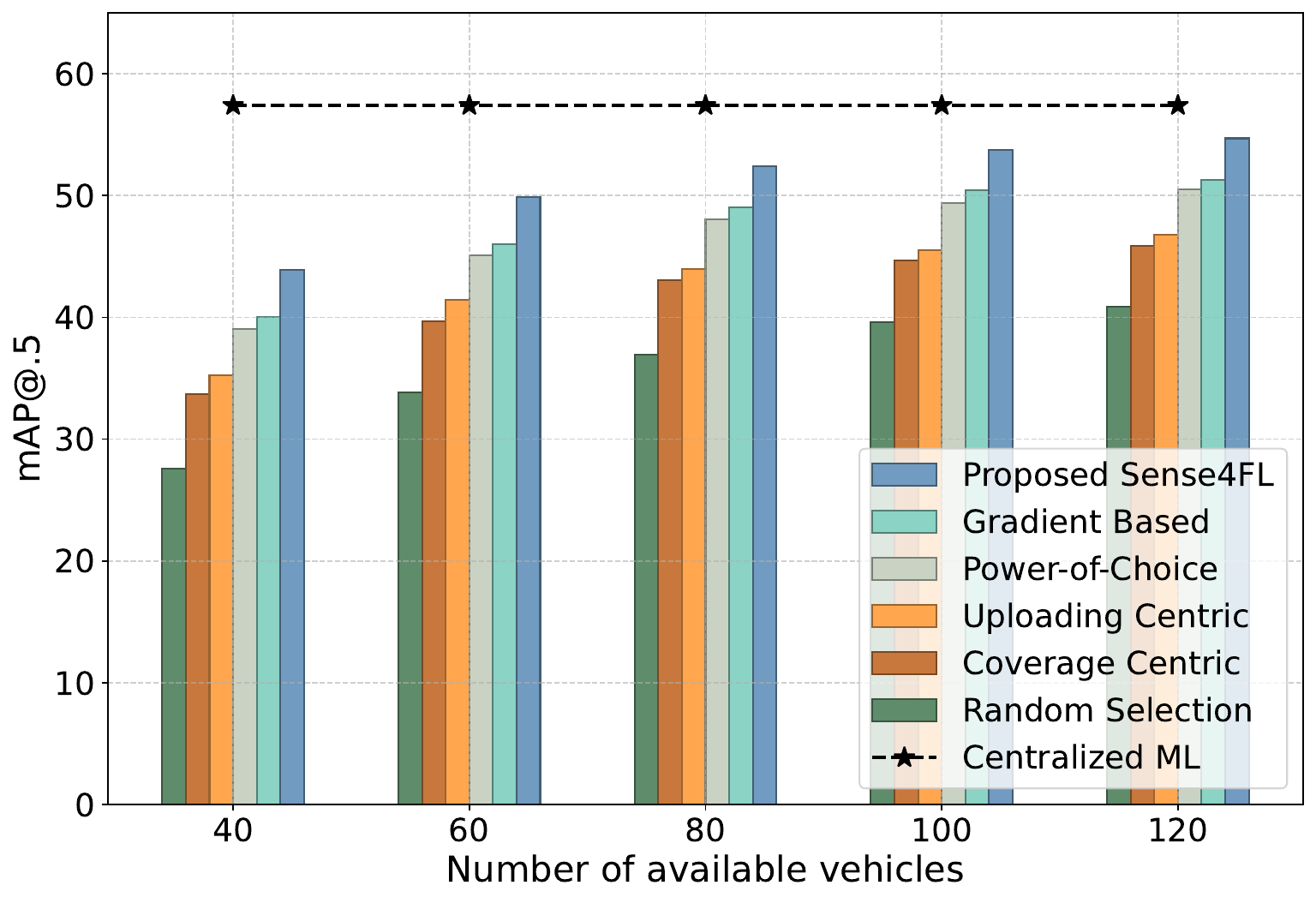}
  \caption{Testing accuracy versus the number of available vehicles.}
\end{subfigure}
% \hfill
\begin{subfigure}[b]{0.32\textwidth}
  \includegraphics[width=\textwidth]{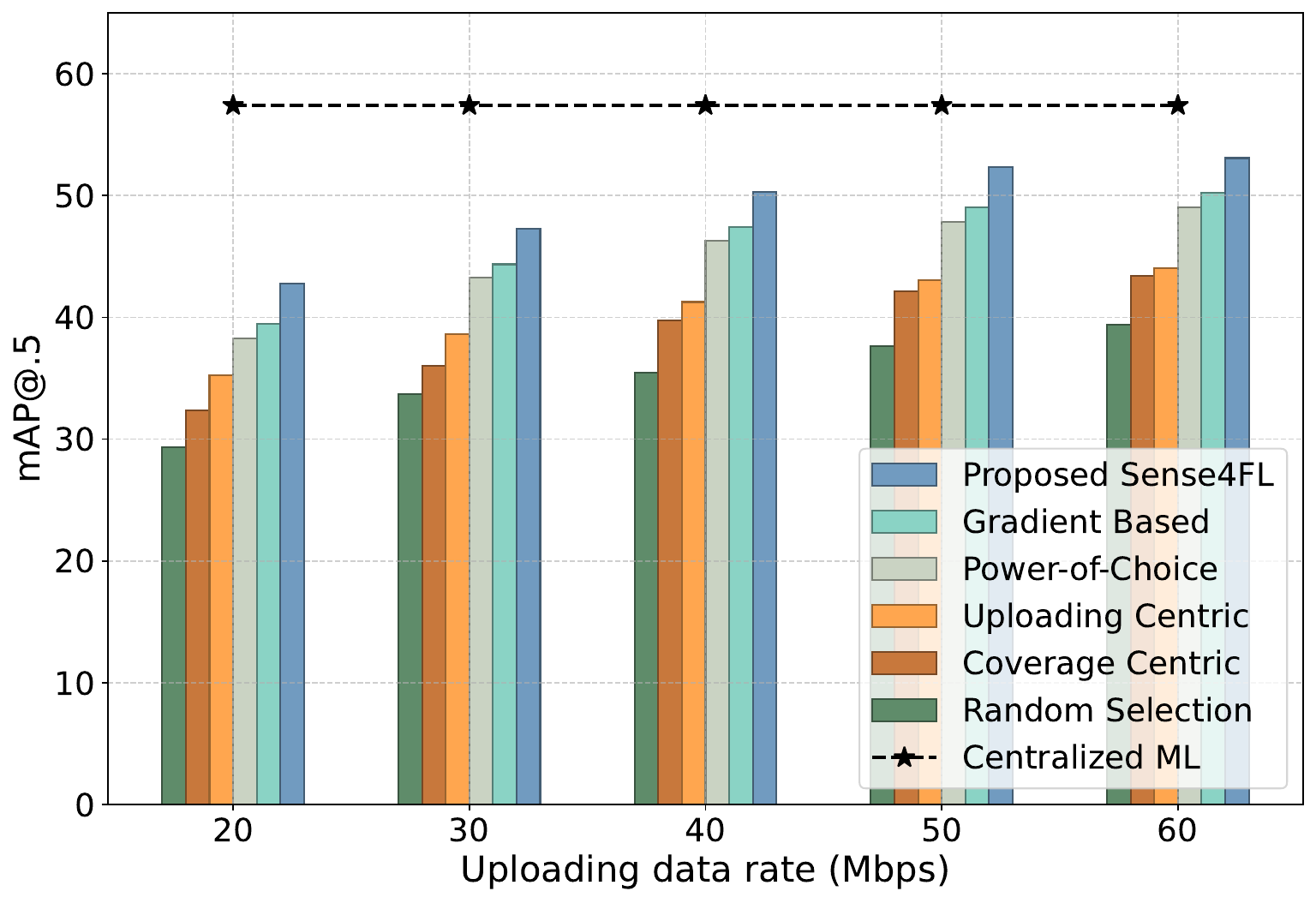}
  \caption{Testing accuracy versus the uploading data rate.}
\end{subfigure}
% \hfill
\begin{subfigure}[b]{0.32\textwidth}
  \includegraphics[width=\textwidth]{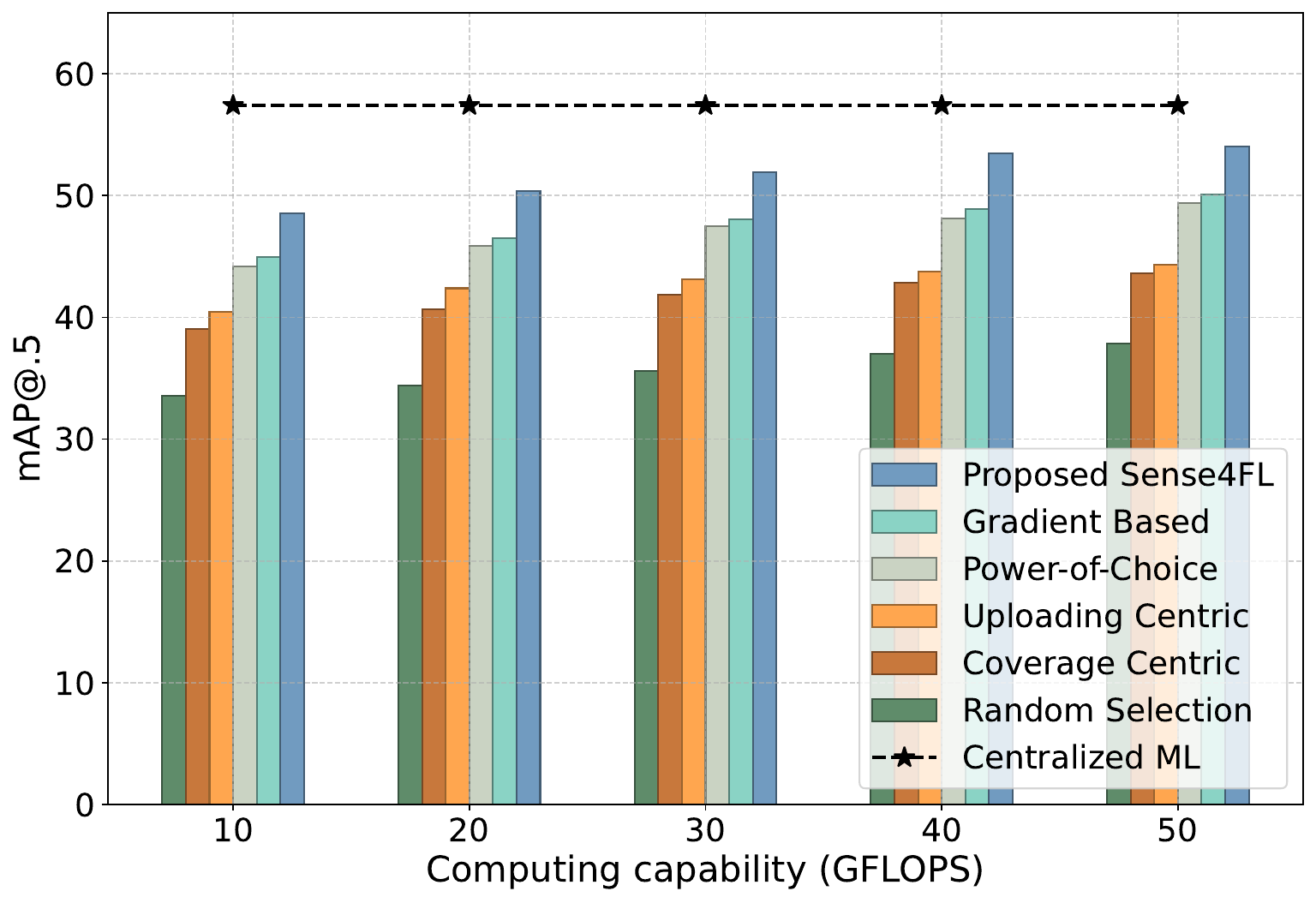}
  \caption{Testing accuracy versus the computing capability.}
\end{subfigure}
\caption{\rev{Experiment results for model adaptation. We adapt a model pre-trained in Singapore’s One-North to the Boston Seaport scenario.}}\label{fig:adapting}
\vspace{-0.3cm}
\end{figure*}

\subsection{Performance Evaluation of Sense4FL Framework}

To validate the effectiveness of the proposed Sense4FL framework, we compare it against several benchmark methods. 
\begin{itemize}
    \item \rev{\textbf{Gradient-based selection~\cite{marnissi2024client}.} This method selects vehicles with the highest norms of gradient values at each communication round. A vehicle stops data collection once the edge server has selected it.}

    \item \rev{\textbf{Power-of-Choice~\cite{cho2020client}.} This framework selects vehicles based on their local loss value. Upon selection by the edge server, the vehicle ceases its data collection process.}

    \item \textbf{Uploading-centric Selection~\cite{pervej2023resource}:} This method selects $S$ vehicles with the highest uploading probabilities from all available vehicles. Each vehicle stops collecting data {upon being selected.}

    \item \textbf{Coverage-centric Selection~\cite{7218644}:} This benchmark involves selecting $S$ vehicles and their data collection to maximize the number of covered street blocks. A street block is said to be covered as long as one vehicle collects training data from it. Hence, this approach can be formulated as a maximum coverage problem.

    \item \textbf{\rev{Random Selection~\cite{pmlr-v54-mcmahan17a}}:} In this method, we randomly select $S$ vehicles and their data collection in each round from all available vehicles to participate in FL training.

    % \item \rev{\textbf{Full-data Selection:} This method selects $S$ vehicles and lets each selected vehicle collect all available data along its entire predicted trajectory; equivalently, we set $g_{v,m} = N_{v,m}$ for vehicle $v$. }
    % % This baseline removes data-collection scheduling and prioritization, emphasizing the effect of maximal collection given the chosen $a_v^{(k)}$.

    % \item \rev{\textbf{Only Selection:} This method selects $S$ vehicles (i.e., determines $a_v^{(k)}$) but does not optimize data collection. Each vehicle stops collecting data upon being selected; $g_{v,m}^{(k)} = c_{v,m}$ is not adapted to representativeness. }

    \item \rev{\textbf{Centralized ML:} We also compare the results with the centralized machine learning (ML), which serves as the performance upper bound as the server can access all training data in this specific region.}

\end{itemize}

Fig. \ref{fig:training} illustrates the training loss and accuracy under different cases. The training loss is derived from the average of the local losses computed by each vehicle selected to participate in the FL process. The mean average precision (mAP) is measured by evaluating the global model at the end of each round on a separate test dataset stored on the FL server. 
%In the FL setting, the aggregated training loss is derived from the weighted average of the local losses computed by each vehicle selected to participate in the FL process.  This figure shows the decreasing trend of training loss, with it converging after approximately 170 rounds. 
%Specifically, The mAP@.5 and mAP@.5:.95 are measured by evaluating the global model at the end of each round on a separate validation dataset.
As seen from Fig. \ref{fig:training} (a)-(c), the training performance of our Sense4FL scheme considerably outperforms other benchmarks. This is because our scheme largely overcomes the model bias resulting from the inappropriate selection of vehicles and their training data and, therefore, makes the object detection model better represent the region of interest. 
Fig. \ref{fig:training} (d)-(f) evaluate the performance by varying the network settings, including the number of available vehicles, the uploading data rate, and the computing capabilities. Intuitively, where there are more vehicles available, there exists more flexibility to select better vehicles with desired training data, thereby enhancing the object detection performance. Similarly, higher uploading data rates or more powerful onboard computing capabilities increase the probability of successful reception of models by reducing the communication-computing latency in each training round, thus leading to better training performance. We also find that the proposed Sense4FL outperforms the other baselines since the training data distribution and the uploading probabilities are both taken into account in vehicle selection. These baselines perform worse as they introduce model bias caused by the non-representative training datasets collected by vehicles.

% %%%% Table: label %%%%%
% \begin{table}[h]
% \centering
% \caption{\newrev{Precision and Recall for different vehicle selection and data collection strategies, with the best results excluding Centralized ML shown in bold. (change this table to a figure?)}
% }
% % \resizebox{0.5\textwidth}{!}
% {
% \setlength{\tabcolsep}{2.3mm}  % 保持原压缩
% \begin{tabular}{lcc}
% \toprule
% % & \multicolumn{4}{c}{V2X-Sim} & \multicolumn{4}{c}{DAIR-V2X} \\
% % \cmidrule(lr){2-5}\cmidrule(lr){6-9}
%  Method & {Precision} & {Recall} \\
% % \cmidrule(lr){2-3}\cmidrule(lr){4-5}\cmidrule(lr){6-7}\cmidrule(lr){8-9}
% % & \multicolumn{2}{c}{Recall} & \multicolumn{2}{c}{Precision} \\
% % \cmidrule(lr){2-3}\cmidrule(lr){4-5}

%   % & Recall  & Precision \\
% % &  IoU@0.3 &IoU@0.5 &  IoU@0.3 & IoU@0.5 &  IoU@0.3 &IoU@0.5 &  IoU@0.3 & IoU@0.5\\

% % Pretrained  & 69.09& 66.37 \\
%  \midrule
% Proposed Sense4FL & \textbf{78.16} & \textbf{69.28} \\
% Coverage Centric  & 71.21 & 60.17 \\
% Uploading Centric & 72.32 & 63.43 \\
% Random Selection  & 69.07 & 58.80 \\
% Full-data Selection & 66.20 & 54.35 \\
% Only Selection    & 67.86 & 55.81 \\
% Centralized ML    & 80.55 & 71.22 \\
% % FS & 62.60&51.87 & 79.34&65.74 \\
% \bottomrule
% \end{tabular}
% }
% \label{tab:precision}
% \end{table}

Moreover, Fig. \ref{Fig:detection} provides visualized results of object detection for our Sense4FL framework and benchmarks. The results demonstrate that our approach is much closer to the ground truth than other methods. In contrast, the benchmarks exhibit notable misclassifications, including misidentifying trucks as barriers and pedestrians as bicycles. In addition, the benchmarks also fail to detect certain objects, leading to incomplete or inaccurate scene understanding.

%In the FL setting, the aggregated training loss is derived from the weighted average of the local losses computed by each vehicle selected to participate in the FL process.  This figure shows the decreasing trend of training loss, with it converging after approximately 170 rounds. 
%Specifically, The mAP@.5 and mAP@.5:.95 are measured by evaluating the global model at the end of each round on a separate validation dataset.
%Fig. \ref{fig:map_round} demonstrates the mAP performance. We can see that our proposed Sense4FL outperforms the other benchmarks as we select vehicles with representative datasets. The performance of other frameworks is not satisfactory as they introduce model bias caused by the non-representative training dataset. 

% Furthermore, to demonstrate the necessity and effectiveness of our Sense4FL framework, i.e., incorporating trajectory-dependent training data collection and vehicle selection for FL to improve AD performance, we adapt the model pre-trained from Singapore's One-North area into another region, i.e., Boston.
\rev{In reality, autonomous driving companies can have a pre-trained model (say, from another city or from a cloud-based pre-training dataset) and then deploy the model in a new environment. To reflect this, we adapt the model pre-trained from Singapore's One-North area into another region, i.e., Boston Seaport.}
As shown in Fig. \ref{fig:adapting}, the phenomenon is similar to what we can observe from Fig. \ref{fig:training}. Specifically, the model performance of our scheme substantially outperforms other benchmarks in the model adaptation. We can take three insights from Fig. \ref{fig:adapting}. 
First, the AD model well-trained for one area may not perform very well under the other scenario, implying that a general AI model is not applicable to all cases and adapting is needed. 
Second, the time needed for convergence is much shorter. This highlights the feasibility of adapting an object detection model to achieve satisfactory performance in a new environment within a reasonable time frame.
At last, the proposed Sense4FL is a trajectory-dependent approach, which outperforms other methods, as it overcomes the model bias resulting from the inappropriate selection of vehicles.

% \newrev{Finally, we show the precision and recall for different vehicle selection and data collection strategies in Table \ref{tab:precision}. We can see our proposed Sense4FL method performance in avoiding false positives (high precision) and minimizing false negatives (high recall). The simulation results clearly demonstrate that our Sense4FL framework outperforms the baselines.}

%%%% Table: precision %%%%%
\begin{table}[h]
\centering
\caption{\rev{Precision and Recall for different vehicle selection and data collection strategies, with the best results excluding Centralized ML shown in bold.}
}
\resizebox{\linewidth}{!}
{
\rev{
\setlength{\tabcolsep}{2.3mm}  % 保持原压缩
\begin{tabular}{lcccc}
\toprule
& \multicolumn{2}{c}{Singapore's One-North} & \multicolumn{2}{c}{Boston Seaport} \\
\cmidrule(lr){2-3}\cmidrule(lr){4-5}
 Method & {Precision} & {Recall} & Precision & Recall \\
 \midrule
Proposed Sense4FL   & \textbf{70.16} & \textbf{61.28} & \textbf{69.13} & \textbf{60.27} \\
Gradient Based  & 65.32 & 57.43 & 65.14 & 56.31 \\
Power-of-Choice  & 63.29 & 56.17 & 64.28 & 54.24 \\
Uploading Centric      & 59.86 & 47.81 & 58.23 & 46.18 \\
Coverage Centric    & 57.77 & 46.40 & 56.57 & 45.64 \\
Random Selection & 52.20 & 41.35 & 50.19 & 40.22 \\
Centralized ML      & 77.55 & 69.22 & 76.28 & 68.25 \\
\bottomrule
\end{tabular}
}
}
\label{tab:precision}
\end{table}

% Moreover, Fig. \ref{Fig:detection} provides visualization results of object detection for our Sense4FL framework and benchmarks. It can be concluded that our approach is much closer to the ground-truth results than other approaches, underscoring the significant improvement from our Sense4FL framework.

%%%% Table: abalation %%%%%
\begin{table}[h]
\centering
\caption{\rev{Effects of the data collection strategies, with the best results excluding Centralized ML shown in bold.}
}
\resizebox{\linewidth}{!}
{
\rev{
\setlength{\tabcolsep}{2.3mm}  % 保持原压缩
\begin{tabular}{lcccc}
\toprule
& \multicolumn{2}{c}{Singapore's One-North} & \multicolumn{2}{c}{Boston Seaport} \\
\cmidrule(lr){2-3}\cmidrule(lr){4-5}
 Method & {mAP@0.5} & mAP@0.5:.95 & mAP@0.5 & mAP@0.5:.95 \\
 \midrule
Proposed Sense4FL   & \textbf{55.41} & \textbf{36.23} & \textbf{52.83} & \textbf{35.76} \\
Full-data Collection   & 45.82 & 31.48 & 45.10 & 30.61 \\
Selection Only   & 49.05 & 33.89 & 47.23 & 33.10 \\
Centralized ML      & 58.38 & 40.12 & 57.42 & 39.76 \\
\bottomrule
\end{tabular}
}
}
\label{tab:ablation}
\end{table}

\rev{Table \ref{tab:precision} presents the precision and recall achieved by various vehicle selection and data collection strategies. The simulation results indicate that the Sense4FL framework outperforms all baselines in avoiding false positives (high precision) and minimizing false negatives (high recall).
Finally, the ablation studies on the data collection scheme in Table \ref{tab:ablation} reports two baselines: 1) \textbf{Full-data Collection:} Each vehicle collects all available data along its entire predicted trajectory, i.e., $g_{v,m} = N_{v,m}$, and we select $S$ vehicles; 2) \textbf{Selection Only:} This method selects $S$ vehicles but does not optimize data collection and each vehicle stops collecting data upon being selected, i.e., $g_{v,m}= c_{v,m}$. We can see that the Sense4FL framework consistently outperforms the baselines, demonstrating the salient advantage of joint optimization of vehicle selection and data collection.}

\section{Conclusion}
\label{sec:conclusion}
% In this paper, we proposed vehicular crowdsensing enabled federated learning to improve autonomous driving performance for a region of interest. To achieve this goal, we addressed the joint optimization problem of vehicle selection and data collection for federated learning. These two decision variables result in not only the different collected data distributions but also the uploading opportunities, thereby significantly impacting training performance. To tackle this problem, we have first laid the theoretical foundation by establishing the convergence upper bound in terms of vehicles' collected training data along their trajectories. Then, we formulated the problem to minimize the training loss, which is equivalent to a combination of local and global earth mover's distances between vehicles' collected datasets and global datasets. Simulation results based on nuImages dataset have demonstrated the significance of Sense4FL for improving object detection performance under different driving scenarios and the superiority of our schemes compared to other benchmarks.

In this paper, we have proposed vehicular crowdsensing enabled federated learning to improve autonomous driving performance by considering the impact of vehicles' uncertain trajectories. We have first laid the theoretical foundation by establishing the convergence upper bound of federated learning in terms of vehicles' collected training data along their trajectories. Our theoretical analysis reveals that vehicle selection and data collection strategies have a significant influence on the training data distribution and, consequently, the performance of FL. Then, we have formulated the problem to minimize the training loss, which is equivalent to a combination of local and global earth mover's distances between vehicles' collected datasets and global datasets and developed an efficient algorithm to find the
solution with an approximation guarantee. 
Simulation results based on nuImages dataset have demonstrated the significance of Sense4FL for improving object detection performance under different driving scenarios and the superiority of our schemes compared to other benchmarks.

\rev{While in this paper, we choose the object detection task in autonomous driving as the subject of study, our proposed Sense4FL framework with trajectory-aware vehicle selection can be extended to other critical tasks, such as semantic segmentation and trajectory prediction, which can be left as future work.}
% We can also couple bandwidth-adaptive scheduling, and strengthen privacy/robustness (domain adaptation, uncertainty-aware aggregation) with large-scale on-road validation.
% \rev{While different from existing FL proposals in vehicular networks by pioneering object detection performance on AVs, Sense4FL is also the first to consider the effects of trajectory-dependent vehicular data collection on AVs. However, our current Sense4FL framework still bears limitation: we only consider the classical object detection task in this paper.} 
% Although object detection is one of the most important tasks in autonomous driving, other critical tasks, such as semantic segmentation, traffic light and sign recognition, trajectory prediction, and so on, are not studied in this paper, which can be generalized in future work.

% Future work will extend Sense4FL from object detection to a unified multi-task perception stack (semantic segmentation, traffic signal recognition, multi-object tracking, and trajectory prediction).
%We will also couple learning with trajectory-aware data acquisition and bandwidth-adaptive scheduling, and strengthen privacy/robustness (domain adaptation, uncertainty-aware aggregation) with large-scale on-road validation.

}

%Unlike existing works on vehicle selection for FL assuming \textit{predetermined} and \textit{location-independent} vehicles’ datasets, our work considers the collection of training data samples based on vehicles' trajectories. To this end, we first derive the convergence bound of FL by taking into account the impact of both vehicles’ uncertain trajectories and uploading probabilities, finding that minimizing the training loss is equivalent to minimizing a weighted summation of local and global earth mover's distance (EMD) between vehicles’ collected data distribution and global data distribution. Based on this observation, we formulate the \textit{trajectory-dependent} vehicle selection and data collection problem and then develop a heuristic algorithm to find the solution efficiently. Extensive simulation results have demonstrated the effectiveness of our approach in improving object detection model performance compared with existing benchmarks without considering trajectory-dependent training data collection from vehicles.

\bibliographystyle{IEEEtran}
\bibliography{reference}

\newpage

\appendix
\subsection{Proof of Theorem \ref{theo:FL}}\label{app:converge}
To analyze the convergence performance of Sense4FL and characterize the relationship between training loss and data distribution, we assume an idealized centralized machine learning (CML) where the data distribution is identical to that of Sense4FL for AD in our paper, i.e., it matches the data distribution across all street blocks in the coverage region. 
Denoting the model of CML in the \(k\)-th round as \(\mathbf{w}_\mathrm{c}^{(k)}\), the loss function can be expressed by
\begin{equation}
\mathcal{F}(\mathbf{w}_\mathrm{c}^{(k)})=\sum_{b=1}^B l_b\sum_{i=1}^C p_{b}^{i} \mathbb{E}_{\mathbf{x}_b^i}\left[ f\left(\mathbf{w}_\mathrm{c}^{(k)}, \mathbf{x}_b^i\right)\right].
\end{equation}

The CML also performs a \(T\)-step SGD update. In the \(k\)-th round, CML updates at step $t$ as follows
\begin{equation} \label{eq:wc_update}
\begin{aligned}
\mathbf{w}_{\mathrm{c}}^{(k),t+1}=&\mathbf{w}_{\mathrm{c}}^{(k),t}- \eta \sum_{b=1}^B l_b \sum_{i=1}^C p_{b}^{i} \nabla_\mathbf{w} \mathbb{E}_{\mathbf{x}_b^i}\left[f\left(\mathbf{w}_{\mathrm{c}}^{(k),t},\mathbf{x}_b^i\right)\right],
\end{aligned}
\end{equation}
where $\eta$ is the learning rate.
It can be observed that the primary difference between \eqref{eq:wv_update} and \eqref{eq:wc_update} lies in the data distribution, i.e., $ p_{v,m}^{i}$ and $ p_{b}^{i}$, which means that EMD is a good metric to quantify the weight divergence and thus the model accuracy.

% \textcolor{blue}{For further analysis, we first introduce some definitions and lemmas~\cite{wang2019adaptive}}.

We define $\gamma^{(k),t}\triangleq \mathcal{F}(\mathbf{w}_{\mathrm{c}}^{(k),t})-\mathcal{F}(\mathbf{w}^\star)$.
According to the convergence lower bound of gradient descent in Theorem 3.14 in~\cite{bubeck2015convex},  we always have
\begin{equation}
    \gamma^{(k),t}= \mathcal{F}(\mathbf{w}_{\mathrm{c}}^{(k),t})-\mathcal{F}(\mathbf{w}^\star) >0, ~\forall t, ~\forall k.
\end{equation}

% \textcolor{blue}{\textbf{Lemma 1.} For $t=0, 1,2,...,T$ and $k=1,2,...,K$, when $\eta\leq\frac{1}{\beta}$, $\|\mathbf{w}_{\mathrm{c}}^{(k),t}-\mathbf{w}^\star\|$ does not increase with $t$, where $\mathbf{w}^\star$ is the optimal model.}

% \textcolor{blue}{\textit{Proof. According to \eqref{eq:wc_update}, we have}
% \begin{equation}\label{eq:wc-w^star}
%     \begin{aligned}
%     &\|\mathbf{w}_{\mathrm{c}}^{(k),t+1}-\mathbf{w}^\star\|^2\\
%     = & \|\mathbf{w}_{\mathrm{c}}^{(k),t}-\eta \nabla \mathcal{F}(\mathbf{w}_{\mathrm{c}}^{(k),t})-\mathbf{w}^\star\|^2\\
%     = & \|\mathbf{w}_{\mathrm{c}}^{(k),t}-\mathbf{w}^\star\|^2-2\eta\nabla \mathcal{F}(\mathbf{w}_{\mathrm{c}}^{(k),t})^T(\mathbf{w}_{\mathrm{c}}^{(k),t}-\mathbf{w}^\star) \\
%     &+ \eta^2 \|\nabla \mathcal{F}(\mathbf{w}_{\mathrm{c}}^{(k),t})\|^2.
%     \end{aligned}
% \end{equation}}

Now, we analyze the divergence between $\mathbf{w}_{\mathrm{f}}^{(k),T}$ and $\mathbf{w}_{\mathrm{c}}^{(k),T}$.
By defining $\xi_{v,m}^{(k)}\triangleq \frac{z_{v,m}^{(k)}e_{v,m}^{(k)}}{\sum_{m=1}^{M_v^{(k)}}z_{v,m}^{(k)}e_{v,m}^{(k)}}$ to denote the weighting factor of $\mathbf{w}_{v,m}^{(k),T}$ and according to the local SGD update process \eqref{eq:wv_update} and \eqref{eq:wc_update}, we can obtain
\begin{equation}
    \begin{aligned}
   &\Vert\mathbf{w}_{\mathrm{f}}^{(k),T}-\mathbf{w}_{\mathrm{c}}^{(k),T}\Vert\\
    =&\|\sum_{v=1}^{ V^{(k)}}\frac{a_{v}^{(k)} \rho_v^{(k)}}{\sum_{v=1}^{ V^{(k)}}a_{v}^{(k)} \rho_v^{(k)}}\sum_{m=1}^{M_v^{(k)}}\xi_{v,m}^{(k)}(\mathbf{w}_{v,m}^{(k),T-1}-\eta\\
    &\sum_{i=1}^C p_{v,m}^{i,(k)}\nabla_{\mathbf{w}} \mathbb{E}_{\mathbf{x}_{v,m}^i}[f(\mathbf{w}_{v,m}^{(k),T-1}, \mathbf{x}_{v,m}^i)]) -(\mathbf{w}_{\mathrm{c}}^{(k),T-1}\\
    &-\eta \sum_{b=1}^B l_b \sum_{i=1}^C p_{b}^{i} \nabla_\mathbf{w} \mathbb{E}_{\mathbf{x}_b^i}[f(\mathbf{w}_{\mathrm{c}}^{(k),T-1}, \mathbf{x}_b^i)])\|\\
    {\leq}&\|\sum_{v=1}^{ V^{(k)}}\frac{a_{v}^{(k)} \rho_v^{(k)}}{\sum_{v=1}^{ V^{(k)}}a_{v}^{(k)} \rho_v^{(k)}}\sum_{m=1}^{M_v^{(k)}}\xi_{v,m}^{(k)}\mathbf{w}_{v,m}^{(k),T-1}-\mathbf{w}_{\mathrm{c}}^{(k),T-1}\|+
    \eta\\
    &\|\sum_{v=1}^{ V^{(k)}}\frac{a_{v}^{(k)} \rho_v^{(k)}}{\sum_{v=1}^{ V^{(k)}}a_{v}^{(k)} \rho_v^{(k)}}\sum_{m=1}^{M_v^{(k)}}\xi_{v,m}^{(k)}(\sum_{i=1}^C p_{v,m}^{i,(k)}\nabla_{\mathbf{w}} \mathbb{E}_{\mathbf{x}_{v,m}^i} \\
    &[f(\mathbf{w}_{v,m}^{(k),T-1}, \mathbf{x}_{v,m}^i)]  
    -\sum_{b=1}^B l_b\sum_{i=1}^C p_{b}^{i} \nabla_\mathbf{w} \mathbb{E}_{\mathbf{x}_b^i}[f(\mathbf{w}_{\mathrm{c}}^{(k),T-1},\mathbf{x}_b^i)])\|.
    % \overset{(a)}{\leq}&\sum_{v=1}^{ V^{(k)}}a_{v}^{(k)} \rho_v^{(k)}\sum_{m=1}^{M_v}\xi_{v,m}^{(k)}
    % (1+\eta \sum_{i=1}^C p_{v,m}^{i}\lambda_i) \\&
    % \| \mathbf{w}_{v,m}^{(k),T-1}-\mathbf{w}_{\mathrm{c}}^{(k),T-1}\|+\eta \| \sum_{v=1}^{ V^{(k)}}a_{v}^{(k)} \rho_v^{(k)}\sum_{m=1}^{M_v}\xi_{v,m}^{(k)}\\
    % &(\sum_{i=1}^C p_{v,m}^{i}- \sum_{b=1}^B l_b \sum_{i=1}^C p_{b}^{i})
    % \nabla_\mathbf{w} \mathbb{E}_{\mathbf{x}^i}[f(\mathbf{w}_{\mathrm{c}}^{(k),T-1},\mathbf{x}^i)]\|\\
    \end{aligned}
\end{equation}
Then, we use $\sum_{v=1}^{ V^{(k)}}\frac{a_{v}^{(k)} \rho_v^{(k)}}{\sum_{v=1}^{ V^{(k)}}a_{v}^{(k)} \rho_v^{(k)}}\sum_{m=1}^{M_v^{(k)}}\xi_{v,m}^{(k)}\sum_{i=1}^C p_{v,m}^{i,(k)}$ $\nabla_{\mathbf{w}} \mathbb{E}_{\mathbf{x}_{v,m}^i}[f(\mathbf{w}_{\mathrm{c}}^{(k),T-1}, \mathbf{x}_{v,m}^i)]$ as an intermediate item and obtain
\begin{equation}
    \begin{aligned}
    &\Vert\mathbf{w}_{\mathrm{f}}^{(k),T}-\mathbf{w}_{\mathrm{c}}^{(k),T}\Vert\\
    \leq&\|\sum_{v=1}^{ V^{(k)}}\frac{a_{v}^{(k)} \rho_v^{(k)}}{\sum_{v=1}^{ V^{(k)}}a_{v}^{(k)} \rho_v^{(k)}}\sum_{m=1}^{M_v^{(k)}}\xi_{v,m}^{(k)}\mathbf{w}_{v,m}^{(k),T-1}-\mathbf{w}_{\mathrm{c}}^{(k),T-1}\|+
    \eta \\
    &\|\sum_{v=1}^{ V^{(k)}}\frac{a_{v}^{(k)} \rho_v^{(k)}}{\sum_{v=1}^{ V^{(k)}}a_{v}^{(k)} \rho_v^{(k)}}\sum_{m=1}^{M_v^{(k)}}\xi_{v,m}^{(k)}\sum_{i=1}^C p_{v,m}^{i,(k)}(\nabla_{\mathbf{w}} \mathbb{E}_{\mathbf{x}_{v,m}^i}\\
    &[f(\mathbf{w}_{v,m}^{(k),T-1}, \mathbf{x}_{v,m}^i)] -\nabla_{\mathbf{w}} \mathbb{E}_{\mathbf{x}_{v,m}^i}[f(\mathbf{w}_{\mathrm{c}}^{(k),T-1}, \mathbf{x}_{v,m}^i)]) +\\
    &\sum_{v=1}^{ V^{(k)}}\frac{a_{v}^{(k)} \rho_v^{(k)}}{\sum_{v=1}^{ V^{(k)}}a_{v}^{(k)} \rho_v^{(k)}}\sum_{m=1}^{M_v^{(k)}}\xi_{v,m}^{(k)}(\sum_{i=1}^C p_{v,m}^{i,(k)}- \sum_{b=1}^B l_b \sum_{i=1}^C p_{b}^{i}) \\
    &\nabla_\mathbf{w} \mathbb{E}_{\mathbf{x}_b^i}[f(\mathbf{w}_{\mathrm{c}}^{(k),T-1},\mathbf{x}_b^i)]\|\\
    \overset{(a)}{\leq}&\sum_{v=1}^{ V^{(k)}}\frac{a_{v}^{(k)} \rho_v^{(k)}}{\sum_{v=1}^{ V^{(k)}}a_{v}^{(k)} \rho_v^{(k)}}\sum_{m=1}^{M_v^{(k)}}\xi_{v,m}^{(k)}
    (1+\eta \sum_{i=1}^C p_{v,m}^{i,(k)}\lambda_i) \\
    &\|\mathbf{w}_{v,m}^{(k),T-1}-\mathbf{w}_{\mathrm{c}}^{(k),T-1}\|+\eta \mu_{\max}(\mathbf{w}_{\mathrm{c}}^{(k),T-1})\\& \sum_{i=1}^C \|\sum_{v=1}^{ V^{(k)}}\frac{a_{v}^{(k)} \rho_v^{(k)}}{\sum_{v=1}^{ V^{(k)}}a_{v}^{(k)} \rho_v^{(k)}}\sum_{m=1}^{M_v^{(k)}}\xi_{v,m}^{(k)}p_{v,m}^{i,(k)}- \sum_{b=1}^B l_b p_{b}^{i}\|,
    \end{aligned}
\end{equation}
where inequality $(a)$ holds because of Assumption 4 and $\mu_{\max}(\mathbf{w}_{\mathrm{c}}^{(k), T-1})\triangleq \max_{i=1}^C\|\nabla_\mathbf{w} \mathbb{E}_{\mathbf{x}_b^i}[f(\mathbf{w}_{\mathrm{c}}^{(k), T-1},\mathbf{x}_b^i)]\|$ is defined to represent the maximum norm of the expected gradient for each class of model $\mathbf{w}_{\mathrm{c}}^{(k), T-1}$ at the $k$-th round.

In the following, we focus on the term $\|\mathbf{w}_{v,m}^{(k),T-1}-\mathbf{w}_{\mathrm{c}}^{(k),T-1}\|$.
Similarly, defining $\theta_{v,m}^{(k)} \triangleq 1+\eta \sum_{i=1}^C p_{v,m}^{i,(k)}\lambda_i$,
% according to the local SGD update process \eqref{eq:wv_update} and \eqref{eq:wc_update}, 
we can get
\begin{equation}
    \begin{aligned}
    &\|\mathbf{w}_{v,m}^{(k),T-1}-\mathbf{w}_{\mathrm{c}}^{(k),T-1}\|\\
    \leq&  \theta_{v,m}^{(k)}\|\mathbf{w}_{v,m}^{(k),T-2}-\mathbf{w}_{\mathrm{c}}^{(k),T-2}\|\\
    &+\eta \mu_{\max}(\mathbf{w}_{\mathrm{c}}^{(k),T-2}) \sum_{i=1}^C \|p_{v,m}^{i,(k)}- \sum_{b=1}^B l_b p_{b}^{i}\|\\
    \leq& (\theta_{v,m}^{(k)})^{T-1} \|\mathbf{w}_{v,m}^{(k),0}-\mathbf{w}_{\mathrm{c}}^{(k),0}\|+\eta \sum_{i=1}^C \|p_{v,m}^{i,(k)}- \sum_{b=1}^B l_b p_{b}^{i}\|\\
    &((\theta_{v,m}^{(k)})^{T-2}\mu_{\max}(\mathbf{w}_{\mathrm{c}}^{(k),0})+... 
    +\theta_{v,m}^{(k)}\mu_{\max}(\mathbf{w}_{\mathrm{c}}^{(k),T-3})\\
    &+\mu_{\max}(\mathbf{w}_{\mathrm{c}}^{(k),T-2}))\\
    = &  (\theta_{v,m}^{(k)})^{T-1} \|\mathbf{w}_{\mathrm{f}}^{(k),0}-\mathbf{w}_{\mathrm{c}}^{(k),0}\|+\eta \sum_{i=1}^C \|p_{v,m}^{i,(k)}- \sum_{b=1}^B l_b p_{b}^{i}\|\\
    &
    \sum_{j=1}^{T-1}(\theta_{v,m}^{(k)})^{(j-1)}
    \mu_{\max}(\mathbf{w}_{\mathrm{c}}^{(k),T-1-j}).
    \end{aligned}
\end{equation}

We assume that in each round, the initial model for Sense4FL is equivalent to the initial model for CML, which means
$\mathbf{w}_{\mathrm{f}}^{(k),0}=\mathbf{w}_{\mathrm{c}}^{(k),0}$.
Based on this, we can obtain 
% \begin{equation}
%     \begin{aligned}\label{eq:w_vm-w_r}
%     &\|\mathbf{w}_{v,m}^{(k),T-1}-\mathbf{w}_{\mathrm{c}}^{(k),T-1}\|\\
%     \leq&\eta \sum_{i=1}^C \|p_{v,m}^{i}- \sum_{b=1}^B l_b p_{b}^{i}\|
%    \sum_{j=1}^{T-1}(\theta_{v,m})^{(j-1)}
%     \mu_{\max}(\mathbf{w}_{\mathrm{c}}^{(k),T-1-j}).
%     \end{aligned}
% \end{equation}
% Therefore, based on \eqref{eq:w_f-w_r} and \eqref{eq:w_vm-w_r}, it follows that
\begin{equation}
    \begin{aligned}
   &\Vert\mathbf{w}_{\mathrm{f}}^{(k),T}-\mathbf{w}_{\mathrm{c}}^{(k),T}\Vert\\
    \leq&\eta\sum_{v=1}^{ V^{(k)}}\frac{a_{v}^{(k)} \rho_v^{(k)}}{\sum_{v=1}^{ V^{(k)}}a_{v}^{(k)} \rho_v^{(k)}}\sum_{m=1}^{M_v^{(k)}}\xi_{v,m}^{(k)}
    \sum_{i=1}^C \|p_{v,m}^{i,(k)}- \sum_{b=1}^B l_b p_{b}^{i}\| \\&
    \sum_{j=1}^{T-1}(\theta_{v,m}^{(k)})^j
    \mu_{\max}(\mathbf{w}_{\mathrm{c}}^{(k),T-1-j})+\eta \mu_{\max}(\mathbf{w}_{\mathrm{c}}^{(k),T-1})\\& \sum_{i=1}^C \|\sum_{v=1}^{ V^{(k)}}\frac{a_{v}^{(k)} \rho_v^{(k)}}{\sum_{v=1}^{ V^{(k)}}a_{v}^{(k)} \rho_v^{(k)}}\sum_{m=1}^{M_v^{(k)}}\xi_{v,m}^{(k)}p_{v,m}^{i,(k)}- \sum_{b=1}^B l_b p_{b}^{i}\|.
    \end{aligned}
\end{equation}

Thus, we have obtained the divergence between the Sense4FL model $\mathbf{w}_{\mathrm{f}}^{(k),T}$ and the CML model $\mathbf{w}_{\mathrm{c}}^{(k),T}$ at the $k$-th round. Next, we analyze the relationship between this divergence and the learning performance of FL~\cite{zhang2024coalitional}.

% \textcolor{blue}{We will discuss the following two cases:}

% \textcolor{blue}{\textbf{Case 1:}
% We consider that the loss of CML at the $T$-th step in the $k$-th round is equal to the optimal loss, i.e.,
%     $\mathcal{F}(\mathbf{w}_{\mathrm{c}}^{(k),T})-\mathcal{F}(\mathbf{w}^\star)= 0$.
% In this case, we have
% \begin{equation}
%     \begin{aligned}
%         &\mathcal{F}(\mathbf{w}_\mathrm{f}^{(K),T}) -\mathcal{F}(\mathbf{w}^\star)\\
%         =&\mathcal{F}(\mathbf{w}_\mathrm{f}^{(K),T}) -\mathcal{F}(\mathbf{w}_{\mathrm{c}}^{(k),T})\\
%         \overset{(b)}{\leq}&L\|\mathbf{w}_\mathrm{f}^{(K),T} -\mathbf{w}_{\mathrm{c}}^{(k),T}\|,      
%     \end{aligned}
% \end{equation}
% where inequality $(b)$ holds because of Assumption 2.}

% \textcolor{blue}{\textbf{Case 2:}
% We consider that there is a gap between the loss of CML at the $T$-th step in the $k$-th round and the optimal loss, i.e.,
%     $\mathcal{F}(\mathbf{w}_{\mathrm{c}}^{(k),T})-\mathcal{F}(\mathbf{w}^\star)\geq \epsilon_1>0$.}
% % In this case, 

According to Assumption 2, when $\eta \leq \frac{1}{\beta}$, we have
\begin{equation}\label{eq:fr}
    \begin{aligned}
        &\mathcal{F}(\mathbf{w}_{\mathrm{c}}^{(k),t+1})-\mathcal{F}(\mathbf{w}_{\mathrm{c}}^{(k),t})\\
        \leq&
        \nabla \mathcal{F}(\mathbf{w}_{\mathrm{c}}^{(k),t})^T(\mathbf{w}_{\mathrm{c}}^{(k),t+1}-\mathbf{w}_{\mathrm{c}}^{(k),t})+\frac{\beta}{2}\Vert\mathbf{w}_{\mathrm{c}}^{(k),t+1}-\mathbf{w}_{\mathrm{c}}^{(k),t}\Vert^2\\
        \leq&-\eta\nabla \mathcal{F}(\mathbf{w}_{\mathrm{c}}^{(k),t})^T\nabla \mathcal{F}(\mathbf{w}_{\mathrm{c}}^{(k),t})+\frac{\beta\eta^2}{2}\|\nabla \mathcal{F}(\mathbf{w}_{\mathrm{c}}^{(k),t})\|^2\\
        =&
        -\eta(1-\frac{\beta\eta}{2})\|\nabla \mathcal{F}(\mathbf{w}_{\mathrm{c}}^{(k),t})\|^2,
    \end{aligned}
\end{equation}
where 
\begin{equation}
    \nabla \mathcal{F}(\mathbf{w}_{\mathrm{c}}^{(k),t}) = \eta \sum_{b=1}^B l_b \sum_{i=1}^C p_{b}^{i} \nabla_\mathbf{w} \mathbb{E}_{\mathbf{x}_b^i}[f(\mathbf{w}_{\mathrm{c}}^{(k),t},\mathbf{x}_b^i)] 
\end{equation}
is the gradient of $ \mathcal{F}(\mathbf{w}_{\mathrm{c}}^{(k),t})$.
Since $\gamma^{(k),t+1}\triangleq \mathcal{F}(\mathbf{w}_{\mathrm{c}}^{(k),t+1})-\mathcal{F}(\mathbf{w}^\star)$ and $\gamma^{(k),t}\triangleq \mathcal{F}(\mathbf{w}_{\mathrm{c}}^{(k),t})-\mathcal{F}(\mathbf{w}^\star)$, substituting these into \eqref{eq:fr}, we obtain 
\begin{equation}
    \begin{aligned}\label{eq:diff}
        \gamma^{(k),t+1}-\gamma^{(k),t}\leq -\eta(1-\frac{\beta\eta}{2})\|\nabla \mathcal{F}(\mathbf{w}_{\mathrm{c}}^{(k),t})\|^2.
    \end{aligned}
\end{equation}

Assumption 1 gives
\begin{equation}
    \begin{aligned}
        \gamma^{(k),t} 
        =& \mathcal{F}(\mathbf{w}_{\mathrm{c}}^{(k),t})-\mathcal{F}(\mathbf{w}^\star)\\
        \leq&\nabla\mathcal{F}(\mathbf{w}_{\mathrm{c}}^{(k),t})^T(\mathbf{w}_{\mathrm{c}}^{(k),t}-\mathbf{w}^\star)\\
        \leq& \|\nabla\mathcal{F}(\mathbf{w}_{\mathrm{c}}^{(k),t})\|\|\mathbf{w}_{\mathrm{c}}^{(k),t}-\mathbf{w}^\star\|,
    \end{aligned}
\end{equation}
which can be transformed into
\begin{equation}
    \begin{aligned}\label{eq:in_gamma}
        \frac{\gamma^{(k),t} }{\|\mathbf{w}_{\mathrm{c}}^{(k),t}-\mathbf{w}^\star\|}\leq \|\nabla\mathcal{F}(\mathbf{w}_{\mathrm{c}}^{(k),t})\|.
    \end{aligned}
\end{equation}

Hence, by combining equation \eqref{eq:diff} and \eqref{eq:in_gamma}, we have 
\begin{equation}\label{eq:gamma-gamma}
    \begin{aligned}
        \gamma^{(k),t+1}-\gamma^{(k),t}\leq
        &-\eta(1-\frac{\beta\eta}{2})\frac{(\gamma^{(k),t})^2 }{\|\mathbf{w}_{\mathrm{c}}^{(k),t}-\mathbf{w}^\star\|^2}.
    \end{aligned}
\end{equation}
 % For the term $\|\mathbf{w}_{\mathrm{c}}^{(k),t}-\mathbf{w}^\star\|$, we have the following lemma.

To proceed further, we need the following lemma. 
\begin{lemma} \label{lemma:non_increase}
    For $t=0, 1,2,...,T$ and $k=1,2,...,K$, when $\eta\leq\frac{1}{\beta}$, $\|\mathbf{w}_{\mathrm{c}}^{(k),t}-\mathbf{w}^\star\|$ does not increase with $t$~\cite{wang2019adaptive}.
\end{lemma}

\begin{proof}
    According to \eqref{eq:wc_update}, we have
\begin{equation}\label{eq:wc-w^star}
    \begin{aligned}
    &\|\mathbf{w}_{\mathrm{c}}^{(k),t+1}-\mathbf{w}^\star\|^2\\
    = & \|\mathbf{w}_{\mathrm{c}}^{(k),t}-\eta \nabla \mathcal{F}(\mathbf{w}_{\mathrm{c}}^{(k),t})-\mathbf{w}^\star\|^2\\
    = & \|\mathbf{w}_{\mathrm{c}}^{(k),t}-\mathbf{w}^\star\|^2-2\eta\nabla \mathcal{F}(\mathbf{w}_{\mathrm{c}}^{(k),t})^T(\mathbf{w}_{\mathrm{c}}^{(k),t}-\mathbf{w}^\star) \\
    &+ \eta^2 \|\nabla \mathcal{F}(\mathbf{w}_{\mathrm{c}}^{(k),t})\|^2.
    \end{aligned}
\end{equation}

Since $\mathcal{F}(\cdot)$ is $\beta$-smooth, according to Lemma 3.14 in~\cite{bubeck2015convex}, we have $\gamma^{(k),t} > 0$ for any $k$ and $t$. Additionally, according to
Lemma 3.5 in~\cite{bubeck2015convex}, we obtain the following inequality
\begin{equation}
    \begin{aligned}
        0<\gamma^{(k),t}\leq \nabla \mathcal{F}(\mathbf{w}_{\mathrm{c}}^{(k),t})^T(\mathbf{w}_{\mathrm{c}}^{(k),t}-\mathbf{w}^\star) -\frac{\|\nabla \mathcal{F}(\mathbf{w}_{\mathrm{c}}^{(k),t})\|^2}{2\beta},
    \end{aligned}
\end{equation}
which can be transformed into
\begin{equation}\label{eq:g_wc}
    \begin{aligned}
        -\nabla \mathcal{F}(\mathbf{w}_{\mathrm{c}}^{(k),t})^T(\mathbf{w}_{\mathrm{c}}^{(k),t}-\mathbf{w}^\star) < - \frac{\|\nabla \mathcal{F}(\mathbf{w}_{\mathrm{c}}^{(k),t})\|^2}{2\beta}.
    \end{aligned}
\end{equation}

By combining \eqref{eq:wc-w^star} and \eqref{eq:g_wc}, we have
\begin{equation}
    \begin{aligned}
        &\|\mathbf{w}_{\mathrm{c}}^{(k),t+1}-\mathbf{w}^\star\|^2\\
    < & \|\mathbf{w}_{\mathrm{c}}^{(k),t}-\mathbf{w}^\star\|^2-\frac{\eta}{\beta}\| \nabla \mathcal{F}(\mathbf{w}_{\mathrm{c}}^{(k),t})\|^2+\eta^2\|\nabla \mathcal{F}(\mathbf{w}_{\mathrm{c}}^{(k),t})\|^2\\
    = &\|\mathbf{w}_{\mathrm{c}}^{(k),t}-\mathbf{w}^\star\|^2-\eta(\frac{1}{\beta}-\eta)\| \nabla \mathcal{F}(\mathbf{w}_{\mathrm{c}}^{(k),t})\|^2.
    \end{aligned}
\end{equation}
    
When $\eta\leq\frac{1}{\beta}$, we have
\begin{equation}
      \|\mathbf{w}_{\mathrm{c}}^{(k),t+1}-\mathbf{w}^\star\|^2  \leq\|\mathbf{w}_{\mathrm{c}}^{(k),t}-\mathbf{w}^\star\|^2.
\end{equation}

This completes the proof of Lemma \ref{lemma:non_increase}.
\end{proof}

By defining $\phi\triangleq\min_{k}\frac{1}{\|\mathbf{w}_{\mathrm{c}}^{(k),1}-\mathbf{w}^{\star}\|^2}$, 
we have
\begin{equation}
    \phi \leq \min_{k}\frac{1}{\|\mathbf{w}_{\mathrm{c}}^{(k),t}-\mathbf{w}^{\star}\|^2}.
\end{equation}

Therefore,
the inequality \eqref{eq:gamma-gamma} follows that
\begin{equation}
    \begin{aligned}\label{eq:new_gamma-gamma}
        \gamma^{(k),t+1}-\gamma^{(k),t}\leq
        % &-\eta(1-\frac{\beta\eta}{2})\frac{(\gamma^{(k),t})^2 }{\|\mathbf{w}_{\mathrm{c}}^{(k),t}-\mathbf{w}^\star\|^2}\\\leq &
        -\phi\eta(1-\frac{\beta\eta}{2})(\gamma^{(k),t})^2.
    \end{aligned}
\end{equation}

By dividing both sides by $\gamma^{(k),t+1}\gamma^{(k),t}$ in \eqref{eq:new_gamma-gamma}, we can get
% \begin{equation}
%     \begin{aligned}
%         \frac{1}{\gamma^{(k),t}}-\frac{1}{\gamma^{(k),t+1}}\leq-\phi\eta(1-\frac{\beta\eta}{2})\frac{\gamma^{(k),t}}{\gamma^{(k),t+1}},
%     \end{aligned}
% \end{equation}
% which can be transformed into
\begin{equation}
    \frac{1}{\gamma^{(k),t+1}}-\frac{1}{\gamma^{(k),t}} \geq \phi\eta(1-\frac{\beta\eta}{2})\frac{\gamma^{(k),t}}{\gamma^{(k),t+1}}\geq \phi\eta(1-\frac{\beta\eta}{2}).
\end{equation}

Based on this, 
for the $k$-th round, we have
\begin{equation}
    \begin{aligned}
        \frac{1}{\gamma^{(k),T}}-\frac{1}{\gamma^{(k),0}} &= \sum_{t=0}^{T-1}(\frac{1}{\gamma^{(k),t+1}}-\frac{1}{\gamma^{(k),t}})\\
       & \geq T\phi\eta(1-\frac{\beta\eta}{2}).
    \end{aligned}
\end{equation}

Summing up the above for all FL rounds yields
% it follows that 
\begin{equation}\label{eq:all_FL}
    \sum_{k=1}^{K}\frac{1}{\gamma^{(k),T}}-\frac{1}{\gamma^{(k),0}} \geq KT\phi\eta(1-\frac{\beta\eta}{2}).
\end{equation}

Rearranging the left-hand side of this inequality yields that 
\begin{equation}
    \begin{aligned}\label{eq:1/gamma}
        &\frac{1}{\gamma^{(K),T}}-\frac{1}{\gamma^{(K),0}}+\frac{1}{\gamma^{(K-1),T}}-\frac{1}{\gamma^{(K-1),0}}+...+\\
        &\frac{1}{\gamma^{(2),T}}-\frac{1}{\gamma^{(2),0}}+\frac{1}{\gamma^{(1),T}}-\frac{1}{\gamma^{(1),0}}\\
        =
        &\frac{1}{\gamma^{(k),T}}-\frac{1}{\gamma^{(1),0}}-\sum_{k=1}^{K-1}(\frac{1}{\gamma^{(k+1),0}}-\frac{1}{\gamma^{(k),T}})\\
        \geq& KT\phi\eta(1-\frac{\beta\eta}{2}).
    \end{aligned}
\end{equation}

Hence, we can get
\begin{equation}
    \begin{aligned}\label{eq:1/gamma}
        &\frac{1}{\gamma^{(k),T}}-\frac{1}{\gamma^{(1),0}}\\
        \geq &
        KT\phi\eta(1-\frac{\beta\eta}{2})+
        \sum_{k=1}^{K-1}(\frac{1}{\gamma^{(k+1),0}}-\frac{1}{\gamma^{(k),T}}).
    \end{aligned}
\end{equation}

% For further analysis, we make the following assumptions.

% \textbf{Assumption 4.} $\gamma^{(k),T}\triangleq \mathcal{F}(\mathbf{w}_{\mathrm{c}}^{(k),T})-\mathcal{F}(\mathbf{w}^\star)\geq \epsilon$, for $k=1,2,...,K$.

% Then, it follows that
% \begin{equation}
%     \gamma^{(k),t}\geq \gamma^{(k),T} \geq \epsilon, ~t=1,2,...,T.
% \end{equation}

% \textbf{Assumption 5.} $\mathcal{F}(\mathbf{w}_\mathrm{f}^{(K),T}) -\mathcal{F}(\mathbf{w}^\star)\geq \epsilon$.

According to \eqref{eq:fr}, $\mathcal{F}(\mathbf{w}_{\mathrm{c}}^{(k),t}) \geq \mathcal{F}(\mathbf{w}_{\mathrm{c}}^{(k),t+1})$, $\forall t$. 
Applying the condition (3) in Theorem \ref{theo:FL}, we have $\gamma^{(k),t}= \mathcal{F}(\mathbf{w}_{\mathrm{c}}^{(k),t})-\mathcal{F}(\mathbf{w}^\star)\geq \epsilon$ for all $t$ and $k$. Hence,
\begin{equation}
    \gamma^{(k),T}\gamma^{(k+1),0} \geq \epsilon^2.
\end{equation}

% It is assumed that $\mathcal{F}(\mathbf{w}_{\mathrm{c}}^{(k),T})-\mathcal{F}(\mathbf{w}^\star)\geq \epsilon$ for all $k$. 
According to Assumption 3, we obtain
\begin{equation}\label{eq:1/gamma_k}
    \begin{aligned}
        \frac{1}{\gamma^{(k+1),0}}-\frac{1}{\gamma^{(k),T}}
        =&\frac{\gamma^{(k),T}-\gamma^{(k+1),0}}{\gamma^{(k),T}\gamma^{(k+1),0}}\\
        =&\frac{\mathcal{F}(\mathbf{w}_{\mathrm{c}}^{(k),T})-\mathcal{F}(\mathbf{w}_{\mathrm{c}}^{(k+1),0})}{\gamma^{(k),T}\gamma^{(k+1),0}}\\
        =&-\frac{\mathcal{F}(\mathbf{w}_{\mathrm{f}}^{(k),T})-\mathcal{F}(\mathbf{w}_{\mathrm{c}}^{(k),T})}{\gamma^{(k),T}\gamma^{(k+1),0}}\\
        {\geq}&-\frac{L\|\mathbf{w}_{\mathrm{f}}^{(k),T}-\mathbf{w}_{\mathrm{c}}^{(k),T}\|}{\gamma^{(k),T}\gamma^{(k+1),0}}\\
        {\geq}&-\frac{L}{\epsilon^2}\|\mathbf{w}_{\mathrm{f}}^{(k),T}-\mathbf{w}_{\mathrm{c}}^{(k),T}\|.
    \end{aligned}
\end{equation}

From \eqref{eq:1/gamma} and \eqref{eq:1/gamma_k}, it follows that
\begin{equation}\label{eq:1/gammak-1}
    \begin{aligned}
        &\frac{1}{\gamma^{(k),T}}-\frac{1}{\gamma^{(1),0}}\\
        \geq &
        KT\phi\eta(1-\frac{\beta\eta}{2})-\frac{L}{\epsilon^2}
        \sum_{k=1}^{K-1}\|\mathbf{w}_{\mathrm{f}}^{(k),T}-\mathbf{w}_{\mathrm{c}}^{(k),T}\|.
    \end{aligned}
\end{equation}

Applying the condition (4) in Theorem \ref{theo:FL}, we get
\begin{equation}
    \begin{aligned}
        -\frac{1}{(\mathcal{F}(\mathbf{w}_\mathrm{f}^{(K),T}) -\mathcal{F}(\mathbf{w}^\star))\gamma^{(K),T}}\geq -\frac{1}{\epsilon^2}.
    \end{aligned}
\end{equation}

Thus, we have 
\begin{equation}
    \begin{aligned}\label{eq:f-fstar}
        &\frac{1}{\mathcal{F}(\mathbf{w}_\mathrm{f}^{(K),T}) -\mathcal{F}(\mathbf{w}^\star)}-\frac{1}{\gamma^{{K},T}}\\
        =&-\frac{\mathcal{F}(\mathbf{w}_\mathrm{f}^{(K),T}) -\mathcal{F}(\mathbf{w}_\mathrm{c}^{(K),T})}{(\mathcal{F}(\mathbf{w}_\mathrm{f}^{(K),T}) -\mathcal{F}(\mathbf{w}^\star))\gamma^{{K},T}}\\
        \overset{(b)}{\geq}&-\frac{L\|\mathbf{w}_\mathrm{f}^{(K),T} -\mathbf{w}_\mathrm{c}^{(K),T}\|}{(\mathcal{F}(\mathbf{w}_\mathrm{f}^{(K),T}) -\mathcal{F}(\mathbf{w}^\star))\gamma^{{K},T}}\\
        {\geq}&-\frac{L}{\epsilon^2}\|\mathbf{w}_\mathrm{f}^{(K),T} -\mathbf{w}_\mathrm{c}^{(K),T}\|,
    \end{aligned}
\end{equation}
where inequality $(b)$ holds because of Assumption 3.

Combining \eqref{eq:1/gammak-1} and \eqref{eq:f-fstar}, we obtain 
\begin{equation}
    \begin{aligned}
        &\frac{1}{\mathcal{F}(\mathbf{w}_\mathrm{f}^{(K),T}) -\mathcal{F}(\mathbf{w}^\star)}-\frac{1}{\gamma^{(1),0}}\\
        \geq &KT\phi\eta(1-\frac{\beta\eta}{2})-\frac{L}{\epsilon^2}
        \sum_{k=1}^{K}\|\mathbf{w}_{\mathrm{f}}^{(k),T}-\mathbf{w}_{\mathrm{c}}^{(k),T}\|.
    \end{aligned}
\end{equation}
% which follows that
% \begin{equation}\label{eq:bound_ini}
%     \begin{aligned}
%         &\frac{1}{\mathcal{F}(\mathbf{w}_\mathrm{f}^{(K),T}) -\mathcal{F}(\mathbf{w}^\star)}\\
%         \geq &KT\phi\eta(1-\frac{\beta\eta}{2})-\frac{L}{\epsilon^2}
%         \sum_{k=1}^{K}\|\mathbf{w}_{\mathrm{f}}^{(k),T}-\mathbf{w}_{\mathrm{c}}^{(k),T}\|.
%     \end{aligned}
% \end{equation}

% Define $U_{v,\max}^{(k)}\triangleq \max_{m}\eta \sum_{j=1}^{T-2}(\theta_{v,m})^j
% \mu_{\max}(\mathbf{w}_{\mathrm{c}}^{(k),T-1-j})$ for each vehicle $v$ in round $k$ and the maximum of $U_{v,\max}^{(k)}$ among each vehicle is defined as $U_{1,\mathrm{max}}^{(k)}$. In addition, we define $U_{2,\mathrm{max}}^{(k)} \triangleq \eta \mu_{\max}(\mathbf{w}_{\mathrm{c}}^{(k),T-1})$.
% Based on this, the expected divergence between $\mathbf{w}_{\mathrm{f}}^{(k),T}$ and $\mathbf{w}_{\mathrm{c}}^{(k),T}$ can be expressed by

Define $U \triangleq \max_{k}\max_{j}\mu_{\max}(\mathbf{w}_{\mathrm{c}}^{(k),j})$, $\overline{\xi_{v,m}^{(k)}} \triangleq\frac{q_{v,m}^{(k)}p_{v,m}^{(k),\text{rcv}}}{\sum_{m=1}^{M_v^{(k)}}{q_{v,m}^{(k)}p_{v,m}^{(k),\text{rcv}}}}$, $\lambda_{\max}\triangleq \max_i\lambda_i$,  $\delta \triangleq \sum_{j=1}^{T-1}(1+\eta \lambda_{\max})^{j}$. If $\phi KT(1-\frac{\beta\eta}{2})-\frac{L}{\epsilon^2}U\sum_{k=1}^{K}\Omega^{(k)}> 0$, we arrive at 
\begin{equation}
    \begin{aligned}
        &\mathbb{E}[\mathcal{F}(\mathbf{w}_\mathrm{f}^{(K),T})] -\mathcal{F}(\mathbf{w}^\star)\\
        \leq& \frac{1}{\eta(\phi KT(1-\frac{\beta\eta}{2})-\frac{L}{\epsilon^2}
         U\sum_{k=1}^{K}\Omega^{(k)})},
    \end{aligned}
\end{equation}
where 
\begin{equation}
    \begin{aligned}
    \Omega^{(k)}\triangleq &
        \delta\sum_{v=1}^{ V^{(k)}}\frac{a_{v}^{(k)} \rho_v^{(k)}}{\sum_{v=1}^{ V^{(k)}}a_{v}^{(k)} \rho_v^{(k)}}\sum_{m=1}^{M_v^{(k)}}\overline{\xi_{v,m}^{(k)}}
    \sum_{i=1}^C \|p_{v,m}^{i,(k)}- \sum_{b=1}^B l_b p_{b}^{i}\|\\
    &+\sum_{i=1}^C \|\sum_{v=1}^{ V^{(k)}}\frac{a_{v}^{(k)} \rho_v^{(k)}}{\sum_{v=1}^{ V^{(k)}}a_{v}^{(k)} \rho_v^{(k)}}\sum_{m=1}^{M_v^{(k)}}\overline{\xi_{v,m}^{(k)}}p_{v,m}^{i,(k)}- \sum_{b=1}^B l_b p_{b}^{i}\|.
    \end{aligned}
\end{equation}

\subsection{Proof of Theorem \ref{theorem:np_hard}}\label{sec: app2}
In a typical multiple-choice knapsack problem (MCKP), we assume that there are usually $J$ classes or groups $G_1,G_2,\ldots,G_J$, each containing $N_j$ items.
% ~\cite{ibaraki1978multiple}. 
Each item $i$ in group $G_j$ has a profit value $v_{ij}$ and a weight $w_{ij}$. The goal is to select exactly one item from each group in a way that maximizes the total value while ensuring that the total weight does not exceed the capacity $C$. 
% The corresponding optimization problem is
% \begin{subequations}
%     \begin{align}
%     \max_{x_{i j}} ~&\sum_{j=1}^J \sum_{i=1}^{N_j} v_{i j} x_{i j} \\ 
%     \text {s.t.}~& \sum_{j=1}^J \sum_{i=1}^{N_j} w_{i j} x_{i j} \leq C, \\ 
%     &\sum_{i=1}^{N_j} x_{i j}=1, ~\forall j=1,\ldots,J, \\ 
%     &x_{i j} \in\{0,1\},~\forall \{i=1,\ldots,N_j\}, \forall \{j = 1,\ldots,J\}.
%     \end{align}
% \end{subequations}   
% For the classic linear multiple-choice knapsack problem, the profit value function $v_{ij}$ is linear.

Consider a simplified version of our optimization problem, i.e., each vehicle has only one known trajectory (which can be obtained by the vehicle's active reporting).
In this simplified problem version, the capacity constraint refers to the maximum number of selected vehicles $S$. We have $V$ groups, each referring to one vehicle. Each group has $N_{v}$ items, belonging to the candidate data collection set $\{h_{v}(1),h_{v}(2),\ldots,h_{v}(N_{v})\}$.
By this definition, our problem can be stated as: given $V$ groups, each having $N_{v}$ items with different profit values and weights, select at most one item from each group to minimize the value of the weighted EMD under the capacity constraint $S$. The mathematical formulation is
\begin{subequations}
    \begin{align}
        \min_{x_{vn}}~~&f(x_{vn})\\  \text{s.t.}~~&\sum_{v=1}^{V}\sum_{n=1}^{N_v}x_{vn}= S,\\
        &\sum_{n=1}^{N_v}x_{vn}\leq 1,~ \forall v \in \{1, ..., {V}\},\\
        &x_{vn}\in\{0,1\},~\forall v \in \{1, ..., {V}\}, ~n \in \{1, ..., N_v\},
    \end{align}
\end{subequations}
where
\begin{equation}
    \begin{aligned}
        f(x_{vn})=&\delta \frac{\sum_{v=1}^{V}\sum_{n=1}^{N_v}x_{vn}\rho_{vn}\sum_{i=1}^C \|p_{vn}^{i}- \sum_{b=1}^B l_b p_{b}^{i}\|}{\sum_{v=1}^{V}\sum_{n=1}^{N_v}x_{vn} \rho_{vn}}
    \\
    &+ \sum_{i=1}^C \| \frac{\sum_{v=1}^{ V}\sum_{n=1}^{N_v}x_{vn}\rho_{vn}p_{vn}^{i}}{\sum_{v=1}^{V}\sum_{n=1}^{N_v}x_{vn} \rho_{vn}}- \sum_{b=1}^B l_b p_{b}^{i}\|,
    \end{aligned}
\end{equation}
with $p_{vn}^{i} = \frac{\sum_{b \in h_{v}(n)}Q_{b} p^{i}_b}{\sum_{b \in h_{v}(n)} Q_b}$ being the data distribution that vehicle $v$ chooses the $n$-th data collection scheme and $\rho_{vn} = \sum_{b\in h_{v}(n)}l_{b}$ being the weighting factor determined by data collection scheme.

We conclude that the simplified version of the optimization problem is in the form of an MCKP, which is widely known to be NP-hard.
% ~\cite{ibaraki1978multiple},~\cite{ibaraki1980approximate}. 
Since the special case is already NP-hard, our problem is NP-hard.  

\end{document}